\documentclass{article}
\usepackage{geometry}
\usepackage[nospace,noadjust]{cite}
\usepackage{amsmath,amssymb,amsfonts,times}
\usepackage{graphicx}
\usepackage{textcomp}
\usepackage{xcolor}
\usepackage{subfigure}
\usepackage{epsfig}
\usepackage{multirow}
\usepackage{algorithm}
\usepackage{algpseudocode}
\usepackage{algpseudocode}
\usepackage{array}
\usepackage{epstopdf}
\usepackage{color}
\usepackage{diagbox}
\usepackage{bm}
\usepackage{bbm}
\usepackage{mathrsfs}
\usepackage{tcolorbox}
\usepackage{pdfsync}
\usepackage{mdframed}

\usepackage{url}
\usepackage[hidelinks]{hyperref}
\usepackage{amsmath,amsthm,amssymb,amsbsy}
\usepackage{paralist}
\usepackage{xcolor}
\usepackage{color}
\usepackage{comment}
\usepackage{multirow}
\usepackage[toc, page]{appendix}

\usepackage{fancyhdr}
\usepackage{cite}
\usepackage{cleveref}

\newtheorem{lemma}{Lemma}
\newtheorem{definition}{Definition}

\newtheorem{theorem}{Theorem}

\theoremstyle{remark}
\newtheorem{remark}{Remark}
\theoremstyle{problem}

\newcommand{\R}{\mathbb{R}}
\def \real    { \mathbb{R} }

\newcommand{\e}{\begin{equation}}
\newcommand{\ee}{\end{equation}}
\newcommand{\en}{\begin{equation*}}
\newcommand{\een}{\end{equation*}}
\newcommand{\eqn}{\begin{eqnarray}}
\newcommand{\eeqn}{\end{eqnarray}}
\newcommand{\bmat}{\begin{bmatrix}}
\newcommand{\emat}{\end{bmatrix}}

\DeclareMathAlphabet\mathbfcal{OMS}{cmsy}{b}{n}
\renewcommand{\P}[1]{\operatorname{\mathbb{P}}\left(#1\right)}
\newcommand{\E}{\operatorname{\mathbb{E}}}

\newcommand{\vct}[1]{\boldsymbol{#1}}
\newcommand{\mtx}[1]{\boldsymbol{#1}}

\newcommand{\<}{\langle}
\renewcommand{\>}{\rangle}

\newcommand{\trace}{\operatorname{trace}}

\newcommand{\rank}{\operatorname{rank}}

\newcommand{\set}[1]{\mathbb{#1}}

\newcommand{\wh}{\widehat}

\newcommand{\wt}{\widetilde}
\newcommand{\ol}{\overline}

\newcommand{\calA}{\mathcal{A}}
\newcommand{\calB}{\mathcal{B}}

\newcommand{\calD}{\mathcal{D}}
\newcommand{\calE}{\mathcal{E}}

\newcommand{\calG}{\mathcal{G}}
\newcommand{\calH}{\mathcal{H}}

\newcommand{\calL}{\mathcal{L}}

\newcommand{\calN}{\mathcal{N}}

\newcommand{\calP}{\mathcal{P}}
\newcommand{\calQ}{\mathcal{Q}}

\newcommand{\calS}{\mathcal{S}}
\newcommand{\calT}{\mathcal{T}}
\newcommand{\calU}{\mathcal{U}}
\newcommand{\calV}{\mathcal{V}}
\newcommand{\calW}{\mathcal{W}}
\newcommand{\calX}{\mathcal{X}}
\newcommand{\calY}{\mathcal{Y}}
\newcommand{\calZ}{\mathcal{Z}}

\newcommand{\va}{\vct{a}}

\newcommand{\vc}{\vct{c}}

\newcommand{\ve}{\vct{e}}

\newcommand{\vr}{\vct{r}}

\newcommand{\vu}{\vct{u}}
\newcommand{\vv}{\vct{v}}

\newcommand{\vx}{\vct{x}}
\newcommand{\vy}{\vct{y}}
\newcommand{\vz}{\vct{z}}

\newcommand{\mA}{\mtx{A}}
\newcommand{\mB}{\mtx{B}}

\newcommand{\mD}{\mtx{D}}
\newcommand{\mE}{\mtx{E}}

\newcommand{\mG}{\mtx{G}}
\newcommand{\mH}{\mtx{H}}

\newcommand{\mM}{\mtx{M}}
\newcommand{\mN}{\mtx{N}}

\newcommand{\mQ}{\mtx{Q}}
\newcommand{\mR}{\mtx{R}}
\newcommand{\mS}{\mtx{S}}

\newcommand{\mU}{\mtx{U}}

\newcommand{\mX}{\mtx{X}}
\newcommand{\mY}{\mtx{Y}}

\newcommand{\mDelta}{\mtx{\Delta}}

\newcommand{\mId}{{\bf I}}

\newcommand{\setL}{\set{L}}

\newcommand{\setN}{\set{N}}

\newcommand{\setS}{\set{S}}

\newcommand{\setX}{\set{X}}

\graphicspath{{./figs/}}

\newlength{\imgwidth}
\newboolean{twoColVersion}
\setboolean{twoColVersion}{false}
\newcommand{\twoCol}[2]{\ifthenelse{\boolean{twoColVersion}} {#1} {#2} }

\usepackage{mathtools}

\makeatletter
\@ifundefined{proposition}{\newtheorem{proposition}{Proposition}}{}
\@ifundefined{remark}{\newtheorem{remark}{Remark}}{}
\@ifundefined{corollary}{}{}
\makeatother
\providecommand{\mId}{\mathbf{I}}
\providecommand{\mDelta}{\boldsymbol{\Delta}}
\providecommand{\mD}{\boldsymbol{D}}
\providecommand{\mG}{\boldsymbol{G}}
\providecommand{\mH}{\boldsymbol{H}}
\providecommand{\mM}{\boldsymbol{M}}
\providecommand{\mQ}{\boldsymbol{Q}}
\providecommand{\mR}{\boldsymbol{R}}
\providecommand{\mS}{\boldsymbol{S}}
\providecommand{\mU}{\boldsymbol{U}}

\providecommand{\mB}{\boldsymbol{B}}
\providecommand{\vc}{\boldsymbol{c}}

\providecommand{\vu}{\boldsymbol{u}}
\providecommand{\vv}{\boldsymbol{v}}
\providecommand{\vz}{\boldsymbol{z}}
\providecommand{\ve}{\boldsymbol{e}}
\providecommand{\calG}{\mathcal{G}}
\providecommand{\calH}{\mathcal{H}}
\providecommand{\calT}{\mathcal{T}}
\providecommand{\calW}{\mathcal{W}}
\providecommand{\calZ}{\mathcal{Z}}
\providecommand{\setN}{\mathbb{N}}

\date{}

\title{\LARGE \bf Structured Adaptive Tensor Prediction for Streaming Data}

\author{Zhen Qin  and Yang Chen\thanks{ZQ (email: zhenqin@umich.edu) is with the Michigan Institute for Computational Discovery and Engineering, Department of Statistics, Department of Electrical Engineering and Computer Science,  University of Michigan, Ann Arbor, MI 48109 USA;  and YC (email: ychenang@umich.edu) is with the Department of Statistics, University of Michigan, Ann Arbor, MI 48109 USA.}
}

\begin{document}

\maketitle

\begin{abstract}
Matrix-valued time series arise in a wide range of applications, such as spatio-temporal data in medical imaging and geophysics. Existing forecasting methods are mainly designed for static settings, while adaptive filtering techniques have been largely limited to scalar- or vector-valued data. In this paper, we develop an adaptive tensor prediction framework for streaming matrix-valued data based on a Tensor-on-Matrix (ToM) regression model, in which the responses in a sliding window are lifted to a third-order tensor and the regression tensor is updated online by stochastic gradient descent (SGD). We derive the leading terms of the steady-state and tracking errors of SGD with overlapping windows, which show that the steady-state error is proportional to the step size and the noise variance, and that the step size governs a trade-off between noise and tracking. We further analyze the statistical accuracy of structured ToM regression within a single window. We establish a restricted isometry property, which holds when the window length is proportional to the input dimension, and derive recovery error bounds for sparse, low-rank and sparse plus low-rank regression tensors, which show that the recovery error scales with the intrinsic degrees of freedom rather than the ambient dimension. Simulations confirm the theoretical results, and the effectiveness of the structured adaptive updates is validated on the tracking of a drifting quantum gate, citywide crowd-flow prediction, and global total electron content (TEC) forecasting.
\end{abstract}


\section{Introduction}
\label{sec: introduction}

Matrix-valued time series data have attracted increasing attention in a broad range of scientific and engineering domains. Representative examples include global Total Electron Content (TEC) maps in geophysics and space science, which capture coupled solar-temporal dynamics \cite{sun2022matrix,sun2023complete}; spatiotemporal precipitation and pollution fields in environmental monitoring \cite{bagheri2022machine}; multispectral images in remote sensing \cite{wang2023tensor}; brain activity measurements such as functional MRI (fMRI) voxel grids and EEG sensor arrays in neuroscience \cite{chen2014survey}; and origin-destination flow matrices that describe traffic dynamics in large-scale transportation systems \cite{xiong2020dynamic}. The inherent two-dimensional structure of such data, combined with the temporal dependencies across observations, poses fundamental challenges for both modeling and forecasting, and has motivated the development of a rich family of specialized methodologies.

The most straightforward approach is to vectorize each matrix observation and apply the classical Vector Autoregression (VAR) framework \cite{stock2001vector, wang2023regularized}. While being conceptually convenient, this vectorization strategy discards the intrinsic matrix structure of the data and fails to exploit any prior knowledge regarding the relational geometry among the constituent time series, often resulting in severely over-parameterized models that are difficult to estimate reliably. To overcome these limitations, Matrix Autoregression (MAR) \cite{chen2021autoregressive, sun2023matrix} has been proposed as a principled alternative that operates directly on matrix-valued observations, preserving the row and column structures and enabling a more parsimonious parameter representation. Furthermore, when the data naturally form a higher-order tensor -- for instance, by stacking a sequence of matrix observations along the temporal dimension--it is both natural and advantageous to extend MAR to the Tensor Autoregression (TAR) framework \cite{li2021multi, wang2024high}, which simultaneously exploits low-rank tensor geometry to achieve further dimensionality reduction.

Besides the autoregressive models above, the tensor-on-tensor (ToT) regression framework \cite{sun2017store, lock2018tensor, raskutti2019convex, llosa2022reduced, luo2024tensor, qin2025computational} provides another effective approach for predicting matrix-valued time series. In this framework, both the covariate tensor and the response tensor are represented as higher-order arrays, and their relationship is modeled via a tensor coefficient, often equipped with a low-rank structure to ensure statistical identifiability and computational efficiency. In the time-series setting, the covariate tensor is typically constructed from lagged observations of the target series or related auxiliary variables, while the response tensor encodes the target matrix-valued observations over time. Compared with autoregressive models such as MAR and TAR, which model temporal dependencies by applying structured transformations to past observations, the ToT framework instead learns a direct mapping from tensor-valued inputs to outputs. This formulation provides greater modeling flexibility and enables richer cross-dimensional interactions through a unified low-rank tensor regression structure. More broadly, low-rank tensor decompositions have become a standard tool for spatio-temporal data; in transportation, for example, they are used for the imputation of missing traffic data, traffic state prediction and the analysis of travel patterns~\cite{li2024tensor}.

Despite the aforementioned advances, existing approaches share a fundamental limitation: they assume a fixed autoregressive or prediction model trained on a stationary dataset. In practice, however, many systems are inherently dynamic, where incoming observations follow evolving distributions that may gradually deviate from historical patterns \cite{gama2014survey,iwashita2018overview,goldenberg2018survey}. Under such distribution shifts, models trained offline often suffer rapid performance degradation, and naive retraining on accumulated data is both computationally expensive and statistically inefficient, as it fails to exploit structural continuity over time. This mismatch between static model assumptions and dynamic data environments motivates the need for adaptive methodologies that can efficiently incorporate new observations while preserving the learned structure.

This challenge necessitates an online and recursive update mechanism that can adapt to evolving system dynamics in a streaming setting. A natural approach is to leverage adaptive filter algorithms (AFAs) from adaptive signal processing~\cite{haykin2008adaptive,sayed2011adaptive}. AFAs constitute a class of data-driven methods that recursively update filter coefficients based on streaming observations. Representative algorithms include the least mean squares (LMS)~\cite{hoff1960adaptive}, affine projection algorithm (APA)~\cite{ozeki1984adaptive}, recursive least squares (RLS)~\cite{plackett1950some}, and the Kalman filter~\cite{kalman1960new}. Beyond these classical methods, numerous variants have been developed to address specific structural and noise characteristics. For instance, sparse AFAs~\cite{duttweiler2002proportionate,benesty2002improved,chen2009sparse,gu2009lnorm,das2016improving,eksioglu2011rls,qin2020proportionate,qin2022proportionate,wang2024variable} exploit sparsity in the filter coefficients to enhance convergence and denoising performance. Robust adaptive filters~\cite{Chambers1994LMMN,Petrus1999Huber,Yousef2000Sign,Shao2010APS,chen2016generalized,peng2017constrained,huang2017maximum,chen2019maximum,qin2023proportionate} replace the mean-square error criterion with alternative loss functions to improve resilience under non-Gaussian or impulsive noise. In addition, nonlinear adaptive filtering frameworks~\cite{Mandic2009Complex,Xia2010AAPA,Liu2011KAF,Xia2011WidelyLinear,Raja2014BoxJenkins} have been proposed to overcome the inherent limitations of linear models, enabling more expressive representations for complex dynamical systems. Building upon these AFAs and their variants, a wide range of practical applications has been demonstrated, including channel estimation and equalization~\cite{qin2019direct,qin2020sparse,shao2025adaptive}, echo cancellation~\cite{naylor2006adaptive,deng2006proportionate}, speech enhancement~\cite{hu2006robust,jin2009speech}, frequency estimation in power systems~\cite{xia2011widely,xia2012adaptive}, and beamforming~\cite{hoshuyama2002robust,herbordt2007multichannel}. So-called tensor-based adaptive filters~\cite{rupp2015tensor,dogariu2021tensor} also remain in this vector setting: they reshape the coefficient vector of a long filter into a tensor and exploit its decomposition to reduce the number of parameters, while the input and the output are still a vector and a scalar. Tensor tracking methods, e.g.,~\cite{le2024novel}, adaptively decompose a streaming data tensor itself and do not model a regression from inputs to tensor-valued responses. Despite their success, these methods have been largely restricted to scalar- or vector-valued time series, leaving adaptive forecasting for matrix- or tensor-valued streaming data largely unexplored. This limitation motivates the following central research question:
\vspace{0.2cm}
\smallskip
\begin{mdframed}[linewidth=0.3pt, leftmargin=0mm, rightmargin=0mm,
  innerleftmargin=1mm, innerrightmargin=1mm, innertopmargin=1mm, innerbottommargin=1mm]
\centering {\bf Question}: How can we adaptively forecast streaming matrix- or tensor-valued time series?
\end{mdframed}
\vspace{0.2cm}

In this paper, we address this question within the framework of adaptive tensor prediction. Given auxiliary vector covariates as inputs, we stack the most recent inputs into a matrix and lift the corresponding matrix-valued responses into a third-order tensor, which leads to the Tensor-on-Matrix (ToM) regression model, and the regression tensor is updated online by stochastic gradient descent (SGD) on the squared loss over a sliding window. This work focuses on the theoretical analysis of this framework. The main contributions are summarized as follows.

\begin{itemize}
\item \emph{Adaptive tensor prediction.} We formulate the prediction of streaming matrix-valued responses as an adaptive ToM regression problem, in which the regression coefficient is a third-order tensor and the responses in each window form a tensor. This model can be viewed as an adaptive counterpart of ToT regression~\cite{lock2018tensor}.
\item \emph{Steady-state and tracking analysis.} We derive the leading terms of the steady-state and tracking errors of SGD with overlapping windows (\Cref{The main theorem of error in the model II,The main theorem of tracking ability in the model II}). The results show that the steady-state error is proportional to the step size, the noise variance and the number of entries of the regression tensor, and that the tracking error exhibits the classical trade-off between noise and drift, which determines the optimal step size.
\item \emph{Statistical guarantees for structured ToM regression.} We establish a restricted isometry property (RIP), which holds when the window length is at least proportional to the input dimension (\Cref{RIP condition fro the ToM regression}). Based on the RIP, we derive recovery error bounds for sparse, low-rank (Tucker, tensor train and tubal) and sparse plus low-rank regression tensors (\Cref{theorem:sparse  structure conclusion main paper,theorem: low-rank structure conclusion main paper,theorem:sparse and low-rank structure conclusion main paper}). Compared with the unstructured least-squares estimator formed from the same window, these bounds show that the recovery error scales with the intrinsic degrees of freedom rather than the ambient dimension. Existing theoretical results for ToT regression consider sparse~\cite{raskutti2019convex}, Tucker~\cite{raskutti2019convex,luo2024tensor} and tensor train~\cite{qin2025computational} regression tensors, but none of them covers the sparse plus low-rank model, even in the static setting of a single window.
\end{itemize}
To illustrate these results, we incorporate the structures into the SGD update through projections, as natural extensions of iterative hard thresholding~\cite{blumensath2009iterative,rauhut2017low}, where the sparsity level and the ranks are selected at every window from an unstructured SGD iterate. We evaluate the resulting updates on synthetic data and on three applications: the tracking of a drifting quantum gate, citywide crowd-flow prediction, and the prediction of the global total electron content (TEC) from incomplete maps, in which we compare them with persistence and matrix autoregressive models.

{\bf Notation}: We use calligraphic letters (e.g., $\calY$) to denote tensors,  bold capital letters (e.g., $\mY$) to denote matrices, bold lowercase letters (e.g., $\vy$) to denote vectors, and italic letters (e.g., $y$) to denote scalar quantities.  Elements of matrices and tensors are denoted in parentheses, as in the Matlab notation. For example, $\calY(s_1, s_2, s_3)$ denotes the element in position $(s_1, s_2, s_3)$ of the $3$-order tensor $\calY$.
The inner product of $\calA\in\R^{d_1\times\dots\times d_N}$ and $\calB\in\R^{d_1\times\dots\times d_N}$ is denoted as $\<\calA, \calB \> = \sum_{s_1=1}^{d_1}\cdots \sum_{s_N=1}^{d_N} \calA(s_1,\dots,s_N)\calB(s_1,\dots,s_N) $.
The vectorization of  $\calX\in\R^{d_1\times\dots\times d_N}$, denoted as $\text{vec}(\calX)$, transforms the tensor $\calX$ into a vector. The $(s_1, \dots, s_N)$-th element of $\calX$ can be found in the vector $\text{vec}(\calX)$ at the position $s_1 + d_1(s_2-1) + \cdots + d_1d_2 \cdots d_{N-1}(s_N-1)$.  $\|\calX\|_F = \sqrt{\<\calX, \calX \>}$ is the Frobenius norm of $\calX$. $\|\mX\|_F$ represents the Frobenius norm of $\mX$. $\|\vx\|_2$ denotes the $l_2$ norm of $\vx$. $\|\mX\|$ denotes the spectral norm of $\mX$.
For two positive quantities $a,b\in \real$,  $b = O(a)$ means $b\leq c a$ for some universal constant $c$; likewise, $b = \Omega(a)$ represents $b\ge ca$ for some universal constant $c$.

\section[Adaptive Tensor-on-Matrix Regression Model]{Adaptive Tensor-on-Matrix Regression Model}

As discussed in \Cref{sec: introduction},  adaptive signal processing methods~\cite{haykin2008adaptive, sayed2011adaptive} have long provided efficient frameworks for online parameter estimation from streaming data. However, these classical approaches are designed for scalar- or vector-valued signals and do not readily extend to settings where both the input and output are matrix or tensor-valued. To bridge this gap, we study an online regression problem where the goal is to predict a matrix-valued response from streaming vector inputs. Specifically, we seek to learn an unknown low-dimensional mapping $\calX^\star$ from sequentially observed data. Such mappings naturally arise in spatio-temporal applications, where observations are often high-dimensional, structured, and collected over time.
Let $\{\va_i,\mY_i\}$ denote a sequence of streaming samples, where $\va_i\in\R^{d_3}$ represents the input feature vector and $\mY_i\in\R^{d_1\times d_2}$ denotes the corresponding matrix-valued response. Since the response at time $i$ may depend not only on the current input but also on recent observations, we incorporate temporal context into the regression model. To this end, we stack the most recent $d_4$ input vectors into a matrix
\begin{eqnarray}
    \label{stacked input signal}
    \mA_i = \begin{bmatrix} \va_i & \va_{i-1} & \cdots & \va_{i-d_4+1}    \end{bmatrix} \in\R^{d_3\times d_4}.
\end{eqnarray}
Here, $\mA_i$ represents the measurement operator, and the window moves forward by one sample at each time instant. Since $\mA_i$ is matrix-valued, the regression coefficient relating $\mA_i$ to the responses in the window is naturally represented by a higher-order tensor that encodes the underlying multilinear mapping. This motivates the framework of \emph{adaptive tensor regression}, in which a tensor-valued coefficient is updated incrementally as new observations arrive. In this work, we lift the matrix-valued responses in the window to a third-order tensor, which leads to the Tensor-on-Matrix (ToM) regression model with a third-order regression tensor $\calX^\star\in\R^{d_1\times d_2\times d_3}$.

\paragraph[Tensor-on-Matrix (ToM) Model: Jointly Predicting a Sequence of Observations]{Tensor-on-Matrix (ToM) Model: Jointly Predicting a Sequence of Observations}
 To leverage the temporal information contained in $\mA_i$, we take as prediction target the sequence of $d_4$ consecutive observations $\{\mY_j\}_{j=i-d_4+1}^{i}$. We stack these outputs into a third-order tensor $\calY_i \in \R^{d_1 \times d_2 \times d_4}$, where $\calY_i(:,:,s_4) = \mY_{i-s_4+1}$. The observation $\mY_{i-s_4+1}$ is generated at time $i-s_4+1$ by the input $\va_{i-s_4+1}=\mA_i(:,s_4)$ and the regression tensor at that time, so the ToM regression model takes the form
\begin{eqnarray}
    \label{ToM regression general data model II}
    \calY_i(:,:,s_4) = \calX_{i-s_4+1}^\star \times_{3}^{1} \mA_i(:,s_4) + \calE_i(:,:,s_4),\qquad s_4=1,\dots,d_4,
\end{eqnarray}
where $\calX_j^\star \in \R^{d_1 \times d_2 \times d_3}$ is the regression tensor at time $j$, $\calE_i\in\R^{d_1\times d_2\times d_4}$ denotes noise, and, for a vector $\va\in\R^{d_3}$, $(\calX\times_3^1\va)(s_1,s_2)=\sum_{s_3=1}^{d_3}\calX(s_1,s_2,s_3)\va(s_3)$. Since the window moves forward by one observation at each time instant, consecutive windows share $d_4-1$ samples, i.e., for $s_4=2,\dots,d_4$,
\begin{eqnarray}
    \label{window overlap}
    \mA_i(:,s_4)=\mA_{i-1}(:,s_4-1),\qquad \calY_i(:,:,s_4)=\calY_{i-1}(:,:,s_4-1),\qquad \calE_i(:,:,s_4)=\calE_{i-1}(:,:,s_4-1),
\end{eqnarray}
so that each observation $\mY_j$, together with its noise $\calE_j(:,:,1)$, enters the $d_4$ consecutive windows $i=j,\dots,j+d_4-1$, and $\calE_i(:,:,s_4)=\calE_{i-s_4+1}(:,:,1)$. Given the current estimate $\calX_i$, the prediction of $\calY_i$ is $\calX_i\times_3^1\mA_i$, where the contraction $\times_{3}^{1}$ couples the third mode of $\calX_i$ with the first mode of $\mA_i$:
\begin{eqnarray}
    \label{contraction ToM}
    (\calX_i \times_{3}^{1} \mA_i)(s_1,s_2,s_4) =
    \sum_{s_3=1}^{d_3} \calX_i(s_1,s_2,s_3)\, \mA_i(s_3,s_4),
\end{eqnarray}
so that its $s_4$-th slice $\calX_i\times_3^1\mA_i(:,s_4)$ predicts $\calY_i(:,:,s_4)$ and the third mode indexes the temporal observations. The estimator solves
\begin{eqnarray}
    \label{loss function of X another model II}
    \min_{\calX_i\in\R^{d_1\times d_2 \times d_3}} f(\calX_i) =  \frac{1}{2d_4}\|\calX_i \times_{3}^{1} \mA_i -   \calY_i  \|_F^2,
\end{eqnarray}
where $\|\calX_i\times_3^1\mA_i-\calY_i\|_F^2=\sum_{s_4=1}^{d_4}\|\calX_i\times_3^1\mA_i(:,s_4)-\calY_i(:,:,s_4)\|_F^2$, i.e., the current estimate $\calX_i$ is fitted to all $d_4$ observations in the window, and the $1/d_4$ normalization ensures comparability of the loss scale across different window sizes. Since $\calY_i(:,:,s_4)$ is generated by $\calX^\star_{i-s_4+1}$, a time variation of the regression tensor within the window introduces a lag, which is quantified in \Cref{sec:tracking ability performance}. The corresponding stochastic gradient descent (SGD) update is given by
\begin{eqnarray}
    \label{updating method X another model II}
    \calX_{i+1} = \calX_i - \frac{\mu}{d_4} (\calX_i \times_{3}^{1} \mA_i -   \calY_i) \times_{3}^{2} \mA_i,
\end{eqnarray}
where $\mu>0$ is the step size and $(\calT\times_3^2\mA_i)(s_1,s_2,s_3)=\sum_{s_4=1}^{d_4}\calT(s_1,s_2,s_4)\mA_i(s_3,s_4)$ for $\calT\in\R^{d_1\times d_2\times d_4}$. Each update jointly fits the $d_4$ observations in the window. Unlike mini-batch SGD, however, the windows overlap, so each sample enters $d_4$ consecutive updates with weight $\mu/d_4$ each, i.e., with total weight $\mu$. As shown next, this reuse determines how the window size $d_4$ affects the steady-state and tracking errors. Moreover, without structural constraints, \eqref{updating method X another model II} is the SGD update of a multi-output linear regression over the sliding window: each mode-3 fiber $\calX_i(s_1,s_2,:)$ is updated separately, as a linear filter that predicts the $d_4$ responses $\calY_i(s_1,s_2,:)$ from the inputs $\mA_i$. Hence, the tensor form by itself does not change the estimator, and an improvement in accuracy can only come from low-dimensional structure shared across the fibers, such as sparsity or low-rankness of $\calX^\star$, which we quantify in \Cref{sec: low-dimensional ToM regression}.

\subsection{Steady-state Performance}
\label{sec:steady-state performance}

A natural question is whether the temporal context of the window translates into a measurable improvement in estimation accuracy. To address this, we analyze the steady-state behavior of the SGD update \eqref{updating method X another model II} with the overlapping windows \eqref{window overlap}, characterizing the limiting mean squared error as a function of the step size $\mu$, the noise level $\sigma_e^2$, the window size $d_4$ and the problem dimensions. For notational convenience, we index the observations so that the first window is complete at $i=0$, i.e., $\{\va_j,\mY_j\}$ are available for $j\ge-d_4+1$, and \eqref{updating method X another model II} is applied for $i=0,1,2,\dots$ from a deterministic initialization $\calX_0$. We establish the results under the following assumptions.
\begin{enumerate}[(i)]
\item \emph{Inputs.} The input vectors $\{\va_i\}$ are independent and identically distributed (i.i.d.) with zero mean and covariance matrix $\E[\va_i\va_i^\top]=\sigma_a^2\mId_{d_3\times d_3}$.
\item \emph{Noise.} The entries of the noise slices $\{\calE_i(:,:,1)\}$ are i.i.d.\ with zero mean and variance $\sigma_e^2$, and are independent of the inputs $\{\va_i\}$.
\item \emph{Static regression tensor.} The regression tensor is constant, i.e., $\calX_i^\star=\calX^\star$ for all $i$.
\item \emph{Bounded inputs.} There exists a constant $b_a>0$ such that $\|\va_i\|_2^2\le b_a$ almost surely.
\end{enumerate}
Assumption (iv) is a technical condition that is used in the proofs only to bound the higher-order terms; accordingly, the constant $b_a$ does not enter the leading terms of the results below. The same boundedness condition is required in the averaging analysis of adaptive filters, which dispenses with the independence assumption~\cite[Appendix~24.A]{sayed2011adaptive}. It is satisfied, for instance, by inputs with i.i.d.\ entries drawn uniformly from $\{-\sigma_a,\sigma_a\}$, for which $b_a=d_3\sigma_a^2$, and by spherically truncated Gaussian inputs. We emphasize that, by \eqref{window overlap}, consecutive windows share $d_4-1$ input vectors and noise slices. Consequently, $\{\mA_i\}$ and $\{\calE_i\}$ are dependent across $i$, and $\calX_i$ depends on $d_4-1$ of the $d_4$ columns of $\mA_i$; both forms of dependence are taken into account in the analysis.
\begin{theorem} (Steady-state error of ToM regression model)
\label{The main theorem of error in the model II}
Under assumptions (i)--(iv), for every $d_4\ge1$ and every step size satisfying
\begin{eqnarray}
    \label{mu0 definition}
    0<\mu\le\mu_0,\qquad \mu_0=\min\Big\{\frac{1}{2d_4\sigma_a^2},\ \frac{1}{64b_a},\ \frac{\sigma_a^2}{2d_4b_a^2}\Big\},
\end{eqnarray}
there exists a constant $C'>0$, which depends on $d_1,\dots,d_4$, $\sigma_a^2$, $\sigma_e^2$ and $b_a$ but not on $\mu$, such that
\begin{eqnarray}
    \label{steady state performance in the ToM regression model main theorem}
    \limsup_{i\to\infty}\bigg|\E[\|\calX_i-\calX^\star\|_F^2]-\frac{\mu d_1d_2d_3\sigma_e^2}{2}\bigg|\le C'\mu^{3/2}.
\end{eqnarray}
Here, the expectation is taken with respect to the joint distribution of the streaming inputs and observation noise, which constitute the sources of randomness driving the SGD iterates $\{\calX_i\}$.
\end{theorem}
The proof is provided in {Appendix}~\ref{Proof of error analysis SGD of model II}. \Cref{The main theorem of error in the model II} shows that the steady-state error is determined by the step size $\mu$, the noise variance $\sigma_e^2$ and the number of entries $d_1d_2d_3$ of the regression tensor: its leading term $\frac{\mu d_1d_2d_3\sigma_e^2}{2}$ is proportional to each of them. Since the leading term is of order $\mu$ and the remainder is of order $\mu^{3/2}$, the remainder is negligible compared with the leading term when $\mu$ is small. Hence, for small step sizes, the steady-state error is accurately approximated by $\frac{\mu d_1d_2d_3\sigma_e^2}{2}$.

\paragraph{Simulation}
 In this part, we verify \Cref{The main theorem of error in the model II}. The ground-truth tensor $\calX^\star$ in \eqref{ToM regression general data model II} is generated from a standard Gaussian distribution, the entries of the inputs $\va_i$ are drawn uniformly from $\{-\sigma_a,\sigma_a\}$, so that (iv) holds, and the noise is Gaussian. The SGD update \eqref{updating method X another model II} with the overlapping windows \eqref{window overlap} is initialized with the zero tensor. For each step size $\mu$, we run $\lceil10/(\mu\sigma_a^2)\rceil$ iterations to reach the steady state, and then average the error over the next $\lceil10/(\mu\sigma_a^2)\rceil$ iterations and over $20$ independent Monte Carlo trials. As shown in \Cref{steady state diff sigma_e}, the simulated steady-state error agrees with the leading term $\frac{\mu d_1d_2d_3\sigma_e^2}{2}$ of \Cref{The main theorem of error in the model II} over the whole range of $\mu$, and it grows linearly with both $\mu$ and $\sigma_e^2$. Moreover, as shown in \Cref{steady state diff d4}, the curves for different $d_4$ coincide, which confirms that the steady-state error does not depend on the window size.

\begin{figure}[htbp]
\centering
\subfigure[$d_4=5$]{
\begin{minipage}[t]{0.45\textwidth}
\centering
\includegraphics[width=5.5cm]{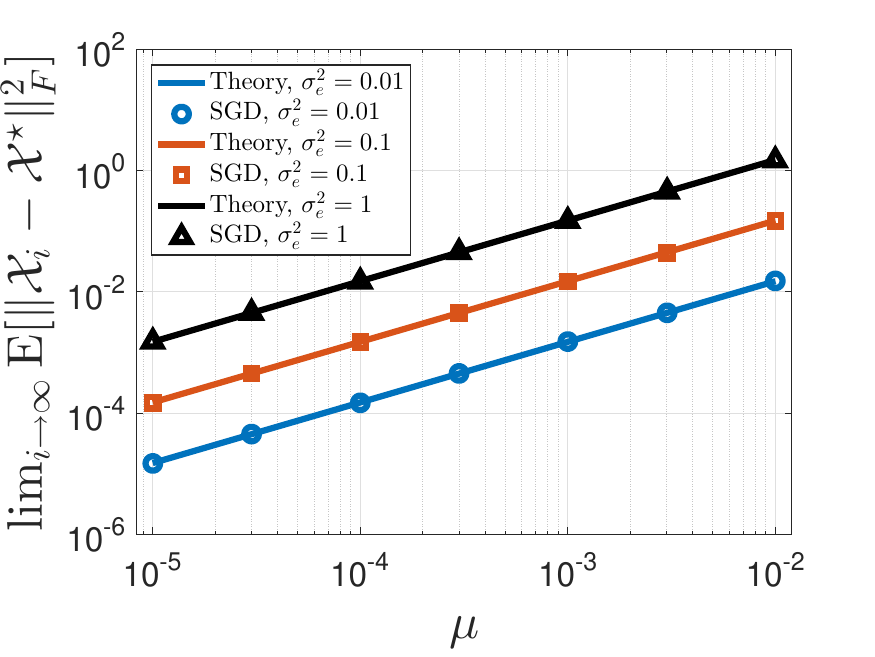}
\end{minipage}
\label{steady state diff sigma_e}
}
\subfigure[$\sigma_e^2=0.1$]{
\begin{minipage}[t]{0.45\textwidth}
\centering
\includegraphics[width=5.5cm]{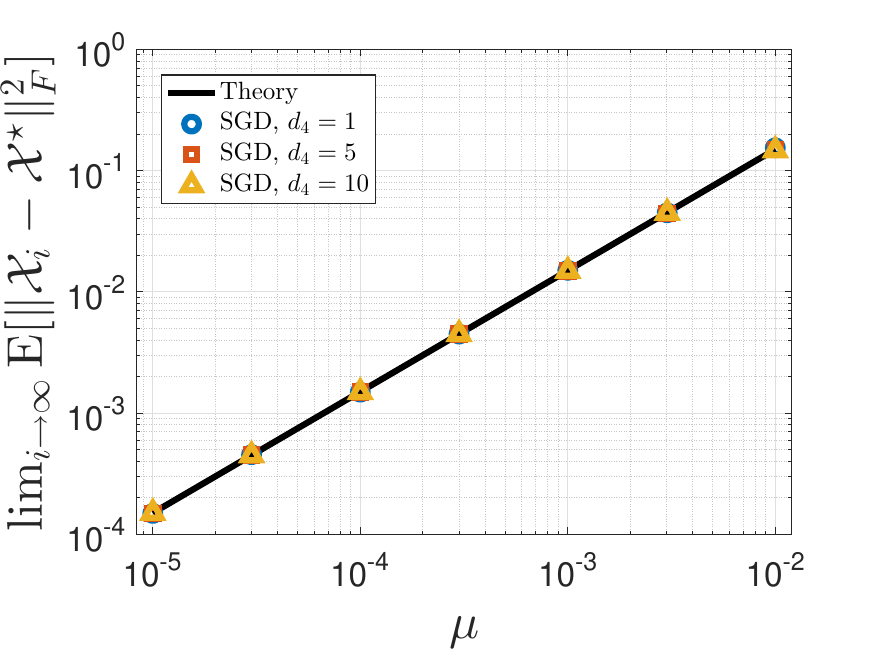}
\end{minipage}
\label{steady state diff d4}
}
\caption{Steady-state errors of the ToM regression model versus the step size $\mu$: (a) for different $\sigma_e^2$ with $d_4=5$, and (b) for different $d_4$ with $\sigma_e^2=0.1$, where $d_1=d_2=10$, $d_3=3$ and $\sigma_a^2=1$.}
\label{steady state of ToM regression model}
\end{figure}

\subsection{Tracking Ability of Tensor-on-Matrix Regression}
\label{sec:tracking ability performance}

In adaptive signal processing, the ability of adaptive algorithms to track a drifting system is a fundamental performance criterion, complementary to steady-state error~\cite{haykin2008adaptive,sayed2011adaptive}. To model temporal variations in the ground-truth tensor, we adopt the first-order autoregressive (AR(1)) model
\begin{eqnarray}
\label{RW model definition}
\calX_i^\star = (1-\lambda)\calX_{i-1}^\star + \lambda\calQ_i,\qquad i\ge1,
\end{eqnarray}
where $0 < \lambda < 1$ is a constant, each entry of $\calQ_i \in \R^{d_1\times d_2 \times d_3}$ is i.i.d.\ Gaussian with zero mean and variance $\sigma_q^2$, and the initial tensor $\calX_0^\star$ satisfies $\E[\|\calX_0^\star\|_F^2]<\infty$. Since $0<1-\lambda<1$, \eqref{RW model definition} is mean reverting: $\calX_i^\star$ is pulled towards zero at rate $\lambda$ while being driven by $\calQ_i$; iterating it gives
\begin{eqnarray}
    \label{AR solution}
    \calX_i^\star=(1-\lambda)^i\calX_0^\star+\lambda\sum_{n=0}^{i-1}(1-\lambda)^n\calQ_{i-n},\qquad i\ge0.
\end{eqnarray}
The parameter $\lambda$ controls the rate of temporal variation: smaller values of $\lambda$ correspond to slowly varying systems with strong temporal correlation, whereas larger values of $\lambda$ lead to faster system evolution and a more challenging tracking task. By \eqref{ToM regression general data model II}, the window $\calY_i$ contains responses generated by $\calX_i^\star,\dots,\calX_{i-d_4+1}^\star$, which the analysis accounts for explicitly. Since the first window is complete at $i=0$, the first $d_4-1$ windows also contain $\calX^\star_{-d_4+1},\dots,\calX^\star_{-1}$; these are arbitrary random tensors with finite second moments (e.g., all equal to $\calX_0^\star$), and they enter only finitely many updates. We assume that $\{\calQ_i\}_{i\ge1}$, $\calX_0^\star$ and $\calX^\star_{-d_4+1},\dots,\calX^\star_{-1}$ are mutually independent and independent of $\{\va_i\}$ and $\{\calE_i(:,:,1)\}$; the AR(1) model \eqref{RW model definition} then replaces (iii). Under this setting, the following theorem characterizes the steady-state tracking error of the SGD-based estimator.
\begin{theorem} (Steady-state tracking error of ToM regression model)
\label{The main theorem of tracking ability in the model II}
Under assumptions (i), (ii) and (iv), and the AR(1) model \eqref{RW model definition}, for every $d_4\ge1$ and every step size $0<\mu\le\mu_0$ with $\mu_0$ given in \eqref{mu0 definition}, there exists a constant $C''>0$, which depends on $d_1,\dots,d_4$, $\sigma_a^2$, $\sigma_e^2$ and $b_a$ but not on $\mu$, $\lambda$ and $\sigma_q$, such that
\begin{eqnarray}
\label{tracking ability performance in the ToM regression model main theorem}
    \limsup_{i\to\infty}\bigg|\E[\|\calX_i-\calX_i^\star\|_F^2]-\frac{\mu d_1d_2d_3\sigma_e^2}{2}-\frac{d_1d_2d_3\lambda^2\sigma_q^2}{2(\lambda+\mu\sigma_a^2)}\bigg|\le C''\sqrt{\mu}\Big(\mu+\frac{\lambda^2\sigma_q^2}{\mu}\Big).
\end{eqnarray}
\end{theorem}
The proof is provided in Appendix~\ref{Proof of tracking ability of SGD of model II}. The tracking error consists of the steady-state error $\frac{\mu d_1d_2d_3\sigma_e^2}{2}$ of \Cref{The main theorem of error in the model II} and the drift term $\frac{d_1d_2d_3\lambda^2\sigma_q^2}{2(\lambda+\mu\sigma_a^2)}$, which is the cost of the variation of the regression tensor. The remainder in \eqref{tracking ability performance in the ToM regression model main theorem} is small relative to the leading terms when $\mu$ is small and $\lambda\lesssim\mu\sigma_a^2$, i.e., when the regression tensor varies slowly compared with the adaptation rate of the algorithm. When $\lambda=0$, the drift term vanishes and \Cref{The main theorem of error in the model II} is recovered. The two leading terms exhibit the classical trade-off between noise and tracking: a larger step size increases the noise term but reduces the drift term. If $\sigma_q\sigma_a>\sigma_e$, their sum is minimized at $\mu^\star=\frac{\lambda(\sigma_q\sigma_a-\sigma_e)}{\sigma_e\sigma_a^2}$; otherwise, the drift is weak compared with the noise, the sum is increasing in $\mu$, and smaller step sizes are better.

\paragraph{Simulation}
 In this part, we verify \Cref{The main theorem of tracking ability in the model II}. The regression tensor follows the AR(1) model \eqref{RW model definition}, with $\calX_0^\star$ drawn from its stationary distribution, and the remaining settings are the same as in \Cref{steady state of ToM regression model} with $d_4=5$ and $\sigma_e^2=0.1$. As shown in \Cref{tracking of ToM regression model lambda,tracking of ToM regression model sigma_q}, the simulated tracking error agrees with the leading terms in \eqref{tracking ability performance in the ToM regression model main theorem} over the whole range of $\mu$. The curves exhibit the trade-off between noise and tracking: the tracking error first decreases and then increases with $\mu$. Furthermore, increasing either $\lambda$ or $\sigma_q^2$ leads to a larger tracking error, and the increase is most pronounced for small step sizes, for which the drift term dominates.

\begin{figure}[htbp]
\centering
\subfigure[$\sigma_q^2=1$]{
\begin{minipage}[t]{0.45\textwidth}
\centering
\includegraphics[width=5.5cm]{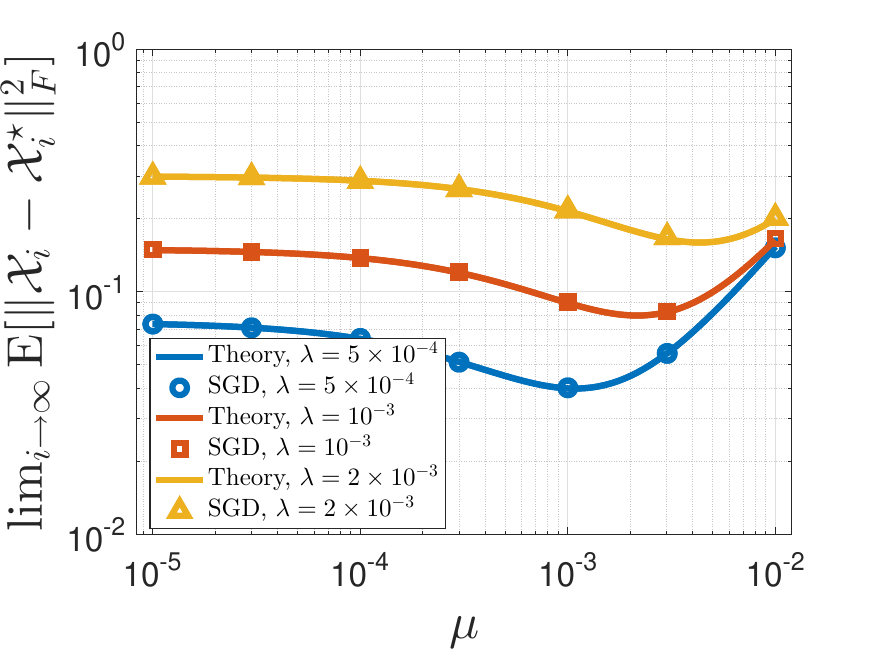}
\end{minipage}
\label{tracking of ToM regression model lambda}
}
\subfigure[$\lambda=10^{-3}$]{
\begin{minipage}[t]{0.45\textwidth}
\centering
\includegraphics[width=5.5cm]{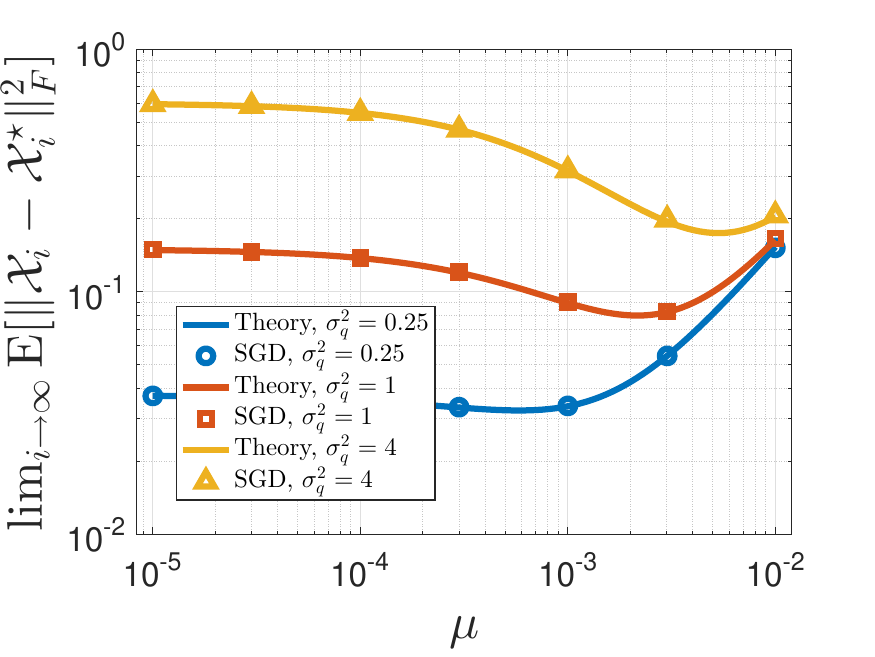}
\end{minipage}
\label{tracking of ToM regression model sigma_q}
}
\caption{Steady-state tracking errors of the ToM regression model versus the step size $\mu$: (a) for different $\lambda$ with $\sigma_q^2=1$, and (b) for different $\sigma_q^2$ with $\lambda=10^{-3}$, where $d_1=d_2=10$, $d_3=3$, $d_4=5$, $\sigma_a^2=1$ and $\sigma_e^2=0.1$.}
\end{figure}

\section{Statistical Guarantees of Low-dimensional Adaptive ToM Regression}
\label{sec: low-dimensional ToM regression}

The steady-state and tracking analyses in Sections~\ref{sec:steady-state performance} and~\ref{sec:tracking ability performance} characterize the long-run error of the online SGD update \eqref{updating method X another model II} under general, unstructured settings; its accuracy is governed by the step size $\mu$, which sets the effective memory of the algorithm. In this section, we adopt a complementary, algorithm-independent perspective: fixing a time instant $i$ and treating the estimation of $\calX_i^\star$ from the $d_4$ measurements $\calY_i$ encoded in $\mA_i$ as a structured optimization problem, we ask how the intrinsic low-dimensional structure of $\calX_i^\star$ affects the statistical accuracy of recovery.

In practice, the ground-truth tensor $\calX_i^\star$ may exhibit low-dimensional structure at certain time instants, which substantially reduces its effective degrees of freedom. We consider three canonical structural models: \emph{sparsity}, which restricts the number of nonzero entries; \emph{low-rankness}, which exploits global multilinear correlations across modes; and their combination, \emph{sparse plus low-rank}, which simultaneously captures both effects. In each case, the recovery error is shown to depend not on the ambient dimension $d_1 d_2 d_3$ but on a much smaller intrinsic complexity measure, yielding guarantees that improve substantially over the unstructured baseline formed from the same window. Our analysis is centered on the sparse plus low-rank model, which contains the sparse and the low-rank models as special cases. Recovery guarantees for sparse, Tucker and tensor train regression tensors have been established for ToT regression~\cite{raskutti2019convex,luo2024tensor,qin2025computational}, whereas the sparse plus low-rank model has not been analyzed. Since the structural properties of the underlying tensor may vary over time~\cite{jiang2026fast} and, in general, cannot be tracked exactly, we analyze the recovery problem at a representative time instant $i$ for which $\calX_i^\star$ admits the assumed low-dimensional structure, and we assume that the regression tensor does not change within the window,
\begin{eqnarray}
    \label{snapshot assumption}
    \calX^\star_{i-s_4+1}=\calX_i^\star,\qquad s_4=1,\dots,d_4,
\end{eqnarray}
while the structure may change from one window to another. This is consistent with the analyses of sparsity-aware adaptive filters, where the sparse structure of the unknown system is likewise assumed to be fixed~\cite{su2012performance,qin2022proportionate,wang2024variable}. The resulting guarantees should therefore be interpreted as instantaneous recovery bounds. In addition to (i)--(ii) in \Cref{sec:steady-state performance}, we assume in this section that the noise entries are Gaussian, $\calE_i(s_1,s_2,s_4)\sim\calN(0,\sigma_e^2)$, and that the inputs are sub-Gaussian, $\|\va_i/\sigma_a\|_{\psi_2}\le K$ with $K\ge1$, where $\|\va\|_{\psi_2}=\sup_{\|\vx\|_2=1}\|\<\va,\vx\>\|_{\psi_2}$; both hold for Gaussian inputs and noise. Within one window, the noise slices $\calE_i(:,:,s_4)=\calE_{i-s_4+1}(:,:,1)$ and the inputs $\mA_i(:,s_4)=\va_{i-s_4+1}$ come from $d_4$ distinct time instants, so the entries of $\calE_i$ are i.i.d., the columns of $\mA_i$ are independent, and $\calE_i$ is independent of $\mA_i$; this holds for every window, although consecutive windows overlap. For notational simplicity, we drop the time index and write $\calX$, $\calX^\star$, $\mA$, $\calE$, and $\calY$ in place of $\calX_i$, $\calX_i^\star$, $\mA_i$, $\calE_i$, and $\calY_i$. We begin with the following optimization problem:
\begin{eqnarray}
    \label{loss function of X another model II general optimization}
    \min_{\calX\in \setX }\frac{1}{2d_4}\|\calX \times_{3}^{1}
    \mA -   \calY  \|_F^2,
\end{eqnarray}
where $\setX$ denotes the structured tensor space corresponding to the assumed low-dimensional model.

The central task is then to recover low-dimensional $\calX^\star$ from its linear measurements $\calY = \calX^\star \times_{3}^{1} \mA + \calE$ via the ToM regression. A key property that facilitates accurate recovery is the \emph{restricted isometry property} (RIP), which plays a fundamental role in compressive sensing and high-dimensional estimation~\cite{donoho2006compressed, candes2006robust, candes2008introduction, recht2010guaranteed}. Intuitively, the RIP ensures that the measurement operator preserves the Euclidean geometry of structured signals, thereby preventing information loss during projection. Recent advances have further extended RIP-based analyses to tensor recovery problems, highlighting its importance in low-dimensional tensor settings~\cite{rauhut2017low, qin2024quantum, qin2025robust, qin2025computational}. The following theorem establishes the RIP.
\begin{theorem}
\label{RIP condition fro the ToM regression}
Suppose that $\{\va_i\}$ satisfy (i) and $\|\va_i/\sigma_a\|_{\psi_2}\le K$ with $K\ge1$. Let $\delta_{d_3}\in(0,1)$. There exist universal constants $C,c>0$, which do not depend on any problem parameter, such that, if
\begin{equation}
d_4 \ge CK^4 \frac{d_3}{\delta_{d_3}^2},
\label{eq:mrip ToM single}
\end{equation}
then, with probability at least $1-2e^{-c\delta_{d_3}^2d_4/K^4}$, for all $\calX\in\R^{d_1\times d_2\times d_3}$,
\begin{eqnarray}
    \label{RIP condition fro the ToM regression single}
    (1-\delta_{d_3})\|\calX\|_F^2\leq \frac{1}{d_4\sigma_a^2}\|\calX \times_{3}^{1} \mA\|_F^2\leq(1+\delta_{d_3})\|\calX\|_F^2 .
\end{eqnarray}
\end{theorem}
The proof is provided in {Appendix}~\ref{Proof of RIP}. The condition $K\ge1$ is no loss of generality, since $\|\<\va/\sigma_a,\vx\>\|_{\psi_2}$ is bounded below by an universal constant times $\|\<\va/\sigma_a,\vx\>\|_{L_2}=1$, and $K$ is an universal constant for Gaussian inputs. The proof shows that \eqref{RIP condition fro the ToM regression single} holds for all $\calX$ simultaneously if and only if the sample covariance $\frac{1}{d_4\sigma_a^2}\mA\mA^\top$ of the window is within $\delta_{d_3}$ of the identity in spectral norm; in particular, the RIP holds for all tensors, not only for structured ones. \Cref{RIP condition fro the ToM regression} ensures a stable embedding for tensors, with the energy $\frac{1}{d_4\sigma_a^2}\|\calX \times_{3}^{1} \mA\|_F^2$ remaining close to $\|\calX\|_F^2$. Beyond this single-tensor perspective, the RIP also ensures separation between different tensors. For any two distinct tensors $\calX_1,\calX_2$, the RIP condition ensures distinct responses $\calX_1 \times_{3}^{1} \mA$ and $\calX_2 \times_{3}^{1} \mA$ as
\begin{eqnarray}
\frac{1}{d_4\sigma_a^2}\|\calX_1 \times_{3}^{1} \mA - \calX_2 \times_{3}^{1} \mA\|_F^2  =  \frac{1}{d_4\sigma_a^2}\|(\calX_1 - \calX_2) \times_{3}^{1} \mA \|_F^2 \ge   (1-\delta_{d_3})\|\calX_1 - \calX_2\|_F^2,
\end{eqnarray}
which indicates that their distance is faithfully preserved. Such a property is precisely what enables exact recovery in noise-free settings.
Another key implication concerns the required sample size. Instead of depending on the full ambient dimension $d_1 d_2 d_3$, the number of observations $d_4$ only grows on the order of $d_3$. This scaling behavior can be explained by noting that $\calX \times_{3}^{1} \mA$ acts on each mode-3 fiber of $\calX$ through the same $d_3\times d_4$ matrix $\mA$, and thus the measurement complexity is governed by $d_3$. The following proposition shows that the condition $d_4\gtrsim d_3$ cannot be removed, even for rank-one tensors.
\begin{proposition}
\label{prop:necessity}
If $d_4<d_3$, there exists $\vz\neq{\bf 0}$ such that the rank-one tensor $\calX=\vu\circ\vv\circ\vz$, i.e., $\calX(s_1,s_2,s_3)=\vu(s_1)\vv(s_2)\vz(s_3)$, satisfies $\calX\times_3^1\mA=0$ for all $\vu\in\R^{d_1}$ and $\vv\in\R^{d_2}$. Hence, the lower bound in \eqref{RIP condition fro the ToM regression single} fails for every $\delta_{d_3}<1$, even when restricted to rank-one tensors.
\end{proposition}
The proof is provided in {Appendix}~\ref{Proof of necessity}. Hence, when $d_4<d_3$, the fixed-window guarantees below do not apply to a single window, and the online estimator instead relies on the information accumulated over many windows.

\paragraph[Unstructured baseline]{Unstructured baseline}
 To quantify the benefit of structural priors, we compare the structured estimators below with the unstructured least-squares estimator formed from the same window, $\wh\calX_{\text{LS}}=\operatorname*{argmin}_{\calX\in\R^{d_1\times d_2\times d_3}}\frac{1}{2d_4}\|\calX\times_3^1\mA-\calY\|_F^2$.
\begin{proposition}
\label{prop:unstructured baseline}
Suppose that $\mA$ satisfies \eqref{RIP condition fro the ToM regression single} with $\delta_{d_3}\in(0,1)$ and that the entries of $\calE$ are i.i.d.\ $\calN(0,\sigma_e^2)$ and independent of $\mA$. Then $\wh\calX_{\text{LS}}$ is unique and
\begin{eqnarray}
    \label{unstructured baseline bound}
    \frac{d_1d_2d_3\sigma_e^2}{(1+\delta_{d_3})d_4\sigma_a^2}\le\E\big[\|\wh\calX_{\text{LS}}-\calX^\star\|_F^2\,\big|\,\mA\big]\le\frac{d_1d_2d_3\sigma_e^2}{(1-\delta_{d_3})d_4\sigma_a^2}.
\end{eqnarray}
\end{proposition}
The proof is provided in {Appendix}~\ref{Proof of baseline}. Thus, without structure, the recovery error from one window scales as $\sqrt{d_1d_2d_3/d_4}\,\sigma_e/\sigma_a$; the factor $1/d_4$ reflects the $d_4$ samples in the window. This is the baseline for the structured guarantees below, which use the same window. It should not be confused with the steady-state error of the online SGD update in \Cref{The main theorem of error in the model II}, whose leading term $\frac{\mu d_1d_2d_3\sigma_e^2}{2}$ is governed by the step size rather than by $d_4$. Indeed, replacing $d_4$ by $\frac{2}{\mu\sigma_a^2}$ in the leading term $\frac{d_1d_2d_3\sigma_e^2}{d_4\sigma_a^2}$ of \eqref{unstructured baseline bound} gives exactly $\frac{\mu d_1d_2d_3\sigma_e^2}{2}$, so the SGD update performs like least squares over the most recent $\frac{2}{\mu\sigma_a^2}$ samples.

\subsection{Sparse Tensors}
\label{subsec:sparse tensors}

One of the most common low-dimensional structures encountered in tensors is \emph{sparsity}, where a large fraction of the tensor entries are exactly zero. Such a structure naturally arises in many high-dimensional applications, for example, when only a limited number of features or interactions are truly relevant. Formally, we say that $\calX^\star = \calS^\star$ is a sparse tensor if $\calS^\star \in \setS_{s}$, where
\begin{eqnarray}
    \label{the set of the sparse tensors}
    \setS_{s} = \{\calS\in\R^{d_1\times d_2 \times d_3}:  \text{supp}(\calS) = s    \},
\end{eqnarray}
and $\text{supp}(\cdot)$ denotes the cardinality of the support set, i.e., the number of nonzero entries. The advantage of exploiting sparsity lies in the substantial reduction of the effective degrees of freedom: rather than estimating all $d_1 d_2 d_3$ parameters, one only needs to recover the $s$ nonzero entries.  Consequently, sparse tensor models serve as a natural and powerful prior in both theoretical analysis and practical applications. We now provide a formal analysis of the recovery performance of
$\wh\calS$, defined as the global minimizer of $\min_{\calS\in \setS_{s}}\frac{1}{2d_4}\|\calS \times_{3}^{1} \mA -   \calY  \|_F^2$.
\begin{theorem}
\label{theorem:sparse  structure conclusion main paper}
Suppose that $s\le d_1d_2d_3/3$, that $\mA$ satisfies \eqref{RIP condition fro the ToM regression single} with $\delta_{d_3}\in(0,1)$ and that the entries of $\calE\in\R^{d_1\times d_2\times d_4}$ are i.i.d.\ $\calN(0, \sigma_e^2)$ and independent of $\mA$. Then, with probability at least $1- e^{-s\log(d_1d_2d_3/s)}$, the solution $\wh\calS$ of \eqref{loss function of X another model II general optimization} with $\setX=\setS_s$ satisfies
\begin{eqnarray}
    \label{upper bound of error final main paper sparse}
    \|\wh\calS  - \calS^\star \|_F\leq O\bigg( \frac{ \sqrt{(1+\delta_{d_3})s\log(\frac{d_1d_2d_3}{ s})}}{\sqrt{(1-\delta_{d_3})^2d_4}}\,\frac{\sigma_e}{\sigma_a} \bigg).
\end{eqnarray}
In particular, \eqref{upper bound of error final main paper sparse} holds with probability at least $1-2e^{-c\delta_{d_3}^2d_4/K^4}-e^{-s\log(d_1d_2d_3/s)}$ under the assumptions of \Cref{RIP condition fro the ToM regression}, where $c$ is the universal constant in \Cref{RIP condition fro the ToM regression}.
\end{theorem}
The proof is provided in {Appendix}~\ref{Proof of structured}. By \eqref{upper bound of error final main paper sparse}, the squared recovery error scales as $O(s\log(d_1d_2d_3/s)/d_4)$, where the effective complexity is governed by $s\log(d_1d_2d_3/s)$ rather than the ambient dimension $d_1d_2d_3$. In particular, when the
sparsity level satisfies $s\log(d_1d_2d_3/s) \ll d_1d_2d_3$, this bound is strictly sharper than the unstructured baseline $\sqrt{d_1d_2d_3/d_4}\,\sigma_e/\sigma_a$ in \Cref{prop:unstructured baseline}, confirming that exploiting sparsity leads to a provable improvement in recovery accuracy.

\subsection{Low-rank Tensors}
\label{subsec:lowrank tensors}

In addition to sparsity, which enforces parsimony by limiting the number of nonzero entries, low-rankness serves as another fundamental low-dimensional structure in tensors. Unlike sparsity, which captures localized simplicity, low-rankness exploits global correlations across modes, enabling a compact representation of high-dimensional data. Specifically, a low-rank tensor can be expressed through a small number of matrix or tensor factors, thereby reducing the effective number of parameters from the ambient dimension $d_1 d_2 d_3$ to a quantity that scales only with the tensor rank. This dramatic reduction not only facilitates efficient storage and computation but also enhances robustness by effectively denoising the underlying signal. Formally, we denote $\calX^\star = \calL^\star$ as a low-rank tensor if $\calL^\star \in \setL_{\vr}$, where
\begin{eqnarray}
    \label{Set of strcutured tensor model low rank}
    \setL_{\vr} = \{ \calL\in\R^{d_1\times d_2 \times d_3}: \text{rank}(\calL) = \vr  \},
\end{eqnarray}
where $\text{rank}(\cdot)$ refers to the tensor rank under a chosen decomposition model. Compared with sparsity, low-rankness provides a complementary mechanism of dimensionality reduction, emphasizing global structure rather than local elimination, and thereby broadening the scope of tractable recovery guarantees.

Before turning to the error analysis, we introduce several representative tensor decompositions that will serve as the foundation for our subsequent discussion.
\begin{definition}[Tucker decomposition]
\label{def:tucker}
For any tensor $\calL\in\R^{d_1\times d_2\times d_3}$, a Tucker decomposition represents each entry as
\begin{eqnarray}
    \label{definition in the Tucker decomposition main paper}
    \calL(s_1,s_2,s_3)=\sum_{i_1=1}^{r_1^\text{tk}}\sum_{i_2=1}^{r_2^\text{tk}}\sum_{i_3=1}^{r_3^\text{tk}}\mU_1(s_1,i_1)\mU_2(s_2,i_2) \mU_3(s_3,i_3)\calB(i_1,i_2,i_3),
\end{eqnarray}
where $\mU_1\in\R^{d_1\times r_1^\text{tk}}$, $\mU_2\in\R^{d_2\times r_2^\text{tk}}$ and $\mU_3\in\R^{d_3\times r_3^\text{tk}}$ are column-wise orthonormal factor matrices and $\calB\in\R^{r_1^\text{tk}\times r_2^\text{tk}\times  r_3^\text{tk}}$ is a core tensor. The $3$-tuple $(r_1^\text{tk},r_2^\text{tk},r_3^\text{tk})$ denotes the multilinear rank of the tensor in the Tucker format. We define the compact form of the Tucker tensor as follows: $\calL= [\![\calB; \mU_1, \mU_2,\mU_3 ]\!]$.
\end{definition}

\begin{definition}[Tensor train decomposition]
\label{def:tt}
For any tensor $\calL\in\R^{d_1\times d_2\times d_3}$, a tensor train (TT) decomposition represents each entry as
\begin{eqnarray}
    \label{definition in the TT decomposition main paper}
    \calL(s_1,s_2,s_3) = \mX_1(s_1,:)\mX_2(:,s_2,:)\mX_3(:,s_3),
\end{eqnarray}
where $\mX_1\in\R^{d_1\times r_1^{\text{tt}}}$, $\mX_2\in\R^{r_1^{\text{tt}} \times d_2\times r_2^{\text{tt}}}$, and $\mX_3\in\R^{r_2^{\text{tt}} \times d_3}$. We denote the TT representation of $\calL$ by $[\mX_1,\mX_2,\mX_3]$ with TT rank $\vr_{\text{tt}}=(r_1^{\text{tt}},r_2^{\text{tt}})$.
\end{definition}

\begin{definition}[Tubal decomposition]
\label{def:tubal}
For any tensor $\calL\in\R^{d_1\times d_2\times d_3}$, applying the Fast Fourier Transform (FFT) of $\calL$ along the third dimension yields  $\ol \calL = \text{fft}(\calL, [\ ],3 )$.  In the transformed domain, a tubal decomposition represents each frontal slice $\ol\calL(:,:,s_3), s_3=1,\dots,d_3$ as a matrix decomposition:
\begin{eqnarray}
    \label{SVD of Tubal decomposition in the frequency domain main paper}
    \ol\calL(:,:,s_3) = \ol\calU(:,:,s_3)\ol\calD(:,:,s_3)\ol\calV^\top(:,:,s_3),
\end{eqnarray}
where $\ol\calU(:,:,s_3)\in\R^{d_1\times r_{\text{tubal}}}$ and $\ol\calV(:,:,s_3)\in\R^{d_2\times r_{\text{tubal}}}$ are orthonormal matrices, and $\ol\calD(:,:,s_3)\in\R^{r_{\text{tubal}}\times r_{\text{tubal}}}$ is a diagonal matrix with a tubal rank $r_{\text{tubal}}$. Furthermore, we can define $\ol\calU \in\R^{d_1\times r_{\text{tubal}}\times d_3}$, $\ol\calV \in\R^{d_2\times r_{\text{tubal}}\times d_3}$ and $\ol\calD \in\R^{r_{\text{tubal}}\times r_{\text{tubal}}\times d_3}$. The collection of these slice-wise decompositions across all Fourier modes defines the tensor singular value decomposition (t-SVD) of $\ol\calL$. Then, by applying the inverse FFT along the third dimension, we obtain the low-rank tubal decomposition of the original tensor, $\calL = \text{ifft}(\ol\calL, [\ ],3 )$.
Define $\calU = \text{ifft}(\ol\calU, [\ ],3 )$, $\calD = \text{ifft}(\ol\calD, [\ ],3 )$, $\calV = \text{ifft}(\ol\calV, [\ ],3 )$. Accordingly, we represent the low-rank tubal decomposition of $\calL$ as $\calL = [\calU, \calD, \calV]$.
\end{definition}
The rank-one tensor $\vu\circ\vv\circ\vz$ in \Cref{prop:necessity}, with $\vu,\vv,\vz$ nonzero, has Tucker rank $(1,1,1)$, since all three unfoldings have rank one; TT rank $(1,1)$, since \eqref{definition in the TT decomposition main paper} holds with $\mX_1=\vu$, $\mX_2(1,:,1)=\vv^\top$ and $\mX_3=\vz^\top$; and tubal rank $1$, since applying the FFT along the third dimension yields the frontal slices $\ol\vz(s_3)\vu\vv^\top$, $s_3=1,\dots,d_3$, where $\ol\vz$ is the FFT of $\vz$, and thus each frontal slice in the transformed domain has rank at most one. Hence, by \Cref{prop:necessity}, the condition $d_4\gtrsim d_3$ in \Cref{RIP condition fro the ToM regression} cannot be removed for any of these three low-rank models.
A unifying characteristic of the aforementioned tensor decompositions is the existence of canonical representations, thus serving as a rigorous foundation for subsequent theoretical analysis and algorithmic design. Building on this structural advantage, we undertake a detailed analysis of the recovery guarantees of $\wh\calL$, defined as the global minimizer of $\min_{\calL\in \setL_{\vr}}\frac{1}{2d_4}\|\calL \times_{3}^{1} \mA -   \calY  \|_F^2$
\begin{theorem}
\label{theorem: low-rank structure conclusion main paper}
Suppose that $\mA$ satisfies \eqref{RIP condition fro the ToM regression single} with $\delta_{d_3}\in(0,1)$ and that the entries of $\calE\in\R^{d_1\times d_2\times d_4}$ are i.i.d.\ $\calN(0, \sigma_e^2)$ and independent of $\mA$. Then, with probability at least $1- e^{-N}$, the solution $\wh\calL$ of \eqref{loss function of X another model II general optimization} with $\setX=\setL_{\vr}$ satisfies
\begin{eqnarray}
    \label{upper bound of error final main paper LR}
    \| \wh\calL  - \calL^\star\|_F\leq O\bigg( \frac{ \sqrt{(1+\delta_{d_3})N }}{\sqrt{(1-\delta_{d_3})^2d_4}}\,\frac{\sigma_e}{\sigma_a} \bigg),
\end{eqnarray}
with
\begin{eqnarray}
    \label{covering number of low-rank decompositions diff main paper LR}
    N = \begin{cases}
    r_1^\text{tk}r_2^\text{tk}r_3^\text{tk} + (d_1-r_1^\text{tk})r_1^\text{tk} + (d_2-r_2^\text{tk})r_2^\text{tk} + (d_3-r_3^\text{tk})r_3^\text{tk}, & \text{Tucker } \\
    d_1r_1^{\text{tt}}+d_2r_1^{\text{tt}}r_2^{\text{tt}}+d_3r_2^{\text{tt}}, & \text{TT}\\
    d_1d_3r_{\text{tubal}} + d_3r_{\text{tubal}} + d_2d_3r_{\text{tubal}} , & \text{Tubal}\\
    \end{cases}
\end{eqnarray}
where $(r_1^\text{tk},r_2^\text{tk},r_3^\text{tk})$ denotes the multilinear rank in the Tucker decomposition, $(r_1^{\text{tt}},r_2^{\text{tt}})$ denotes the tensor train ranks, and $r_{\text{tubal}}$ denotes the tubal rank. In particular, \eqref{upper bound of error final main paper LR} holds with probability at least $1-2e^{-c\delta_{d_3}^2d_4/K^4}- e^{-N}$ under the assumptions of \Cref{RIP condition fro the ToM regression}, where $c$ is the universal constant in \Cref{RIP condition fro the ToM regression}.
\end{theorem}
The proof is provided in {Appendix}~\ref{Proof of structured}. \Cref{theorem: low-rank structure conclusion main paper} shows that the recovery error scales as $O(\sqrt{N/d_4}\,\sigma_e/\sigma_a)$, where $N$ captures the intrinsic degrees of freedom of the low-rank tensor under each decomposition format. In each case, the bound is strictly sharper than the unstructured baseline $\sqrt{d_1d_2d_3/d_4}\,\sigma_e/\sigma_a$ in \Cref{prop:unstructured baseline} whenever the corresponding rank is small, confirming that exploiting low-rank structure leads to a provable reduction in recovery error. Moreover, comparing with \Cref{theorem:sparse  structure conclusion main paper}, the role of the sparsity parameter $s\log(d_1d_2d_3/s)$ is here played by $N$, suggesting a unified view in which the recovery error is governed by the intrinsic complexity of the assumed structural model.

\subsection{Sparse plus Low-rank Tensors}
\label{subsec:SpL tensors}

In a wide range of practical applications, tensors often exhibit both local sparsity and global low-rankness simultaneously \cite{oymak2015simultaneously,cai2023generalized,li2025fourier}. For example, high-dimensional signals may contain a small number of dominant patterns (low-rank structure) superimposed with localized but critical irregularities (sparse components). Such a hybrid model provides a more faithful representation of real-world data by combining the interpretability of sparsity with the expressive power of low-rankness. Consequently, the sparse plus low-rank framework has emerged as a powerful structural prior, and in what follows we establish theoretical recovery guarantees under this composite assumption. Formally, we denote $\calX^\star = \calS^\star + \calL^\star$ as a sparse plus low-rank tensor if $\calS^\star\in\wt\setS_{s}$ and $\calL^\star\in\wt\setL_{\vr}$, where
\begin{eqnarray}
    \label{the definition of sparse set in the SpL set}
    \wt\setS_{s} = \{\calS\in\R^{d_1\times d_2 \times d_3}: \|\calS\|_F\leq S, \text{supp}(\calS) = s    \}, \\
    \label{the definition of low-rank set in the SpL set}
    \wt\setL_{\vr} = \{ \calL\in\R^{d_1\times d_2 \times d_3}: \|\calL\|_F\leq L, \text{rank}(\calL) = \vr  \}.
\end{eqnarray}
The additional Frobenius-norm constraints $\|\calS\|_F\leq S$ and $\|\calL\|_F\leq L$ are imposed to facilitate the theoretical analysis: they make the feasible set bounded, which is required by the covering argument in the proof. In the following, we establish the recovery error guarantees for $(\wh\calS, \wh\calL)$, defined as the global minimizer of $\min_{\calS\in \wt\setS_{s}, \calL\in \wt\setL_{\vr}}  \frac{1}{2d_4}\|(\calS + \calL) \times_{3}^{1} \mA -   \calY  \|_F^2$.
\begin{theorem}
\label{theorem:sparse and low-rank structure conclusion main paper}
Suppose that $s\le d_1d_2d_3/3$, that $\mA$ satisfies \eqref{RIP condition fro the ToM regression single} with $\delta_{d_3}\in(0,1)$ and that the entries of $\calE\in\R^{d_1\times d_2\times d_4}$ are i.i.d.\ $\calN(0, \sigma_e^2)$ and independent of $\mA$. Let
\begin{eqnarray}
    \label{SL complexity}
    \Lambda_{\text{SL}}=s\log\Big(\frac{d_1d_2d_3}{s}\Big)+N+(s+N)\log\Big(1+\frac{(S+L)\sigma_ad_4\sqrt{d_1d_2d_3}}{\sigma_e}\Big).
\end{eqnarray}
Then, with probability at least $1- e^{-s\log(d_1d_2d_3/s)}-e^{-d_1d_2d_3d_4}$, the solution $(\wh\calS, \wh\calL)$ of \eqref{loss function of X another model II general optimization} with $\setX=\wt\setS_s+\wt\setL_{\vr}$ satisfies
\begin{eqnarray}
    \label{upper bound of error final main paper}
    \|\wh\calS + \wh\calL - \calS^\star - \calL^\star\|_F\leq O\bigg( \frac{ \sqrt{(1+\delta_{d_3})\Lambda_{\text{SL}}}}{\sqrt{(1-\delta_{d_3})^2d_4}}\,\frac{\sigma_e}{\sigma_a} \bigg),
\end{eqnarray}
with
\begin{eqnarray}
    \label{covering number of low-rank decompositions diff main paper}
    N = \begin{cases}
    r_1^\text{tk}r_2^\text{tk}r_3^\text{tk} + (d_1-r_1^\text{tk})r_1^\text{tk} + (d_2-r_2^\text{tk})r_2^\text{tk} + (d_3-r_3^\text{tk})r_3^\text{tk}, & \text{Tucker}, \\
    d_1r_1^{\text{tt}}+d_2r_1^{\text{tt}}r_2^{\text{tt}}+d_3r_2^{\text{tt}}, & \text{TT},\\
    d_1d_3r_{\text{tubal}} + d_3r_{\text{tubal}} + d_2d_3r_{\text{tubal}} , & \text{Tubal}.\\
    \end{cases}
\end{eqnarray}
In particular, \eqref{upper bound of error final main paper} holds with probability at least $1-2e^{-c\delta_{d_3}^2d_4/K^4}-e^{-s\log(d_1d_2d_3/s)}-e^{-d_1d_2d_3d_4}$ under the assumptions of \Cref{RIP condition fro the ToM regression}, where $c$ is the universal constant in \Cref{RIP condition fro the ToM regression}.
\end{theorem}
The proof is provided in {Appendix}~\ref{Proof of structured}. From \eqref{upper bound of error final main paper}, it is evident that the recovery error is governed jointly by the sparsity level $s$ and the low-rank degrees of freedom $N$. In addition, the bound depends on the magnitudes $S$ and $L$ of the sparse and low-rank components, but only logarithmically, through the logarithm of the signal-to-noise ratio $\frac{(S+L)\sigma_a}{\sigma_e}$. This factor arises because the error $\wh\calS+\wh\calL-\calS^\star-\calL^\star$ is a sum of a sparse and a low-rank tensor and cannot be normalized as in the purely sparse or purely low-rank cases; the norm constraints $\|\calS\|_F\leq S$ and $\|\calL\|_F\leq L$ make the feasible set bounded, and the covering argument for this bounded set introduces the logarithmic factor. It is important to note that this factor arises primarily from the theoretical derivation to ensure generality and does not necessarily reflect a strict limitation in practical recovery algorithms. We emphasize that \Cref{theorem:sparse and low-rank structure conclusion main paper} bounds the error of the sum $\wh\calS+\wh\calL$, i.e., the recovery of $\calX^\star=\calS^\star+\calL^\star$, which is the quantity that determines the responses in \eqref{ToM regression general data model II}, and not the errors of the two components separately. Without separation assumptions, the components are not identifiable, since a tensor can be both sparse and low-rank. Recovering $\calS^\star$ and $\calL^\star$ individually would require additional separation assumptions, such as incoherence conditions on the low-rank component and on the support of the sparse component~\cite{candes2011robust,chandrasekaran2011rank}. Since the prediction of the responses requires only the sum, we restrict the interpretation of \Cref{theorem:sparse and low-rank structure conclusion main paper} to the recovery of $\calX^\star$.

\begin{remark}[Variation within the window]
\label{rem:within window}
If \eqref{snapshot assumption} does not hold, then $\calY=\calX_i^\star\times_3^1\mA+\calE+\calW_i$, where $\calW_i\in\R^{d_1\times d_2\times d_4}$ with $\calW_i(:,:,s_4)=(\calX^\star_{i-s_4+1}-\calX_i^\star)\times_3^1\mA(:,s_4)$ collects the variation of the regression tensor within the window. Suppose that $\calX_i^\star=\calS^\star+\calL^\star$ has the sparse plus low-rank structure of \Cref{theorem:sparse and low-rank structure conclusion main paper}. Then, under the assumptions of \Cref{theorem:sparse and low-rank structure conclusion main paper} and with the same probability, the solution $(\wh\calS,\wh\calL)$ satisfies
\begin{eqnarray}
    \label{within window SL}
    \|\wh\calS+\wh\calL-\calX_i^\star\|_F\le\text{right-hand side of \eqref{upper bound of error final main paper}}+\frac{2\sqrt{1+\delta_{d_3}}}{(1-\delta_{d_3})\sigma_a}\frac{\|\calW_i\|_F}{\sqrt{d_4}},
\end{eqnarray}
where the first term is the noise term and the second term is the lag term caused by the variation within the window (see {Appendix}~\ref{Proof of within window}). The same lag term is added to the bounds of \Cref{theorem:sparse  structure conclusion main paper,theorem: low-rank structure conclusion main paper} when $\calX_i^\star$ is sparse or low-rank, respectively. For the AR(1) model \eqref{RW model definition} with $\calX_0^\star$ drawn from its stationary distribution, $\E[\|\calW_i\|_F^2]\le\sigma_a^2d_1d_2d_3\lambda^2\sigma_q^2d_4(d_4-1)/(2-\lambda)$. Since $2-\lambda>1$, this gives $\E[\|\calW_i\|_F^2]/d_4\le\sigma_a^2\lambda^2\sigma_q^2d_1d_2d_3(d_4-1)$, so $\|\calW_i\|_F/\sqrt{d_4}$ is of order at most $\sigma_a\lambda\sigma_q\sqrt{d_1d_2d_3(d_4-1)}$. A larger window therefore reduces the noise term, which scales as $1/\sqrt{d_4}$, but increases the lag term, which grows with $d_4$; \eqref{snapshot assumption} is a good approximation when the lag term is small compared with the noise term.
\end{remark}

Taken together, Theorems~\ref{theorem:sparse  structure conclusion main paper}--\ref{theorem:sparse and low-rank structure conclusion main paper} establish a unified picture: the recovery error at a fixed time instant is governed not by the ambient dimension $d_1d_2d_3$ but by the intrinsic complexity of the assumed structural model, whether sparsity, low-rankness, or their combination. These theorems bound the global minimizer of the constrained least-squares problem formed by one window; they are statistical benchmarks for one window of $d_4$ samples, not guarantees for the SGD iterates, whose accuracy is governed by the step size (\Cref{sec:steady-state performance,sec:tracking ability performance}). While the analysis focuses on time instants where the ground-truth tensor exhibits low-dimensional structure, such assumptions remain meaningful in time-varying settings, where the sparse pattern or low-rank subspace may evolve over time. In such scenarios, the bound characterizes the instantaneous recovery accuracy at each window, and its constants depend explicitly on the underlying low-dimensional structure at each time instant and vary accordingly. To obtain a guarantee over a stream of $T$ windows, note that each failure probability in \Cref{theorem:sparse  structure conclusion main paper,theorem: low-rank structure conclusion main paper,theorem:sparse and low-rank structure conclusion main paper} concerns a single window. Hence, by the union bound, which does not require the windows to be independent, the bounds hold for all $T$ windows simultaneously with probability at least one minus $T$ times the failure probability of one window; for example, \eqref{upper bound of error final main paper sparse} holds for all $T$ windows with probability at least $1-2Te^{-c\delta_{d_3}^2d_4/K^4}-Te^{-s\log(d_1d_2d_3/s)}$. Since $Te^{-x}\le e^{-x/2}$ whenever $x\ge2\log T$, this probability satisfies
\begin{eqnarray}
    \label{T windows probability}
    1-2Te^{-c\delta_{d_3}^2d_4/K^4}-Te^{-s\log(d_1d_2d_3/s)}\ge1-2e^{-c\delta_{d_3}^2d_4/(2K^4)}-e^{-s\log(d_1d_2d_3/s)/2},
\end{eqnarray}
provided that
\begin{eqnarray}
    \label{T windows condition}
    d_4\ge\frac{2K^4\log T}{c\,\delta_{d_3}^2}\qquad\text{and}\qquad s\log\frac{d_1d_2d_3}{s}\ge2\log T.
\end{eqnarray}
Together with \eqref{eq:mrip ToM single}, where $C$ and $c$ are the universal constants in \Cref{RIP condition fro the ToM regression}, the window length therefore only needs to satisfy $d_4\ge\max\{CK^4d_3,\,2K^4\log T/c\}/\delta_{d_3}^2$, which grows logarithmically with the horizon $T$.

\subsection[Structured Adaptive Updates]{Structured Adaptive Updates}
\label{sec:structured updates}

\Cref{theorem:sparse  structure conclusion main paper,theorem: low-rank structure conclusion main paper,theorem:sparse and low-rank structure conclusion main paper} concern the global minimizer of \eqref{loss function of X another model II general optimization} for one window. In the streaming setting, the same structural priors can be incorporated into the SGD update \eqref{updating method X another model II} by projecting each iterate onto the corresponding structured set. For sparse tensors, the update becomes
\begin{eqnarray}
    \label{projected SGD sparse}
    \calS_{i+1}=\calP_{s}\Big(\calS_i-\frac{\mu}{d_4}(\calS_i\times_3^1\mA_i-\calY_i)\times_3^2\mA_i\Big),
\end{eqnarray}
where $\calP_s$ denotes the projection onto $\setS_s$, which retains the $s$ entries of largest magnitude and sets the others to zero. For low-rank tensors, the update becomes
\begin{eqnarray}
    \label{projected SGD lowrank}
    \calL_{i+1}=\calP_{\vr}\Big(\calL_i-\frac{\mu}{d_4}(\calL_i\times_3^1\mA_i-\calY_i)\times_3^2\mA_i\Big),
\end{eqnarray}
where $\calP_{\vr}$ denotes the projection onto $\setL_{\vr}$, computed by truncated HOSVD~\cite{sidiropoulos2000uniqueness}, TT-SVD~\cite{oseledets2011tensor} or t-SVD~\cite{kilmer2011factorization} for the Tucker, tensor train or tubal decomposition, respectively. For sparse plus low-rank tensors, the two components are updated with the gradient of the same residual and with step sizes $\mu_1$ and $\mu_2$,
\begin{eqnarray}
    \label{projected SGD SpL S}
    \calS_{i+1}&\!\!\!\!=\!\!\!\!&\calP_{s}\Big(\calS_i-\frac{\mu_1}{d_4}\big((\calS_i+\calL_i)\times_3^1\mA_i-\calY_i\big)\times_3^2\mA_i\Big),\\
    \label{projected SGD SpL L}
    \calL_{i+1}&\!\!\!\!=\!\!\!\!&\calP_{\vr}\Big(\calL_i-\frac{\mu_2}{d_4}\big((\calS_i+\calL_i)\times_3^1\mA_i-\calY_i\big)\times_3^2\mA_i\Big),
\end{eqnarray}
and the estimate of $\calX^\star$ is $\calS_{i+1}+\calL_{i+1}$. Without the projections, each of these updates reduces to \eqref{updating method X another model II}. The updates \eqref{projected SGD sparse}--\eqref{projected SGD SpL L} are adaptive counterparts of iterative hard thresholding (IHT) for sparse and low-rank recovery~\cite{blumensath2009iterative,rauhut2017low}, and they serve to illustrate the theory. The detailed algorithms, which perform several inner iterations on each window, together with a data-driven selection of the sparsity level and the ranks, the choice of their parameters and their computational complexity, are given in {Appendix}~\ref{app:algorithms}. We emphasize that \Cref{theorem:sparse  structure conclusion main paper,theorem: low-rank structure conclusion main paper,theorem:sparse and low-rank structure conclusion main paper} characterize the statistical accuracy of the global minimizer for one window, which serves as a benchmark for these updates. Since the projections are onto nonconvex sets, these results do not directly provide recovery guarantees for \eqref{projected SGD sparse}--\eqref{projected SGD SpL L}, and their convergence analysis is left for future work.

\section{Experimental Results}
\label{sec:experimental_results}

In this section, we evaluate the structured adaptive updates of \Cref{sec:structured updates}, whose detailed algorithms are given in {Appendix}~\ref{app:algorithms}, on synthetic data, on the tracking of a drifting quantum gate, on citywide crowd-flow prediction, and on the prediction of the global total electron content (TEC) distribution.

\subsection[Synthetic Data]{Synthetic Data}
\label{subsec:simulation}

We set $d_1=18$, $d_2=36$, $d_3=3$ and $d_4=100$. The inputs $\va_i$ are i.i.d.\ $\calN(\mathbf{0},\sigma_a^2\mId_{d_3\times d_3})$ with $\sigma_a^2=1$, the entries of the noise are i.i.d.\ $\calN(0,\sigma_e^2)$ with $\sigma_e^2=0.01$, and the windows are formed from a single data stream according to \eqref{window overlap}. The sparse component $\calS^\star$ has $s$ nonzero entries at uniformly random positions, with random signs and magnitudes drawn uniformly from $[1,2]$. The low-rank component $\calL^\star$ is generated from a Gaussian core tensor and random orthonormal factor matrices (Tucker), from Gaussian TT cores (TT), or from Fourier-domain frontal slices with random orthonormal singular vectors (tubal). It is scaled to $\|\calL^\star\|_F=3$, and its nonzero singular values, of the unfoldings (Tucker and TT) or of the Fourier-domain slices (tubal), are at least $0.2$ times the largest one.

The parameters of the algorithms are explained in {Appendix}~\ref{subsec:parameters}. All methods are initialized with zero tensors and use $N_{\max}=10$ and $\epsilon=\epsilon_1=\epsilon_2=10^{-3}$. The sparsity level and the ranks are selected with the thresholds $\eta_s=0.3$ and $\eta_r=0.15$. The step sizes are $\mu=0.01$ for SGD, $\mu=0.02$ for Sparse IHT, $\mu=0.01$ for Low-rank IHT, and $\mu=0.01$, $\mu_1=0.5$ and $\mu_2=0.01$ for Sparse and Low-rank IHT. We report $\text{MSE}=\|\calX_i-\calX_i^\star\|_F^2$, averaged over $10$ independent trials, where $\calX_i$ denotes the estimate of each method, i.e., $\calS_i+\calL_i$ for Sparse and Low-rank IHT.

\paragraph*{Stationary Environment}
We first keep the regression tensor constant over $T=1500$ iterations. It is sparse with $s=20$; low-rank with Tucker ranks $(4,5,2)$, TT ranks $(4,2)$ or tubal rank $5$; or the sum of such a sparse and a low-rank tensor. Each structured method uses the corresponding structure with the selected sparsity level and ranks.

\begin{figure*}[!ht]
\centering
\subfigure[]{
\begin{minipage}[t]{0.23\textwidth}
\centering
\includegraphics[width=\linewidth]{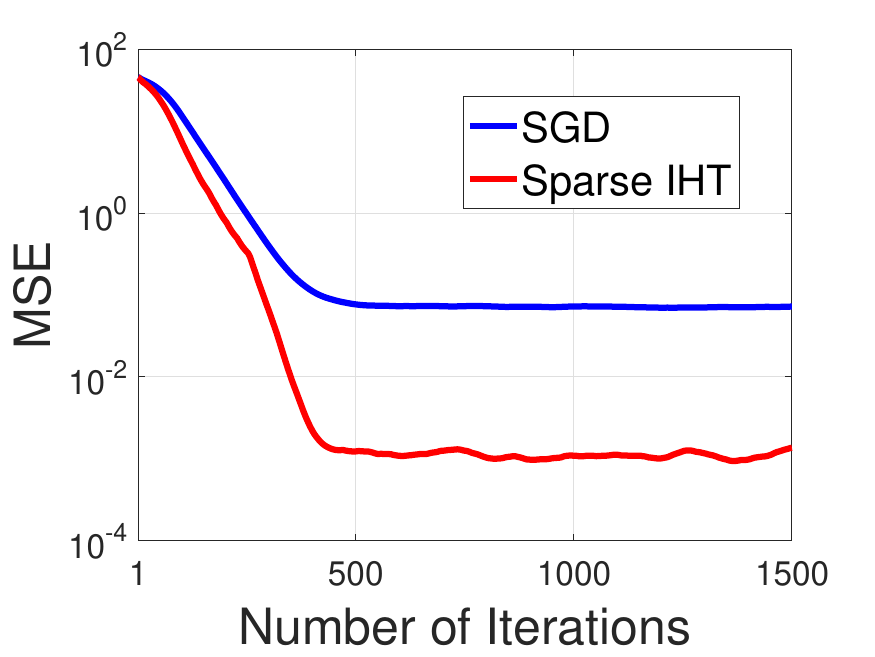}
\end{minipage}
\label{fig:sparse_iht}
}
\hfill
\subfigure[]{
\begin{minipage}[t]{0.23\textwidth}
\centering
\includegraphics[width=\linewidth]{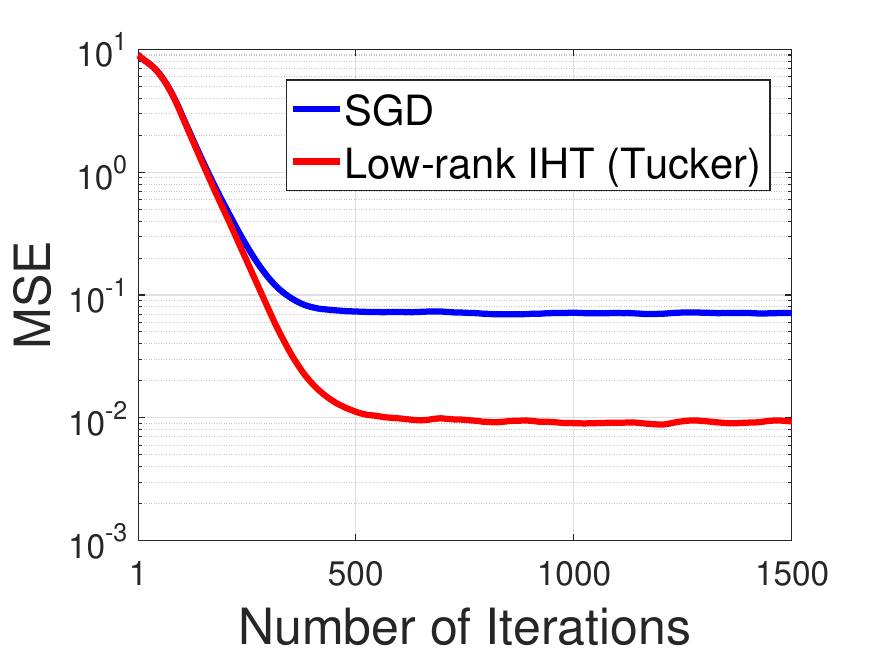}
\end{minipage}
\label{fig:tucker_iht}
}
\hfill
\subfigure[]{
\begin{minipage}[t]{0.23\textwidth}
\centering
\includegraphics[width=\linewidth]{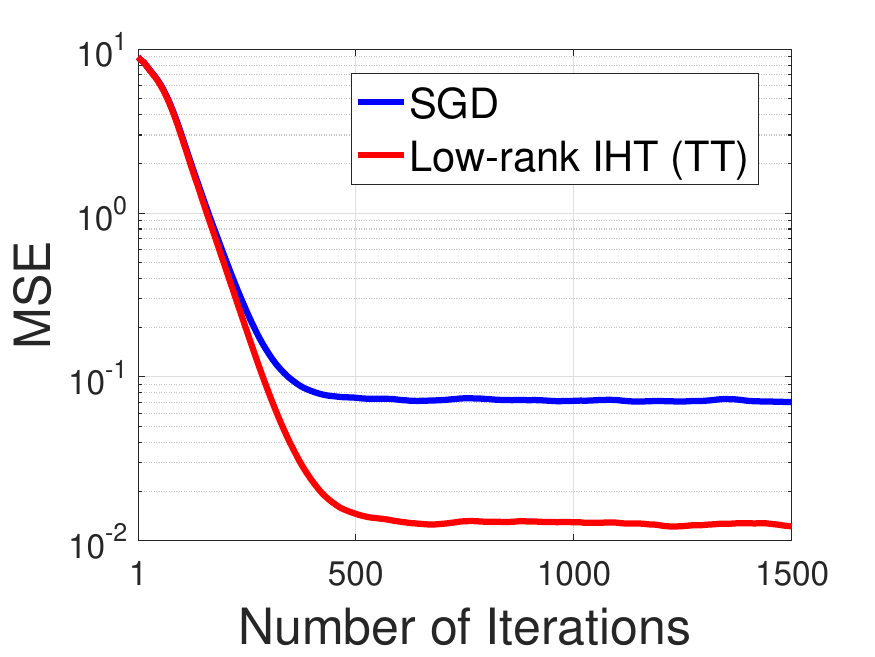}
\end{minipage}
\label{fig:tt_iht}
}
\hfill
\subfigure[]{
\begin{minipage}[t]{0.23\textwidth}
\centering
\includegraphics[width=\linewidth]{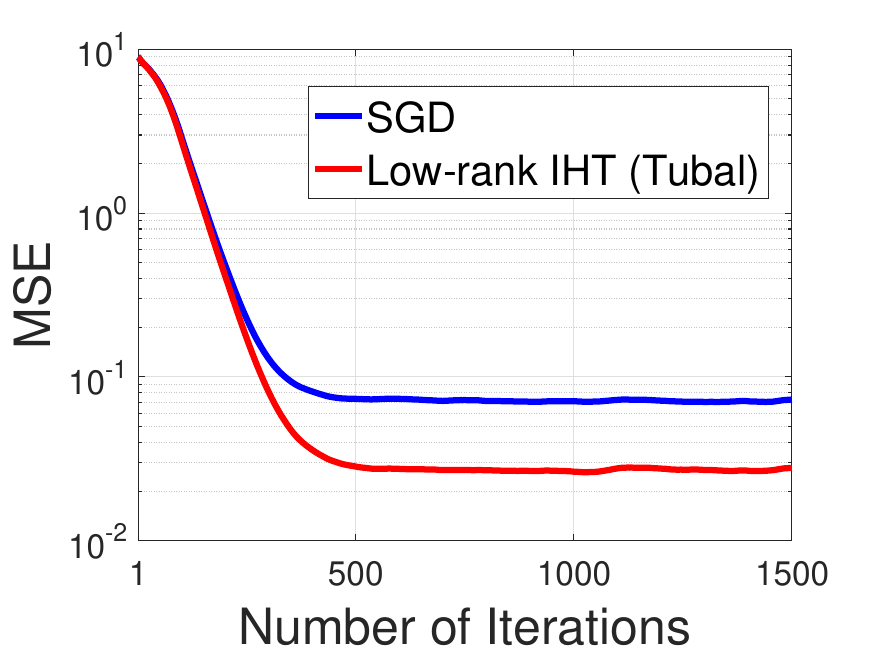}
\end{minipage}
\label{fig:tubal_iht}
}
\caption{Stationary environment: MSE of SGD and of (a) Sparse IHT for a sparse tensor with $s=20$, and of Low-rank IHT for a low-rank tensor with (b) Tucker ranks $(4,5,2)$, (c) TT ranks $(4,2)$ and (d) tubal rank $5$.}
\label{fig:mse_all}
\end{figure*}

\begin{figure}[!ht]
\centering
\subfigure[]{
\begin{minipage}[t]{0.31\textwidth}
\centering
\includegraphics[width=5cm]{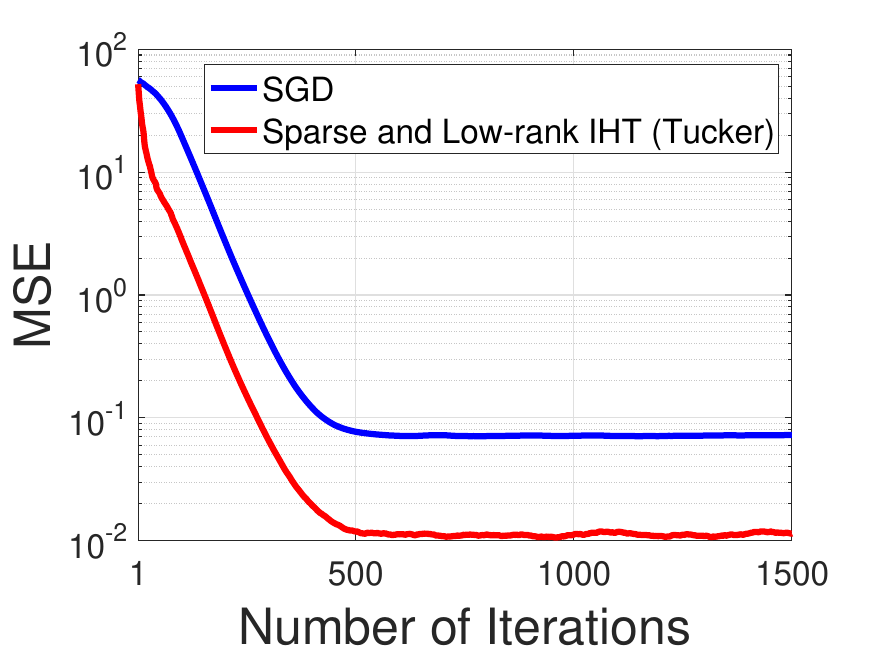}
\end{minipage}
\label{fig:sl_tucker}
}
\subfigure[]{
\begin{minipage}[t]{0.31\textwidth}
\centering
\includegraphics[width=5cm]{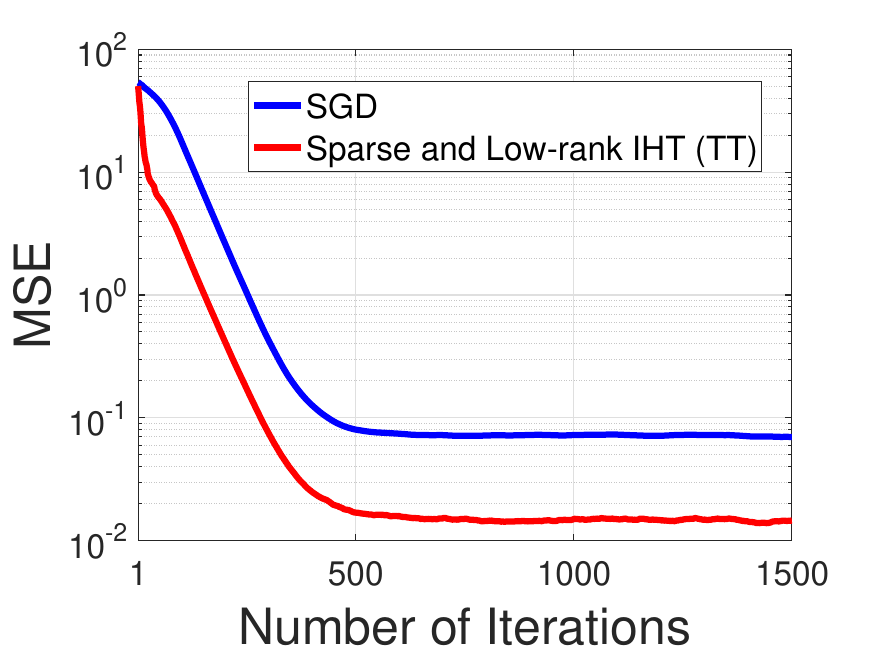}
\end{minipage}
\label{fig:sl_tt}
}
\subfigure[]{
\begin{minipage}[t]{0.31\textwidth}
\centering
\includegraphics[width=5cm]{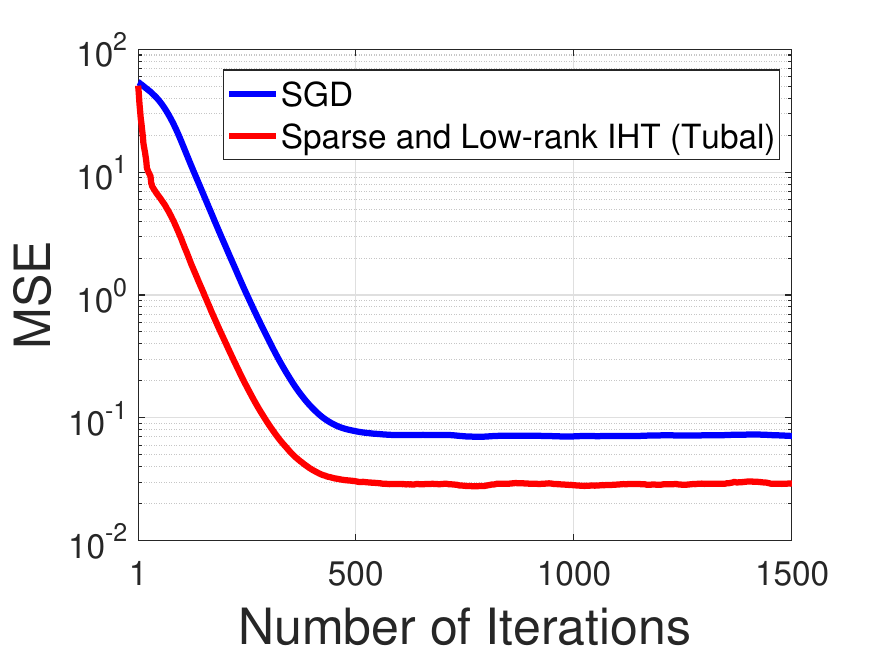}
\end{minipage}
\label{fig:sl_tubal}
}
\caption{Stationary environment: MSE of SGD and Sparse and Low-rank IHT for a sparse plus low-rank tensor with $s=20$ and (a) Tucker ranks $(4,5,2)$, (b) TT ranks $(4,2)$ and (c) tubal rank $5$.}
\label{fig:mse_sl}
\end{figure}

Fig.~\ref{fig:sparse_iht} demonstrates that Sparse IHT outperforms the unstructured SGD baseline in terms of steady-state MSE when the ground-truth tensor is sparse, in agreement with \Cref{theorem:sparse  structure conclusion main paper}. Figures~\ref{fig:tucker_iht},~\ref{fig:tt_iht}, and~\ref{fig:tubal_iht} demonstrate that Low-rank IHT achieves a lower steady-state MSE than SGD under all three tensor decomposition formats, confirming that the low-rank structure effectively suppresses noise by reducing the degrees of freedom of the estimator, as established in \Cref{theorem: low-rank structure conclusion main paper}. Finally, Figures~\ref{fig:sl_tucker},~\ref{fig:sl_tt}, and~\ref{fig:sl_tubal} show that Sparse and Low-rank IHT achieves both a faster convergence rate and a lower steady-state MSE than SGD across all decomposition formats, demonstrating the benefit of jointly exploiting sparsity and low-rankness. These results are consistent with the theoretical analysis in \Cref{theorem:sparse and low-rank structure conclusion main paper}.

\paragraph*{Non-stationary Environment}
We next consider a non-stationary environment, in which the sparse plus low-rank regression tensor is redrawn every $1000$ iterations with a different structure: $s=20$, $40$ and $10$, together with the Tucker ranks $(4,5,2)$, $(2,3,2)$ and $(6,8,3)$, the TT ranks $(4,2)$, $(2,1)$ and $(6,3)$, or the tubal ranks $5$, $2$ and $8$. Besides SGD, we run Sparse and Low-rank IHT with three choices of the sparsity level and the ranks: the \emph{first-phase structure}, which uses the true values of the first phase in all phases; the \emph{largest structure}, which uses the largest values over the three phases, i.e., $s=40$ and the entrywise largest ranks, in all phases; and the \emph{selected structure}, which selects the values from the data at each window. The other settings are the same as above, and the selected sparsity levels and ranks are shown for one trial.

\begin{figure}[!ht]
\centering
\subfigure[]{
\begin{minipage}[t]{0.31\textwidth}
\centering
\includegraphics[width=5cm]{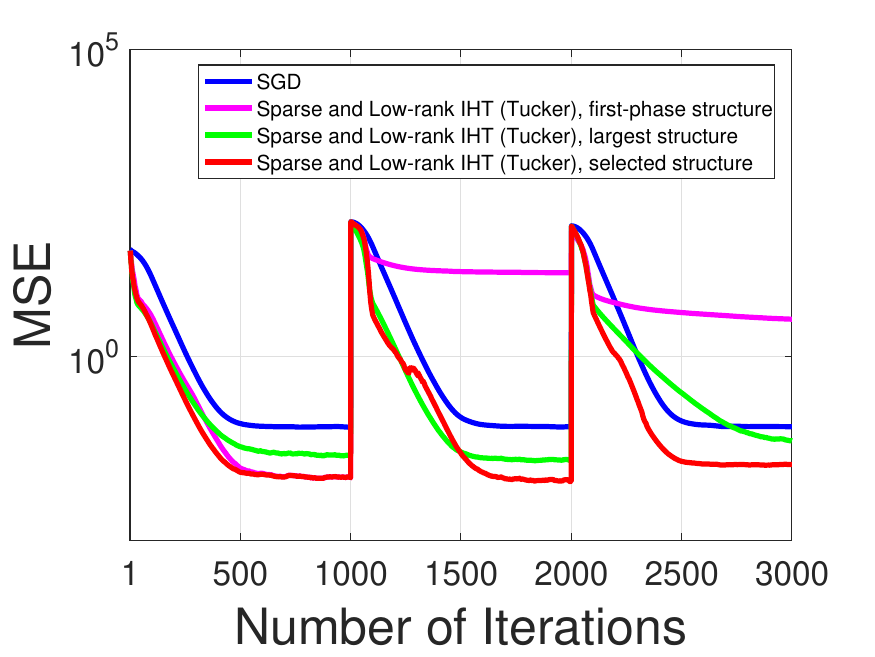}
\end{minipage}
\label{fig:tv_tucker_mse}
}
\subfigure[]{
\begin{minipage}[t]{0.31\textwidth}
\centering
\includegraphics[width=5cm]{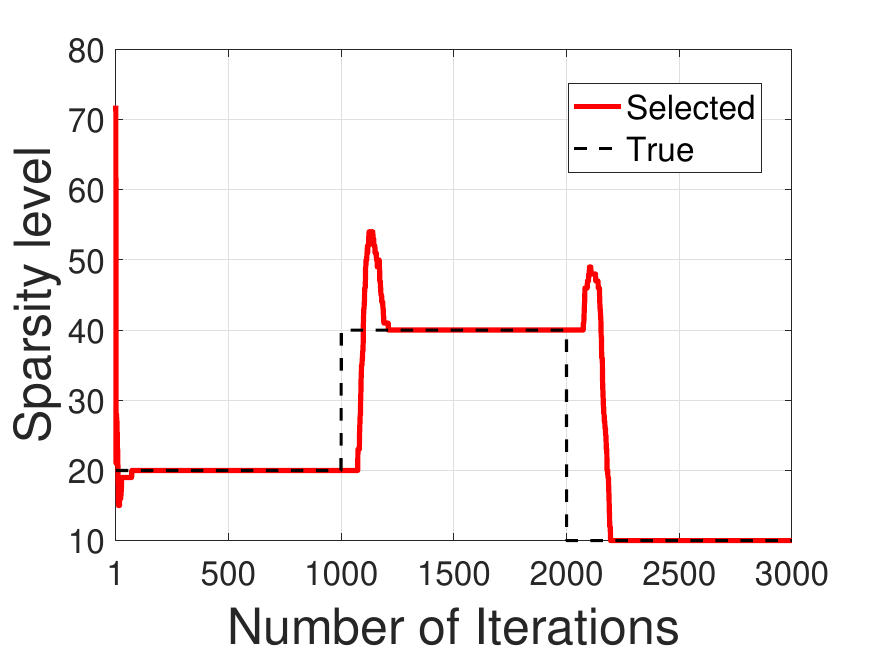}
\end{minipage}
\label{fig:tv_tucker_s}
}
\subfigure[]{
\begin{minipage}[t]{0.31\textwidth}
\centering
\includegraphics[width=5cm]{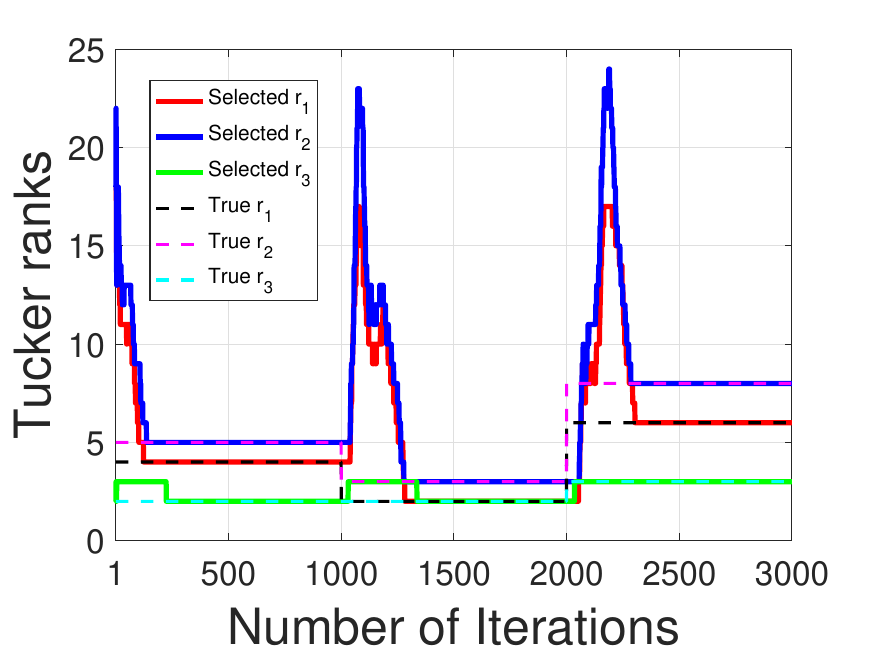}
\end{minipage}
\label{fig:tv_tucker_r}
}
\caption{Non-stationary environment with a Tucker low-rank component, where the phases change at $i=1000$ and $i=2000$: (a) MSE of SGD and of Sparse and Low-rank IHT with the first-phase, the largest and the selected structure; (b) selected and true sparsity levels; (c) selected and true Tucker ranks.}
\label{fig:tv_tucker}
\end{figure}

\begin{figure}[!ht]
\centering
\subfigure[]{
\begin{minipage}[t]{0.31\textwidth}
\centering
\includegraphics[width=5cm]{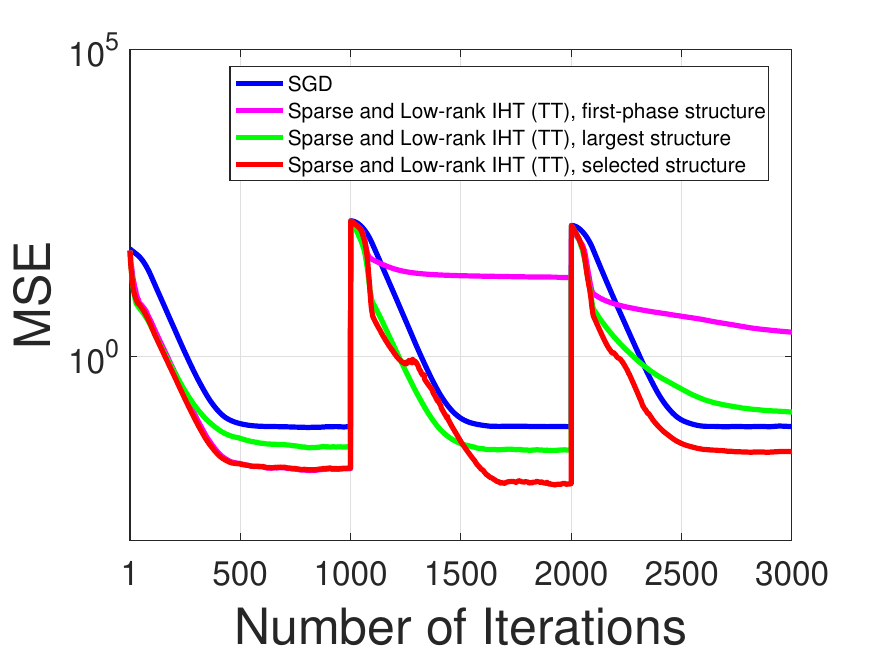}
\end{minipage}
\label{fig:tv_tt_mse}
}
\subfigure[]{
\begin{minipage}[t]{0.31\textwidth}
\centering
\includegraphics[width=5cm]{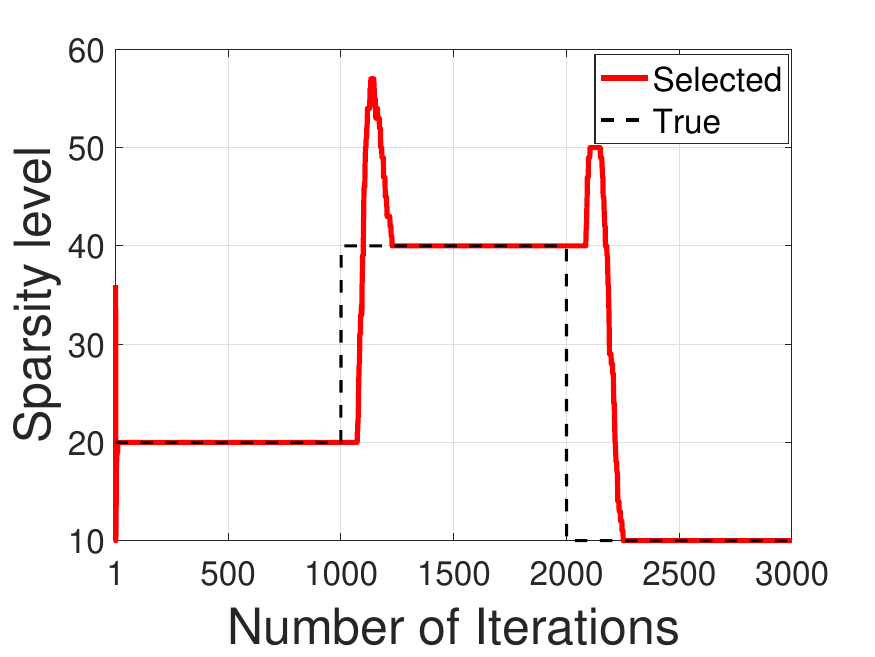}
\end{minipage}
\label{fig:tv_tt_s}
}
\subfigure[]{
\begin{minipage}[t]{0.31\textwidth}
\centering
\includegraphics[width=5cm]{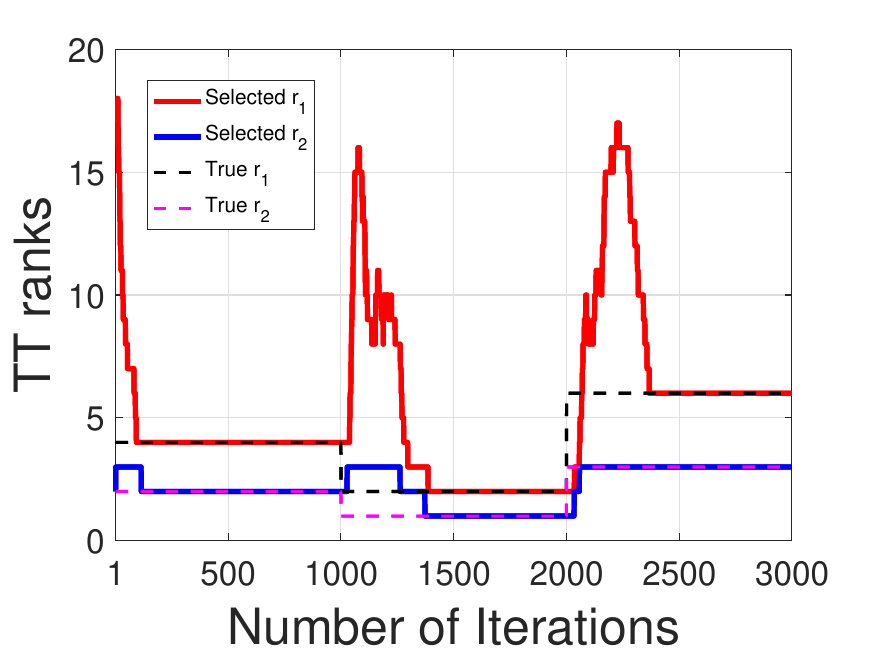}
\end{minipage}
\label{fig:tv_tt_r}
}
\caption{Non-stationary environment with a TT low-rank component, where the phases change at $i=1000$ and $i=2000$: (a) MSE of SGD and of Sparse and Low-rank IHT with the first-phase, the largest and the selected structure; (b) selected and true sparsity levels; (c) selected and true TT ranks.}
\label{fig:tv_tt}
\end{figure}

\begin{figure}[!ht]
\centering
\subfigure[]{
\begin{minipage}[t]{0.31\textwidth}
\centering
\includegraphics[width=5cm]{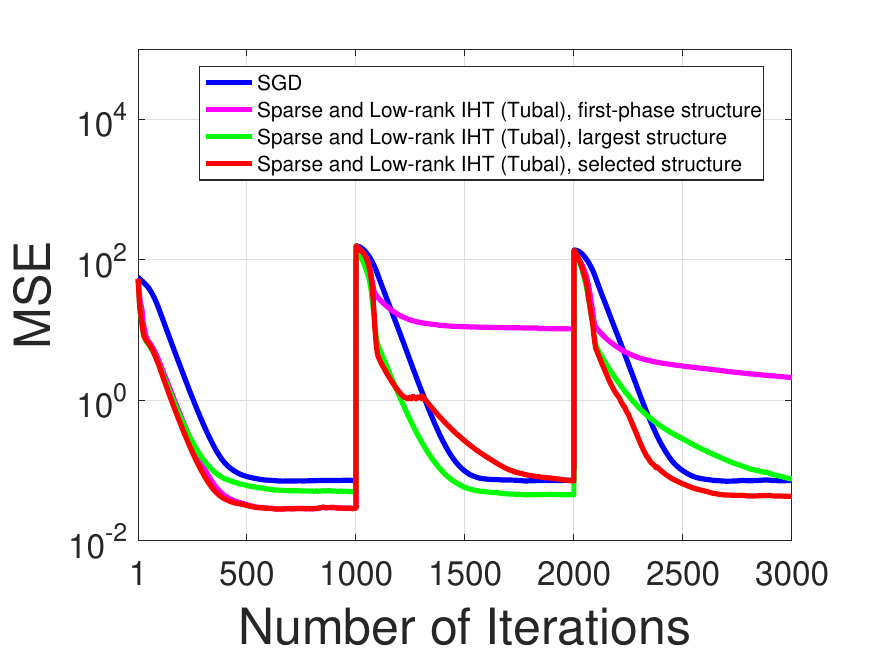}
\end{minipage}
\label{fig:tv_tubal_mse}
}
\subfigure[]{
\begin{minipage}[t]{0.31\textwidth}
\centering
\includegraphics[width=5cm]{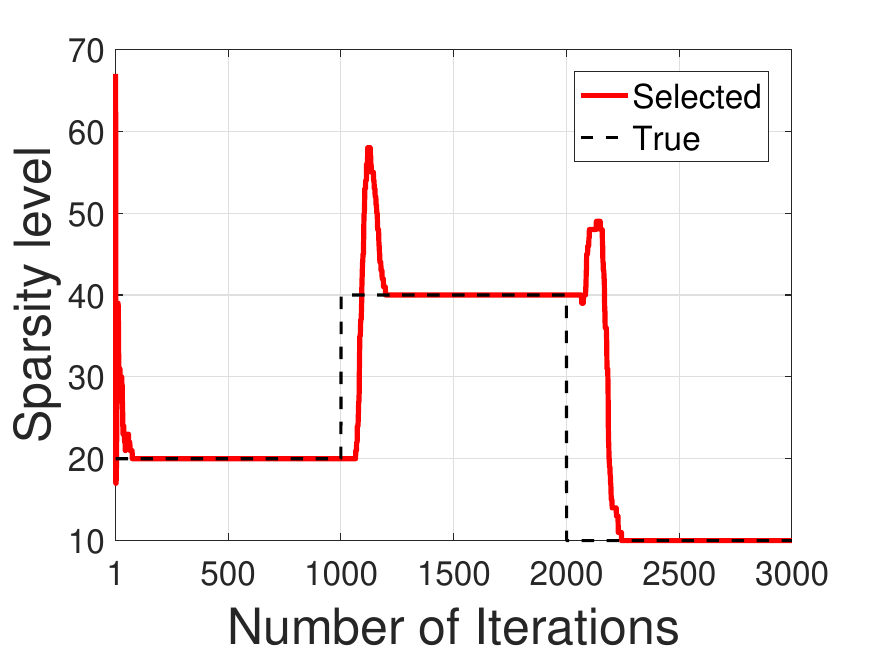}
\end{minipage}
\label{fig:tv_tubal_s}
}
\subfigure[]{
\begin{minipage}[t]{0.31\textwidth}
\centering
\includegraphics[width=5cm]{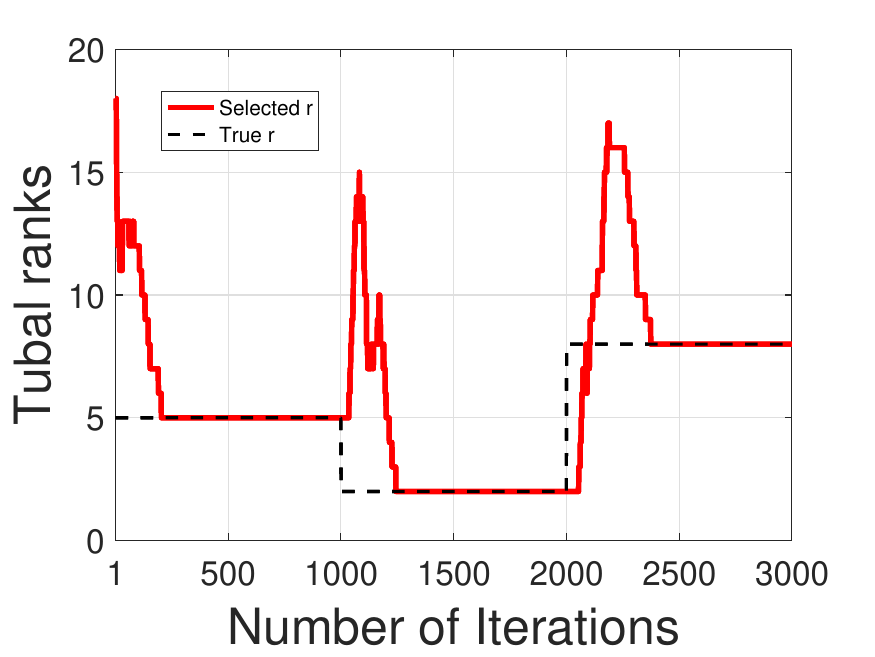}
\end{minipage}
\label{fig:tv_tubal_r}
}
\caption{Non-stationary environment with a tubal low-rank component, where the phases change at $i=1000$ and $i=2000$: (a) MSE of SGD and of Sparse and Low-rank IHT with the first-phase, the largest and the selected structure; (b) selected and true sparsity levels; (c) selected and true tubal ranks.}
\label{fig:tv_tubal}
\end{figure}

\Cref{fig:tv_tucker,fig:tv_tt,fig:tv_tubal} show that no fixed structure is suitable for all phases. The first-phase structure is correct only in the first phase. In the second and third phases, the new regression tensor has more nonzero entries or larger ranks than this structure and thus cannot be represented by it, so that the MSE stagnates far above that of SGD. The largest structure can represent the regression tensor in every phase, but it estimates more parameters than necessary. This leads to a higher steady-state MSE, as predicted by \Cref{theorem:sparse and low-rank structure conclusion main paper}, and to a slower convergence in the third phase. The selected structure avoids both problems. As shown in panels (b) and (c) of these figures, the selected sparsity level and ranks follow each change and return to the true values within about $300$ iterations. Consequently, the selected structure attains the same MSE as the true structure in the first phase and the lowest MSE in almost all phases. Shortly after each change, the selected values temporarily exceed the true ones, because the unstructured SGD iterate used for the selection still contains the previous regression tensor. This overestimation does not exclude any component of the new regression tensor and therefore does not prevent the tracking.

\subsection[Tracking a Drifting Quantum Gate]{Tracking a Drifting Quantum Gate}
\label{subsec:quantum}

In a quantum computer, the gates are affected by noise that drifts over time~\cite{proctor2020detecting}, so that the operation of a gate has to be tracked continuously. We show that this task is an instance of ToM regression. Consider a two-qubit gate. At time $i$, a known input state is prepared, passed through the gate, and the output state is measured. The input state is described by a vector $\va_i\in\R^{15}$, and the measurements give a matrix $\mY_i\in\R^{4\times4}$; specifically, $\va_i$ and $\mY_i$ contain the expectation values of the two-qubit Pauli operators in the input and the output state, respectively. Each entry of $\mY_i$ is the average of $\pm1$ outcomes over $1000$ repeated measurements and is therefore noisy. Since quantum operations are linear, $\mY_i=\calX_i^\star\times_3^1\va_i+\calE_i(:,:,1)$ as in \eqref{ToM regression general data model II}, where the regression tensor $\calX_i^\star\in\R^{4\times4\times15}$ describes the gate, known as its Pauli transfer matrix~\cite{greenbaum2015introduction}, and the measurement noise $\calE_i(:,:,1)$ has variance at most $10^{-3}$. Hence, tracking the gate amounts to tracking $\calX_i^\star$. The input states are drawn uniformly at random, so that $\va_i$ has zero mean and covariance $0.2\,\mId_{15\times15}$.

We consider the CNOT gate. Ideally, it maps each input component to a single output component, so that $\calX_i^\star$ has only $s=15$ nonzero entries. In practice, two kinds of errors occur. First, a weak background noise~\cite{van2023probabilistic}, whose strength drifts randomly over time, changes the values of these entries but not their positions. Second, small unwanted rotations of the qubits, called coherent errors, spread some input components over two output components and thus increase the number of nonzero entries. We switch three such errors on and off every $1000$ iterations, which gives six phases with $s=15$, $23$, $35$, $53$, $23$ and $15$. Since $\calX_i^\star$ is not low-rank, we use Sparse IHT with the smallest sparsity level ($s=15$), the largest one ($s=53$), and the selected one, and compare it with SGD. We use $d_4=20$, $\mu=0.1$ for both SGD and Sparse IHT, and $\eta_s=0.05$, and the MSE is averaged over $5$ trials.

\begin{figure}[!ht]
\centering
\subfigure[]{
\begin{minipage}[t]{0.45\textwidth}
\centering
\includegraphics[width=5cm]{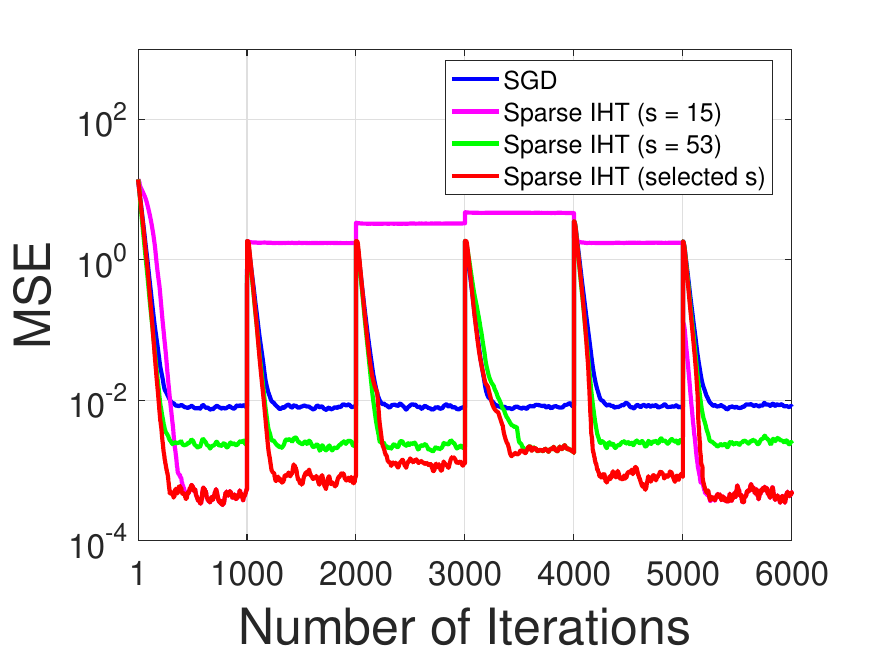}
\end{minipage}
\label{fig:quantum_mse}
}
\subfigure[]{
\begin{minipage}[t]{0.45\textwidth}
\centering
\includegraphics[width=5cm]{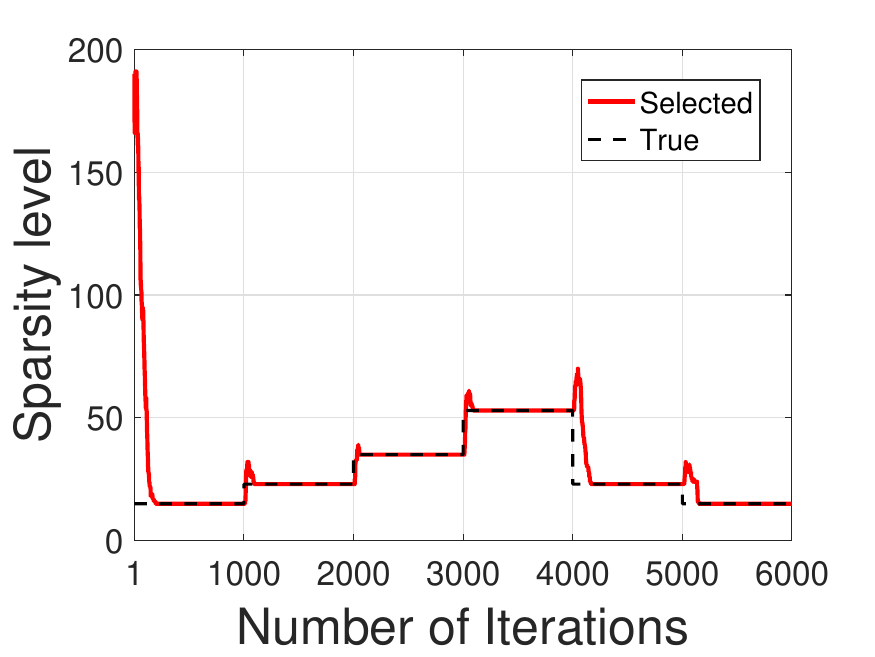}
\end{minipage}
\label{fig:quantum_s}
}
\caption{Tracking a drifting CNOT gate, where the coherent errors are switched at $i=1000,2000,\dots,5000$: (a) MSE of SGD and of Sparse IHT with $s=15$, $s=53$ and the selected sparsity level; (b) selected and true sparsity levels.}
\label{fig:quantum}
\end{figure}

Fig.~\ref{fig:quantum_mse} shows that Sparse IHT with $s=15$ works only in the first and last phases, in which there are no coherent errors; in the other phases, the coherent errors cannot be represented, and the MSE stagnates far above that of SGD. With $s=53$, all errors can be represented, but more parameters than necessary are estimated, so that the MSE is higher than with the selected sparsity level, except in the fourth phase, in which $s=53$ is the true value. With the selected sparsity level, Sparse IHT attains the lowest MSE in every phase. Fig.~\ref{fig:quantum_s} shows that the selected sparsity level follows each switching within about $150$ iterations. Hence, it also indicates when coherent errors appear or disappear.

\subsection[Citywide Crowd-flow Prediction]{Citywide Crowd-flow Prediction}
\label{subsec:traffic}
Next, we apply the adaptive methods, i.e., SGD, sparse IHT (S IHT), low-rank IHT (LR IHT), and sparse and low-rank IHT (S \& LR IHT), to citywide crowd-flow prediction on the TaxiBJ dataset~\cite{zhang2017deep}, which records the number of taxis entering each cell of a $32\times32$ grid over Beijing every $30$ minutes. We use the inflow maps from November 2015 to April 2016 and keep the $138$ days without missing slots, which gives $N=6624$ responses $\mY_i\in\mathbb{R}^{32\times32}$. The last $28$ days ($1344$ slots) form the testing period, and the preceding $110$ days form the training period. The input $\va_i\in\mathbb{R}^{19}$ consists of a constant, $\sin(2\pi k\tau_i/48)$ and $\cos(2\pi k\tau_i/48)$ for $k=1,2$, where $\tau_i\in\{0,\dots,47\}$ is the time slot of the day, the weekend and holiday indicators, the temperature, the wind speed, and the scores of $\mY_{i-1}$ on the $10$ leading principal components of the training maps. The mean training map $\bar{\mY}$ is subtracted from the responses.

All responses are observed in both periods. At each time $i$, the prediction $\widehat{\mY}_i=\bar{\mY}+\calX_i\times_3^1\va_i$ is made before $\mY_i$ is used, and the estimate is then updated with the window of the $d_4=40$ most recent samples, so that the updates continue in the testing period. We set $\mu=\mu_1=\mu_2=5\times10^{-3}$, $N_{\max}=10$ and $\epsilon=\epsilon_1=\epsilon_2=10^{-3}$, and use the thresholds $\eta_s=0.01$ for S IHT, $\eta_r=0.03$ for LR IHT (Tucker) and (TT), $\eta_r=0.07$ for LR IHT (Tubal), and $\eta_s=0.1$, $\eta_r=0.15$ for S \& LR IHT. All estimates are initialized as zero tensors. We compare with the persistence model $\widehat{\mY}_i=\mY_{i-1}$~\cite{sun2023matrix}, the daily persistence model $\widehat{\mY}_i=\mY_{i-48}$, and the MAR(1) model $\widehat{\mY}_i=\bar{\mY}+\mB_1(\mY_{i-1}-\bar{\mY})\mB_2^\top$~\cite{chen2021autoregressive}, whose coefficients $\mB_1,\mB_2\in\mathbb{R}^{32\times32}$ are estimated by alternating least squares on the training period and kept fixed in the testing period.

A few slots of the training period are data outages with almost empty maps, for which the relative error of a single slot is unbounded. We therefore measure the accuracy over a period $\mathcal{T}$ by
\begin{eqnarray}
    \label{NMSE traffic}
    \mathrm{NMSE}=\frac{\sum_{i\in\mathcal{T}}\|\widehat{\mY}_i-\mY_i\|_F^2}{\sum_{i\in\mathcal{T}}\|\mY_i\|_F^2}.
\end{eqnarray}
The training NMSE excludes the first $28$ days of the training period, in which the estimates initialized at zero have not yet reached the steady state. The runtime is the total time of the predictions and the updates in each period, and the storage is the number of stored parameters, averaged over the testing period, relative to the $d_1d_2d_3=19456$ entries of the full tensor.

\begin{table}[!ht]
\centering
\footnotesize
\caption{Crowd-flow prediction on TaxiBJ in terms of NMSE \eqref{NMSE traffic}, runtime, and storage.}
\label{tab:traffic}
\begin{tabular}{lccccc}
\hline
\multirow{2}{*}{\textbf{Method}} & \multicolumn{2}{c}{\textbf{NMSE}} & \multicolumn{2}{c}{\textbf{Runtime (s)}} & \multirow{2}{*}{\textbf{Storage}} \\
\cline{2-5}
 & \textbf{Training} & \textbf{Testing} & \textbf{Training} & \textbf{Testing} & \\
\hline
SGD & \textbf{0.017322} & \textbf{0.017401} & 1.26 & 0.30 & 100\% \\
S IHT & 0.017490 & 0.017553 & 5.58 & 1.40 & 43.95\% \\
LR IHT (Tucker) & 0.018016 & 0.018040 & 23.70 & 6.24 & 62.52\% \\
LR IHT (TT) & 0.017712 & 0.017774 & 15.20 & 3.86 & 70.10\% \\
LR IHT (Tubal) & 0.018443 & 0.018380 & 39.56 & 9.97 & 88.87\% \\
S \& LR IHT (Tucker) & 0.017757 & 0.017822 & 32.00 & 6.84 & 52.97\% \\
S \& LR IHT (TT) & 0.017346 & 0.017500 & 18.14 & 4.56 & 69.69\% \\
S \& LR IHT (Tubal) & 0.017774 & 0.018174 & 38.12 & 9.58 & 54.86\% \\
\hline
Persistence & 0.024548 & 0.025740 & -- & $<$0.01 & 0 \\
Daily persistence & 0.077759 & 0.104010 & -- & $<$0.01 & 0 \\
MAR(1) & 0.020306 & 0.020927 & 3.11 & 0.01 & 10.53\% \\
\hline
\end{tabular}
\end{table}

\begin{figure}[!ht]
\centering
\subfigure[]{
\begin{minipage}[t]{0.31\textwidth}
\centering
\includegraphics[width=\linewidth]{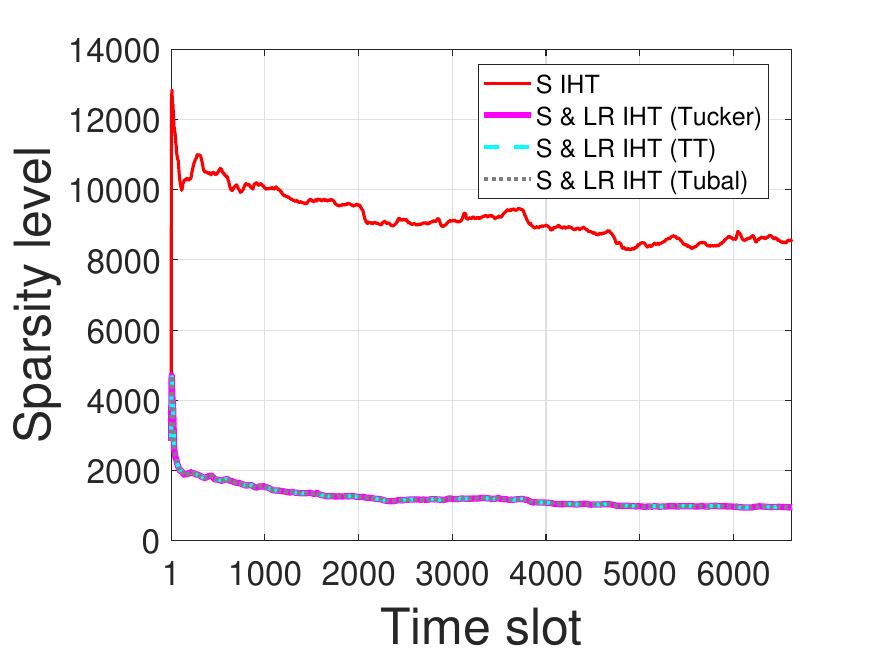}
\end{minipage}
\label{fig:traffic_s}
}
\subfigure[]{
\begin{minipage}[t]{0.31\textwidth}
\centering
\includegraphics[width=\linewidth]{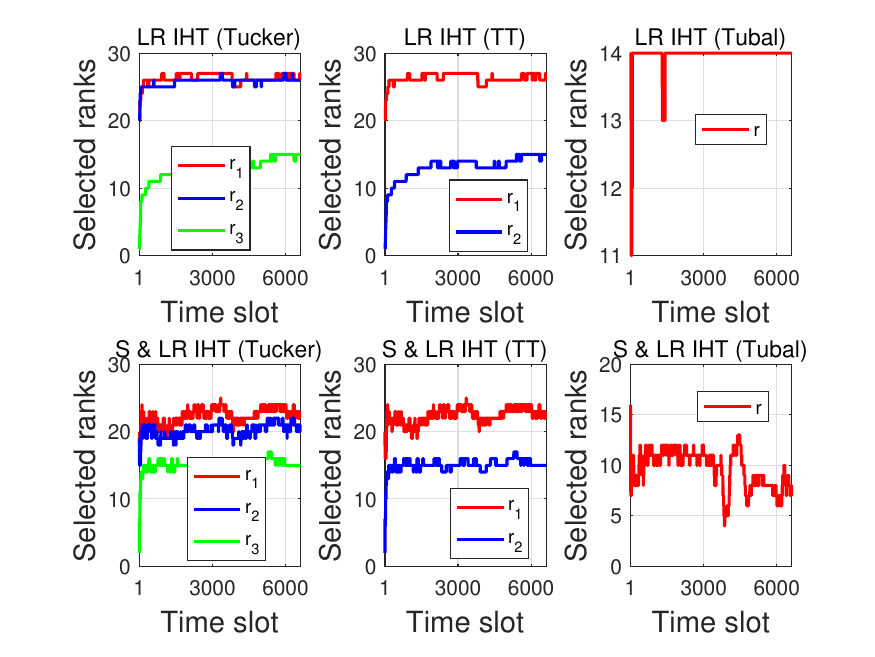}
\end{minipage}
\label{fig:traffic_r}
}
\subfigure[]{
\begin{minipage}[t]{0.31\textwidth}
\centering
\includegraphics[width=\linewidth]{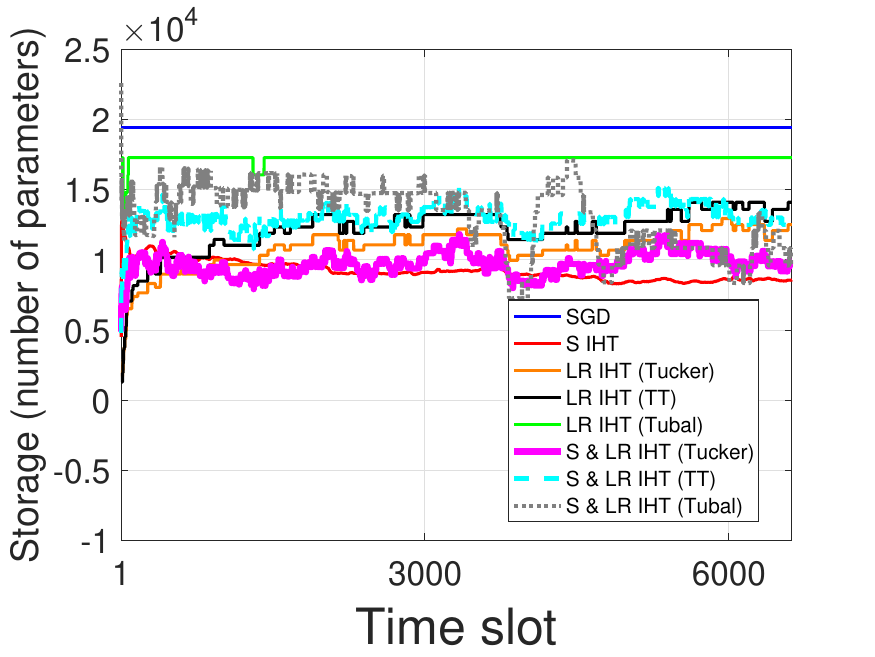}
\end{minipage}
\label{fig:traffic_storage}
}
\caption{Structure selected by the adaptive methods on TaxiBJ: (a) sparsity level of S IHT and S \& LR IHT; (b) ranks of LR IHT and S \& LR IHT; (c) storage.}
\label{fig:traffic}
\end{figure}

As shown in Table~\ref{tab:traffic}, all adaptive methods attain a lower training and testing NMSE than MAR(1) and the persistence models. Unlike MAR(1), the ToM regression also uses the time of day, the calendar, and the weather, whose effects on the crowd flow are captured by the regression tensor. The daily persistence model is the least accurate, since the flow differs between consecutive days, e.g., between weekdays and weekends. SGD attains the lowest training and testing NMSE. The crowd-flow regression tensor is not exactly sparse or low-rank, and the projections discard small entries and singular values that still contribute to the prediction. Since the responses are completely observed and each window contains more samples than the input dimension, the unstructured estimate is already accurate, and the structure brings no gain in accuracy. Nevertheless, the structured updates reduce the storage at a small loss of accuracy. S IHT and S \& LR IHT (TT) are the closest to SGD, and S IHT stores less than half of its parameters, whereas LR IHT (Tubal) stores the most and is the least accurate, so that the tubal format suits this regression tensor least. In terms of runtime, SGD is the fastest, S IHT only requires the selection of the largest entries, and the low-rank projections are more expensive because of the SVDs, in particular in the tubal format, which requires the SVD of every frontal slice in the Fourier domain.

\Cref{fig:traffic} shows the selected structure over time. After a short transient at the beginning, in which the unstructured iterate grows from zero, the sparsity level and the ranks change only slowly, without an abrupt change between the training and testing periods. The sparsity level of S IHT decreases slowly, whereas the rank of the third mode of LR IHT (Tucker) and (TT) increases, i.e., more combinations of the inputs become relevant as the estimate accumulates information. The ranks of the two spatial modes remain large, so that the crowd-flow maps have a rich spatial structure and the low-rank formats reduce the storage only moderately. In S \& LR IHT, the sparse component keeps only a small number of entries, which capture localized patterns of the flow, and allows smaller spatial ranks of the low-rank component than in LR IHT. Therefore, S \& LR IHT stores less than LR IHT in the same format. The storage of all structured updates remains below that of SGD. It changes only in small steps for LR IHT, whose ranks rarely change, and fluctuates the most for S \& LR IHT (Tubal), whose tubal rank varies the most.

\subsection[Global TEC Forecasting]{Global TEC Forecasting}
\label{subsec:tec}
Finally, we validate the effectiveness of the adaptive methods, including SGD, sparse IHT (S IHT), low-rank IHT (LR IHT), and sparse and low-rank IHT (S \& LR IHT), for predicting the global total electron content (TEC) distribution. We use the Video Imputation with SoftImpute, temporal smoothing, and auxiliary (VISTA) dataset \cite{sun2023complete}, where $\{\mY_i\}_{i=1}^{N} \in \mathbb{R}^{181 \times 361}$ denotes the sequence of multivariate responses. The spatial-temporal resolution of the data used here is $1^\circ \times 1^\circ \times 4\,\mathrm{h}$. We consider data from September 2017, where the entire month forms a matrix-valued time series with $N = 180$. For the auxiliary covariates $\{\va_i\}_{i=1}^{N} \in \mathbb{R}^{3}$, we collect 4-hour resolution IMF Bz and Sym-H time series, which characterize near-Earth magnetic field and plasma conditions \cite{papitashvili2014omni}. In addition, we include the daily F10.7 index, which measures the solar radio flux at 10.7 cm, as an auxiliary predictor. The IMF Bz and Sym-H data are obtained from the OMNI dataset \cite{king2020omni}, while the F10.7 index is accessed from the NOAA data repository \cite{tapping201310}.

During the training phase, the response $\calY_i(s_1,s_2,s_4) = \mY_{i-s_4+1}$ is constructed from the ground-truth data. In contrast, during the testing phase, we define $\calY_i(s_1,s_2,1) = (\calX_{i} \times_{3}^{1} \mA_i)(s_1,s_2,1)$ is generated via the model, while all other entries are taken from the ground-truth data. This is due to the causal constraint that only past responses are available at time $i$. For the algorithmic setup, the number of most recent input vectors is set to $d_4 = 10 > d_3$. The first $125$ maps are used for training, and the remaining $55$ maps are used for testing. Common parameters are set as $N_{\max} = 10$ and $\epsilon = 10^{-3}$ (with $\epsilon_1 = \epsilon_2 = 10^{-3}$ for S \& LR IHT). The step sizes are $\mu = 5\times10^{-5}$ for SGD, S IHT and LR IHT, and $\mu_1 = 5\times10^{-6}$ and $\mu_2 = 4\times10^{-5}$ for S \& LR IHT. The thresholds are $\eta_s = 0.05$ for S IHT, $\eta_r = 0.05$ for LR IHT, and $\eta_s = \eta_r = 0.3$ for S \& LR IHT; with these fixed thresholds, the sparsity level and the ranks are selected at every window. All methods are initialized using a zero tensor initialization. We evaluate the performance by the NMSE \eqref{NMSE traffic} over the training or testing period, computed with the complete maps $\mY_i$. Finally, we note that, in all subsequent tables, the minimum NMSE value in each column is highlighted in boldface.

\paragraph{Incomplete information:} In practical scenarios, the responses are often only partially observed, and it is important to evaluate whether accurate prediction of the full responses remains feasible under such conditions. To this end, we randomly sample $10\%$, $30\%$, $50\%$, $70\%$, $90\%$, and $100\%$ of the entries in each response $\mY_i$ and apply the proposed methods. All methods use only the observed entries. The losses of the adaptive updates contain only the observed entries of the responses in each window. The persistence model predicts $\mY_i$ by the observed $\mY_{i-1}$, with the unobserved entries set to zero, and the MAR(1) model $\widehat{\mY}_i=\bar{\mY}+\mB_1(\mY_{i-1}-\bar{\mY})\mB_2^\top$~\cite{chen2021autoregressive} is fitted by least squares on the observed entries of the training period, where $\bar{\mY}$ is the mean of the observed training entries and the unobserved entries of $\mY_{i-1}-\bar{\mY}$ are set to zero.

\begin{table*}[!ht]
\centering
\footnotesize
\caption{Testing NMSE of TEC forecasting under different sampling ratios, and storage and runtime for the sampling ratio $50\%$.}
\label{tab:tec_sampling}
\begin{tabular}{lcccccccc}
\hline
\multirow{2}{*}{\textbf{Method}} & \multicolumn{6}{c}{\textbf{Testing NMSE}} & \multirow{2}{*}{\textbf{Storage}} & \multirow{2}{*}{\textbf{Runtime (s)}} \\
\cline{2-7}
 & $10\%$ & $30\%$ & $50\%$ & $70\%$ & $90\%$ & $100\%$ & & \\
\hline
SGD & 0.077744 & 0.070894 & 0.068740 & 0.067698 & 0.067143 & 0.067024 & 100\% & 1.4 \\
S IHT & 0.079835 & 0.072280 & 0.069769 & 0.068274 & 0.067402 & 0.067185 & 40.23\% & 3.5 \\
LR IHT (Tucker) & 0.070861 & 0.067368 & 0.066530 & 0.066643 & 0.067199 & 0.067731 & 0.84\% & 14.2 \\
LR IHT (TT) & 0.070857 & 0.067167 & 0.066571 & 0.066761 & 0.067431 & 0.067524 & 1.39\% & 7.2 \\
LR IHT (Tubal) & \textbf{0.070708} & \textbf{0.066214} & \textbf{0.065671} & 0.065844 & 0.066389 & 0.066789 & 2.49\% & 11.8 \\
S \& LR IHT (Tucker) & 0.073153 & 0.067575 & 0.066336 & 0.066105 & 0.066095 & 0.066107 & 7.91\% & 20.6 \\
S \& LR IHT (TT) & 0.073163 & 0.067547 & 0.066276 & 0.066088 & 0.066054 & 0.066120 & 8.28\% & 11.5 \\
S \& LR IHT (Tubal) & 0.072807 & 0.066857 & 0.065724 & \textbf{0.065600} & \textbf{0.065627} & \textbf{0.065773} & 9.02\% & 17.7 \\
\hline
Persistence & 0.910450 & 0.731000 & 0.551290 & 0.372400 & 0.193450 & 0.103860 & 0 & -- \\
MAR(1) & 0.084892 & 0.128020 & 0.161220 & 0.174640 & 0.271290 & 0.300140 & 83.20\% & -- \\
\hline
\end{tabular}
\end{table*}

\begin{figure}[!ht]
\centering
\subfigure[]{
\begin{minipage}[t]{0.45\textwidth}
\centering
\includegraphics[width=\linewidth]{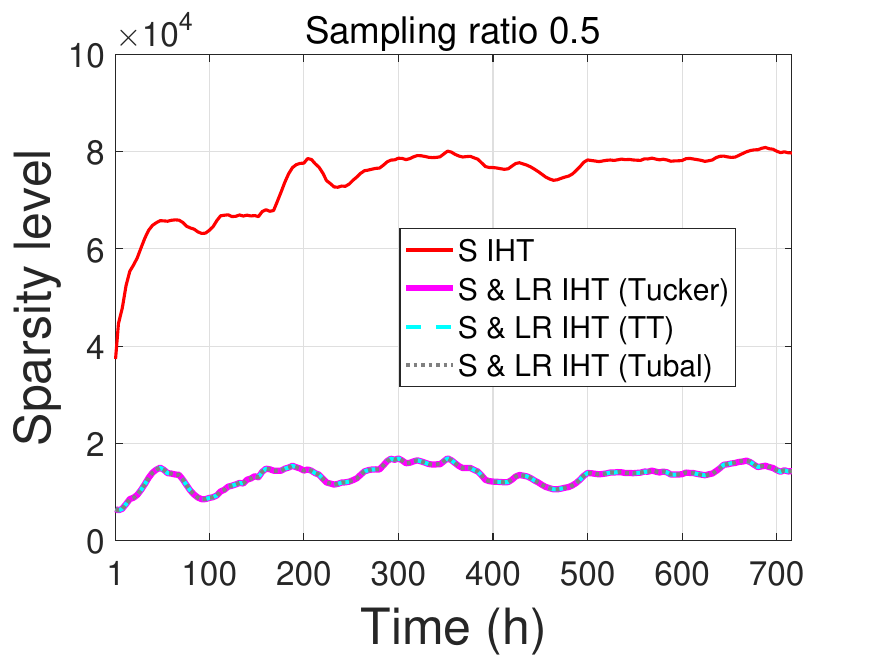}
\end{minipage}
\label{fig:tec_s}
}
\subfigure[]{
\begin{minipage}[t]{0.45\textwidth}
\centering
\includegraphics[width=\linewidth]{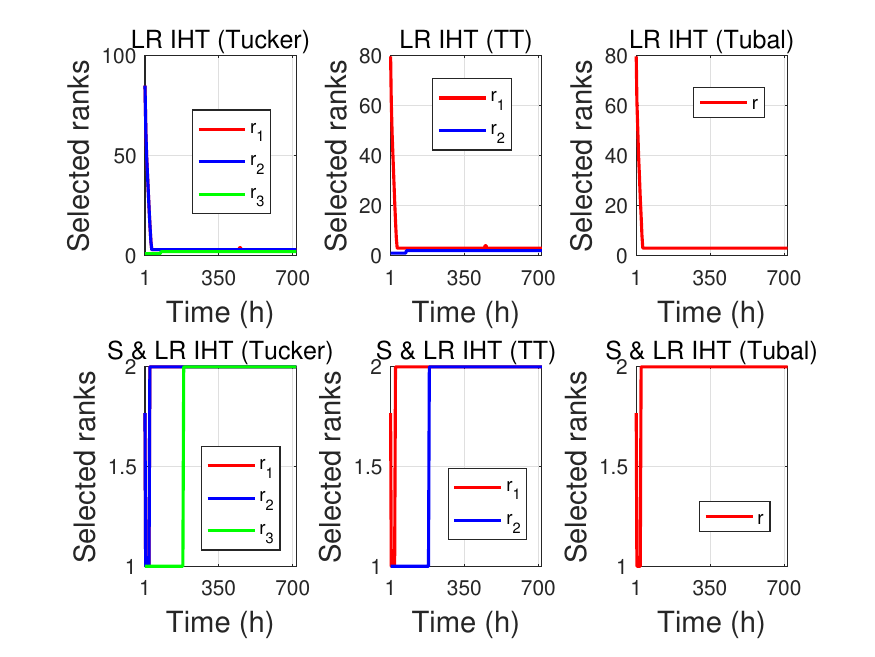}
\end{minipage}
\label{fig:tec_r}
}
\caption{Structure selected by the adaptive methods in TEC forecasting for the sampling ratio $50\%$: (a) sparsity level of S IHT and S \& LR IHT; (b) ranks of LR IHT and S \& LR IHT.}
\label{fig:tec_structure}
\end{figure}

\begin{figure}[!ht]
\centering
\subfigure[]{
\begin{minipage}[t]{0.45\textwidth}
\centering
\includegraphics[width=\linewidth]{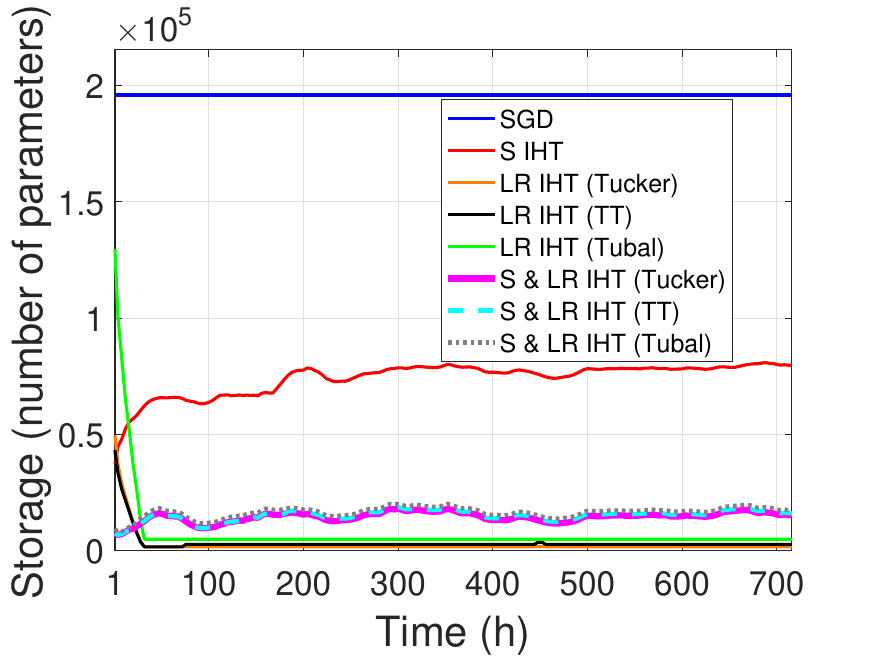}
\end{minipage}
\label{fig:tec_storage_t}
}
\subfigure[]{
\begin{minipage}[t]{0.45\textwidth}
\centering
\includegraphics[width=\linewidth]{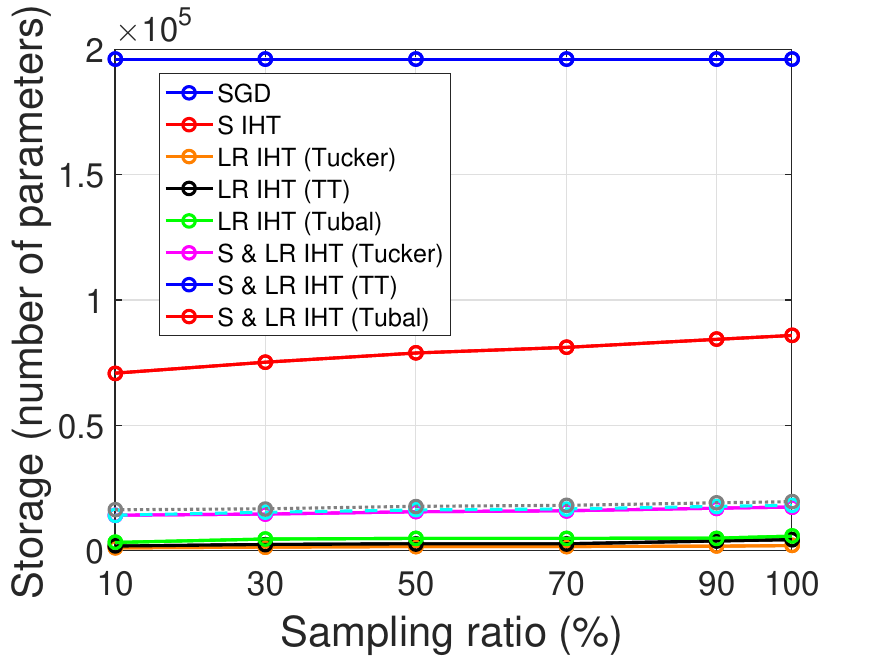}
\end{minipage}
\label{fig:tec_storage_ratio}
}
\caption{Storage of the adaptive methods in TEC forecasting: (a) over time for the sampling ratio $50\%$; (b) average over the testing period versus the sampling ratio.}
\label{fig:tec_storage}
\end{figure}

As shown in Table~\ref{tab:tec_sampling}, the testing NMSE of the adaptive methods increases only slightly as the sampling ratio decreases. The inputs are fully observed and the regression tensor changes slowly, so that each update uses the observed entries of the window, and an entry that is unobserved in one map is updated with the maps in which it is observed. In contrast, the testing NMSE of the persistence model increases rapidly, since it propagates the unobserved entries; even with complete responses, it is less accurate than the adaptive methods, since the TEC distribution changes considerably within four hours. MAR(1) predicts $\mY_i$ from the previous map $\mY_{i-1}$, which is observed four hours earlier and only partially, and its $d_1^2+d_2^2=163{,}082$ parameters are estimated from only $125$ training maps and kept fixed in the testing period. In contrast, the adaptive methods use the inputs at the current time and update the regression tensor with every new map, and all of them attain a lower testing NMSE than MAR(1) at every sampling ratio.

Unlike in the crowd-flow experiment, the structured updates improve both the accuracy and the storage. LR IHT (Tubal) attains the lowest testing NMSE for small sampling ratios, and S \& LR IHT (Tubal) for large sampling ratios. The TEC distribution is dominated by a few large-scale spatial patterns, so that the regression tensor is well approximated by a low-rank tensor. The low-rank projection couples the entries through the common factors, so that the estimate of an unobserved entry is also corrected by the observed ones, which is most useful when few entries are observed. With more observed entries, the sparse component of S \& LR IHT additionally captures localized deviations from these patterns. In contrast, S IHT keeps a large part of the entries without improving on SGD, so that the TEC regression tensor is better described by a low-rank than by a sparse structure.

As shown in \Cref{fig:tec_structure}, LR IHT and S \& LR IHT select very small ranks after the first map, whose window contains a single map, and change them only occasionally afterwards. Therefore, they store only a small fraction of the parameters of SGD. The sparsity level of S IHT increases at the beginning and then remains almost constant. As shown in \Cref{fig:tec_storage}, the storage changes with the selected structure over time and is slightly smaller for lower sampling ratios, since the methods select smaller sparsity levels and ranks when fewer entries are observed.

\section{Conclusion}
\label{conclusion}

We developed an adaptive tensor prediction framework for streaming matrix-valued data based on the ToM regression model, in which the responses in a sliding window are lifted to a tensor and the regression tensor is updated online by SGD. We characterized the steady-state and tracking errors of SGD with overlapping windows, and established fixed-window recovery guarantees under sparse, low-rank and sparse plus low-rank models. These results show that the recovery error depends on the intrinsic degrees of freedom rather than the ambient dimension, which justifies structural modeling in adaptive settings. Experiments on synthetic data, the tracking of a drifting quantum gate, crowd-flow prediction and TEC forecasting validated the effectiveness of the structured adaptive updates. The convergence analysis of the structured updates and the extension to other streaming applications are left for future work.

\section{Acknowledgments}
\label{sec: ack}

We acknowledge funding support from NASA Grants No. 80NSSC23M0191 and 80NSSC23M0192. We thank Kevin Jin for providing and organizing the dataset used in this work. ZQ gratefully acknowledges support from the MICDE Research Scholars Program at the University of Michigan.

\clearpage

\crefalias{section}{appendix}
\appendices

\section{Proof of \Cref{The main theorem of error in the model II}}
\label{Proof of error analysis SGD of model II}

\subsection[The averaged recursion]{The averaged recursion}
\label{sec long term model}

\paragraph[Rewriting the update.]{Rewriting the update.}
 By \eqref{ToM regression general data model II}, the $(k+1)$-th slice of the residual $\calX_i\times_3^1\mA_i-\calY_i$ is $(\calX_i-\calX^\star_{i-k})\times_3^1\va_{i-k}-\calE_i(:,:,k+1)$, where $\va_{i-k}=\mA_i(:,k+1)$. Hence, by the definition of $\times_3^2$ below \eqref{updating method X another model II}, the update \eqref{updating method X another model II} reads
\begin{eqnarray}
    \label{exact update tensor}
    \calX_{i+1}=\calX_i-\frac{\mu}{d_4}\sum_{k=0}^{d_4-1}\big(\calX_i-\calX^\star_{i-k}\big)\times_3^1\big(\va_{i-k}\va_{i-k}^\top\big)+\frac{\mu}{d_4}\calE_i\times_3^2\mA_i,
\end{eqnarray}
where $(\calZ\times_3^1\mM)(s_1,s_2,s_3')=\sum_{s_3=1}^{d_3}\calZ(s_1,s_2,s_3)\mM(s_3,s_3')$ for $\calZ\in\R^{d_1\times d_2\times d_3}$ and $\mM\in\R^{d_3\times d_3}$.

\paragraph[Averaged recursion.]{Averaged recursion.}
 We define the averaged iterate $\wt\calX_i$ by replacing $\va_{i-k}\va_{i-k}^\top$ in \eqref{exact update tensor} with its mean $\sigma_a^2\mId_{d_3\times d_3}$:
\begin{eqnarray}
    \label{long term iterate}
    \wt\calX_0=\calX_0,\qquad
    \wt\calX_{i+1}=\wt\calX_i-\frac{\mu\sigma_a^2}{d_4}\sum_{k=0}^{d_4-1}\big(\wt\calX_i-\calX^\star_{i-k}\big)+\frac{\mu}{d_4}\calE_i\times_3^2\mA_i .
\end{eqnarray}
Its distance to $\calX_i$ is bounded in \Cref{lemma:long term accuracy}.

\subsection[Steady state of the averaged recursion]{Steady state of the averaged recursion}
\label{Proof of steady state long term}

\begin{proposition}
\label{prop:steady long term}
Let $\rho_{d_4}=\frac{1}{d_4^2}\sum_{k=0}^{d_4-1}\sum_{k'=0}^{d_4-1}(1-\mu\sigma_a^2)^{|k-k'|}$. Under assumptions (i)--(iii), provided that the step size obeys $\mu<\frac{2}{\sigma_a^2}$,
\begin{eqnarray}
    \label{steady state long term main}
    \lim_{i\to\infty}\E[\|\wt\calX_i-\calX^\star\|_F^2]=\frac{\mu d_1d_2d_3\sigma_e^2}{2-\mu\sigma_a^2}\,\rho_{d_4},\qquad
    \rho_{d_4}=1-\mu\sigma_a^2\frac{d_4^2-1}{3d_4}+O(\mu^2).
\end{eqnarray}
\end{proposition}

\begin{proof}
First, we define the following notations:
\begin{eqnarray}
    \label{new notations sliding window}
    \mX_i &\!\!\!\!=\!\!\!\!& \text{reshape}(\calX_i,[d_1d_2, d_3] ),\quad
    \wt\mX_i = \text{reshape}(\wt\calX_i,[d_1d_2, d_3] ),\quad
    \mX_i^\star = \text{reshape}(\calX_i^\star,[d_1d_2, d_3] ),\nonumber\\
    \vy_i &\!\!\!\!=\!\!\!\!& \text{reshape}(\mY_i,[d_1d_2, 1] ),\quad
    \ve_i = \text{reshape}(\calE_i(:,:,1),[d_1d_2, 1] ).
\end{eqnarray}
By \eqref{stacked input signal} and \eqref{window overlap}, the $(k+1)$-th slice of \eqref{ToM regression general data model II} reads $\vy_{i-k}=\mX^\star_{i-k}\va_{i-k}+\ve_{i-k}$. Then, \eqref{exact update tensor} and \eqref{long term iterate} can be rewritten as
\begin{eqnarray}
    \label{update with residual}
    \mX_{i+1} &\!\!\!\!=\!\!\!\!& \mX_i-\frac{\mu}{d_4}\sum_{k=0}^{d_4-1}(\mX_i-\mX^\star_{i-k})\va_{i-k}\va_{i-k}^\top+\mu\mG_i,\\
    \label{long term iterate matrix}
    \wt\mX_{i+1} &\!\!\!\!=\!\!\!\!& \wt\mX_i-\frac{\mu\sigma_a^2}{d_4}\sum_{k=0}^{d_4-1}\big(\wt\mX_i-\mX^\star_{i-k}\big)+\mu\mG_i,
\end{eqnarray}
where $\mG_i=\frac{1}{d_4}\sum_{k=0}^{d_4-1}\ve_{i-k}\va_{i-k}^\top$.

Let $\beta=1-\mu\sigma_a^2$ and $\wt\mDelta_i=\wt\mX_i-\mX^\star$. Under (iii), subtracting $\mX^\star$ from both sides of \eqref{long term iterate matrix}, we obtain the first equation below, and iterating it gives the second for $i\ge1$:
\begin{eqnarray}
    \label{long term model sliding}
    \wt\mDelta_{i+1}&\!\!\!\!=\!\!\!\!&\beta\wt\mDelta_i+\mu\mG_i,\\
    \label{unrolled sliding}
    \wt\mDelta_i&\!\!\!\!=\!\!\!\!&\beta^{i}\wt\mDelta_0+\mu\sum_{j=0}^{i-1}\beta^{j}\mG_{i-1-j}.
\end{eqnarray}
Since $\wt\mDelta_0$ is deterministic and $\E[\mG_t]={\bf 0}$, we have
\begin{eqnarray}
    \label{expectation unrolled sliding}
    \E[\|\wt\mDelta_i\|_F^2]=\beta^{2i}\|\wt\mDelta_0\|_F^2+\mu^2\sum_{j=0}^{i-1}\sum_{l=0}^{i-1}\beta^{j+l}r(j-l),\qquad r(m)=\E[\trace(\mG_t^\top\mG_{t-m})],
\end{eqnarray}
where $r(m)$ does not depend on $t$ because $\{\va_i,\ve_i\}$ are i.i.d.

Now, we compute $r(m)$. For any $s,u$,
\begin{eqnarray}
    \label{single term correlation}
    \E[\trace((\ve_s\va_s^\top)^\top\ve_u\va_u^\top)]=\E[(\ve_s^\top\ve_u)(\va_u^\top\va_s)]
    =\begin{cases}
    d_1d_2d_3\sigma_a^2\sigma_e^2, & s=u,\\
    0, & s\neq u,
    \end{cases}
\end{eqnarray}
where we use the independence of the noise and the inputs, $\E[\ve_s]={\bf 0}$, $\E[\|\ve_s\|_2^2]=d_1d_2\sigma_e^2$ and $\E[\|\va_s\|_2^2]=d_3\sigma_a^2$. Substituting the definition of $\mG_t$ into $r(m)$, we obtain
\begin{eqnarray}
    \label{r expansion}
    r(m)=\frac{1}{d_4^2}\sum_{k=0}^{d_4-1}\sum_{k'=0}^{d_4-1}\E\big[\trace\big((\ve_{t-k}\va_{t-k}^\top)^\top\ve_{t-m-k'}\va_{t-m-k'}^\top\big)\big].
\end{eqnarray}
By \eqref{single term correlation}, a summand is nonzero only if $k=k'+m$. Such a pair $(k'+m,k')$ requires $0\le k'\le d_4-1$ and $0\le k'+m\le d_4-1$, i.e., $\max\{0,-m\}\le k'\le\min\{d_4-1,d_4-1-m\}$. Hence, for $|m|\le d_4-1$, the number of such pairs is $\min\{d_4-1,d_4-1-m\}-\max\{0,-m\}+1=d_4-|m|$, and for $|m|\ge d_4$ there is none. Writing $(x)_+=\max\{x,0\}$, this leads to
\begin{eqnarray}
    \label{autocorrelation G}
    r(m)=\frac{(d_4-|m|)_+}{d_4^2}\,d_1d_2d_3\sigma_a^2\sigma_e^2.
\end{eqnarray}

Next, we let $i\to\infty$ in \eqref{expectation unrolled sliding}. Since $|\beta|<1$, $\beta^{2i}\|\wt\mDelta_0\|_F^2\to0$, and, since $|r(m)|\le r(0)$ and $\sum_{j,l\ge0}\beta^{j+l}<\infty$, the sum $\sum_{j=0}^{i-1}\sum_{l=0}^{i-1}\beta^{j+l}r(j-l)$ converges absolutely as $i\to\infty$, so its terms can be regrouped by $m=j-l$ with
\begin{eqnarray}
    \label{grouping by lag}
    \sum_{j,l\ge0,\,j-l=m}\beta^{j+l}=\beta^{|m|}\sum_{l\ge0}\beta^{2l}=\frac{\beta^{|m|}}{1-\beta^2}.
\end{eqnarray}
Hence, combining \eqref{expectation unrolled sliding}, \eqref{autocorrelation G} and \eqref{grouping by lag}, we obtain
\begin{eqnarray}
    \label{steady state sliding derivation}
    \lim_{i\to\infty}\E[\|\wt\mDelta_i\|_F^2]
    =\frac{\mu^2 d_1d_2d_3\sigma_a^2\sigma_e^2}{(1-\beta^2)d_4^2}\sum_{|m|<d_4}(d_4-|m|)\beta^{|m|}
    = \frac{\mu d_1d_2d_3\sigma_e^2}{2-\mu\sigma_a^2}\,\rho_{d_4},
\end{eqnarray}
where the first equation uses \eqref{grouping by lag} and \eqref{autocorrelation G}, i.e., $\mu^2\sum_{j,l\ge0}\beta^{j+l}r(j-l)=\frac{\mu^2}{1-\beta^2}\sum_{|m|<d_4}\beta^{|m|}r(m)$, and the second equation uses $1-\beta^2=\mu\sigma_a^2(2-\mu\sigma_a^2)$ and
\begin{eqnarray}
    \label{rho lag form}
    \sum_{|m|<d_4}(d_4-|m|)\beta^{|m|}=\sum_{k=0}^{d_4-1}\sum_{k'=0}^{d_4-1}\beta^{|k-k'|}=d_4^2\rho_{d_4}.
\end{eqnarray}
Here, the first equation of \eqref{rho lag form} groups the pairs $(k,k')$ by $m=k-k'$, and each $m$ with $|m|<d_4$ has $d_4-|m|$ such pairs, as shown above \eqref{autocorrelation G}.

Finally, we expand $\rho_{d_4}$. Let $\nu=\mu\sigma_a^2\in(0,1)$. For an integer $n\ge0$, by Taylor's theorem and $\frac{\mathrm{d}^2}{\mathrm{d}\nu^2}(1-\nu)^n=n(n-1)(1-\nu)^{n-2}\in[0,n(n-1)]$, we have $\beta^n=(1-\nu)^n=1-n\nu+R_n$ with $0\le R_n\le\frac{n^2\nu^2}{2}$. Substituting this into \eqref{rho lag form}, we obtain
\begin{eqnarray}
    \label{rho expansion}
    d_4^2\rho_{d_4}&\!\!\!\!=\!\!\!\!&\sum_{|m|<d_4}(d_4-|m|)-\nu\sum_{|m|<d_4}(d_4-|m|)|m|+\sum_{|m|<d_4}(d_4-|m|)R_{|m|}\nonumber\\
    &\!\!\!\!=\!\!\!\!&d_4^2-\nu\,\frac{d_4(d_4^2-1)}{3}+\sum_{|m|<d_4}(d_4-|m|)R_{|m|},
\end{eqnarray}
where the second equation uses
\begin{eqnarray}
    \sum_{|m|<d_4}(d_4-|m|)&\!\!\!\!=\!\!\!\!&d_4+2\sum_{m=1}^{d_4-1}(d_4-m)=d_4+d_4(d_4-1)=d_4^2,\\
    \sum_{|m|<d_4}(d_4-|m|)|m|&\!\!\!\!=\!\!\!\!&2\sum_{m=1}^{d_4-1}m(d_4-m)=2\Big(d_4\frac{d_4(d_4-1)}{2}-\frac{(d_4-1)d_4(2d_4-1)}{6}\Big)=\frac{d_4(d_4^2-1)}{3}.
\end{eqnarray}
The last term of \eqref{rho expansion} lies in $[0,\nu^2d_4^4]$, since it has at most $2d_4$ terms, each at most $d_4\cdot\frac{d_4^2\nu^2}{2}$; hence it is $O(\mu^2)$. Dividing \eqref{rho expansion} by $d_4^2$ completes the proof.
\end{proof}

\subsection[Accuracy of the averaged recursion]{Accuracy of the averaged recursion}
\label{Proof of long term accuracy}

\begin{lemma}
\label{lemma:long term accuracy}
For a random matrix $\mM$, let $\|\mM\|_{L_2}=(\E[\|\mM\|_F^2])^{1/2}$. Let $\mDelta_i=\mX_i-\mX_i^\star$ and $\wt\mDelta_i=\wt\mX_i-\mX_i^\star$, with $\mX_t=\wt\mX_t=\mX_0$ for $t<0$, and let $\omega_t=\|\mX_{t+1}^\star-\mX_t^\star\|_{L_2}$. Suppose (i), (ii) and (iv), let $\calX_i^\star$ satisfy either (iii) or \eqref{RW model definition}--\eqref{AR solution}, and let $\mu$ satisfy \eqref{mu0 definition}. If $\sup_t\|\wt\mDelta_t\|_{L_2}<\infty$ and $\sup_t\omega_t<\infty$, then $\sup_i\|\mDelta_i-\wt\mDelta_i\|_{L_2}<\infty$ and
\begin{eqnarray}
    \label{long term accuracy bound}
    \limsup_{i\to\infty}\|\mDelta_i-\wt\mDelta_i\|_{L_2}\le2\sqrt{\mu}\,c_\eta\limsup_{i\to\infty}\|\wt\mDelta_i\|_{L_2}+2\mu\Big(c_\omega\limsup_{i\to\infty}\omega_i+c_G\Big),
\end{eqnarray}
where
\begin{eqnarray}
    \label{lemma constants}
    c_\eta=\sqrt{b_a}+\sqrt{\mu_0}\,\frac{(d_4-1)b_a^2}{2\sigma_a^2},\qquad c_\omega=\frac{(d_4-1)^2b_a^2}{2\sigma_a^2},\qquad c_G=\frac{(d_4-1)b_ag_2}{2\sigma_a^2},\qquad g_2=\sigma_a\sigma_e\sqrt{\frac{d_1d_2d_3}{d_4}} .
\end{eqnarray}
\end{lemma}

\begin{proof}
Let $\mD_i=\mDelta_i-\wt\mDelta_i=\mX_i-\wt\mX_i$, so that $\mD_i={\bf 0}$ for $i\le0$. We derive a recursion for $\mD_i$ (Steps~1--2), bound its two parts (Steps~3--4), and then show that $\|\mD_i\|_{L_2}$ is bounded (Step~5) and take the limit (Step~6).

\paragraph[Step 1: difference recursion.]{Step 1: difference recursion.}
 Let $\mH_t=\va_t\va_t^\top-\sigma_a^2\mId_{d_3\times d_3}$, which is symmetric with $\E[\mH_t]={\bf 0}$. Writing $\va_{i-k}\va_{i-k}^\top=\sigma_a^2\mId_{d_3\times d_3}+\mH_{i-k}$ in \eqref{update with residual} and subtracting \eqref{long term iterate matrix}, we have
\begin{eqnarray}
    \label{difference recursion}
    \mD_{i+1}=\beta\mD_i-\frac{\mu}{d_4}\sum_{k=0}^{d_4-1}(\mX_i-\mX^\star_{i-k})\mH_{i-k},
\end{eqnarray}
and therefore, since $\mD_0={\bf 0}$, for $i\ge1$,
\begin{eqnarray}
    \label{difference unrolled}
    \mD_i=-\frac{\mu}{d_4}\sum_{j=0}^{i-1}\beta^j\sum_{k=0}^{d_4-1}(\mX_{i-1-j}-\mX^\star_{i-1-j-k})\mH_{i-1-j-k}.
\end{eqnarray}
By (i) and (iv), $\sigma_a^2=\frac{1}{d_3}\E[\|\va_t\|_2^2]\le b_a$, and the eigenvalues of $\va_t\va_t^\top$ and $\mH_t$ lie in $[-b_a,b_a]$. Hence, since $\|\mM\mN\|_F\le\|\mM\|_F\|\mN\|$, we have
\begin{eqnarray}
    \label{bounded input facts}
    \|\mM\va_t\va_t^\top\|_F\le b_a\|\mM\|_F,\qquad \|\mM\mH_t\|_F\le b_a\|\mM\|_F .
\end{eqnarray}

\paragraph[Step 2: decomposition.]{Step 2: decomposition.}
 Substituting $\mX_s-\mX^\star_{s-k}=\mDelta_{s-k}+(\mX_s-\mX_{s-k})$ into \eqref{difference unrolled}, and exchanging the order of the sums in the first part, we obtain $\mD_i=\mD_i^{(1)}+\mD_i^{(2)}$ with
\begin{eqnarray}
    \label{split of D1}
    \mD_i^{(1)}&\!\!\!\!=\!\!\!\!&-\frac{\mu}{d_4}\sum_{k=0}^{d_4-1}\sum_{j=0}^{i-1}\beta^j\mDelta_{i-1-j-k}\mH_{i-1-j-k},\\
    \label{split of D}
    \mD_i^{(2)}&\!\!\!\!=\!\!\!\!&-\frac{\mu}{d_4}\sum_{j=0}^{i-1}\beta^j\sum_{k=1}^{d_4-1}(\mX_{i-1-j}-\mX_{i-1-j-k})\mH_{i-1-j-k}.
\end{eqnarray}

\paragraph[Step 3: the part $\mD_i^{(1)}$.]{Step 3: the part $\mD_i^{(1)}$.}
 Fix $i\ge1$ and suppose that $\|\mDelta_t\|_{L_2}<\infty$ for all $t\le i-1$. By \eqref{update with residual}, $\mDelta_t$ is a function of $\mX_0$, $\{\va_r,\ve_r\}_{r\le t-1}$ and $\{\mX_r^\star\}_{r\le t}$ only. Since $\{\va_r\}$ are i.i.d.\ and independent of $\{\ve_r\}$ and $\{\mX_r^\star\}$, $\mH_t$, which depends only on $\va_t$, is independent of $\mDelta_t$ and of $\mDelta_s\mH_s$ for $s<t$. Hence, since $\E[\mH_t]={\bf 0}$, for $s<t$,
\begin{eqnarray}
    \label{martingale orthogonality}
    \E[\trace((\mDelta_s\mH_s)^\top\mDelta_t\mH_t)]=\E\big[\trace\big((\mDelta_s\mH_s)^\top\mDelta_t\,\E[\mH_t]\big)\big]=0 .
\end{eqnarray}
Moreover, $\E[\|\mDelta_t\mH_t\|_F^2]=\E[\trace(\mDelta_t\E[\mH_t^2]\mDelta_t^\top)]$. For any $\vx\in\R^{d_3}$, $\vx^\top\E[\mH_t^2]\vx=\E[\|\mH_t\vx\|_2^2]\ge0$, and, using $\mH_t^2=\|\va_t\|_2^2\va_t\va_t^\top-2\sigma_a^2\va_t\va_t^\top+\sigma_a^4\mId_{d_3\times d_3}$, (i) and (iv),
\begin{eqnarray}
    \vx^\top\E[\mH_t^2]\vx=\E\big[\|\va_t\|_2^2(\va_t^\top\vx)^2\big]-\sigma_a^4\|\vx\|_2^2\le b_a\E\big[(\va_t^\top\vx)^2\big]=b_a\sigma_a^2\|\vx\|_2^2 .
\end{eqnarray}
Applying this to the rows of $\mDelta_t$, we obtain
\begin{eqnarray}
    \label{martingale second moment bound}
    \E[\|\mDelta_t\mH_t\|_F^2]\le b_a\sigma_a^2\,\|\mDelta_t\|_{L_2}^2 .
\end{eqnarray}
For each $k$, the terms of the inner sum in \eqref{split of D1} are therefore orthogonal, and its squared $L_2$ norm is the sum of the squared $L_2$ norms of its terms. Hence, by the Minkowski inequality over $k$, \eqref{martingale orthogonality} and \eqref{martingale second moment bound}, we have
\begin{eqnarray}
    \label{bound martingale part}
    \|\mD_i^{(1)}\|_{L_2}\le\frac{\mu}{d_4}\sum_{k=0}^{d_4-1}\Big(b_a\sigma_a^2\sum_{j=0}^{i-1}\beta^{2j}\|\mDelta_{i-1-j-k}\|_{L_2}^2\Big)^{1/2}.
\end{eqnarray}

\paragraph[Step 4: the part $\mD_i^{(2)}$.]{Step 4: the part $\mD_i^{(2)}$.}
 For $t\ge0$, by \eqref{update with residual}, the Minkowski inequality and \eqref{bounded input facts}, we have
\begin{eqnarray}
    \|\mX_{t+1}-\mX_t\|_{L_2}\le\frac{\mu b_a}{d_4}\sum_{k=0}^{d_4-1}\|\mX_t-\mX^\star_{t-k}\|_{L_2}+\mu\|\mG_t\|_{L_2}.
\end{eqnarray}
Writing $\mX_t-\mX^\star_{t-k}=\mDelta_t+\sum_{n=1}^{k}(\mX^\star_{t-n+1}-\mX^\star_{t-n})$ gives $\|\mX_t-\mX^\star_{t-k}\|_{L_2}\le\|\mDelta_t\|_{L_2}+\sum_{n=1}^{d_4-1}\omega_{t-n}$ for $k\le d_4-1$, and $\|\mG_t\|_{L_2}^2=r(0)=g_2^2$ by \eqref{autocorrelation G}. Hence
\begin{eqnarray}
    \label{one step increment}
    \|\mX_{t+1}-\mX_t\|_{L_2}\le\mu\Big(b_a\|\mDelta_t\|_{L_2}+b_a\sum_{n=1}^{d_4-1}\omega_{t-n}+g_2\Big),\qquad t\ge0,
\end{eqnarray}
while $\mX_{t+1}-\mX_t={\bf 0}$ for $t<0$. Writing $\mX_s-\mX_{s-k}=\sum_{l=1}^{k}(\mX_{s-l+1}-\mX_{s-l})$ in \eqref{split of D}, by the Minkowski inequality and \eqref{bounded input facts}, we obtain
\begin{eqnarray}
    \label{bound remainder part}
    \|\mD_i^{(2)}\|_{L_2}\le\frac{\mu b_a}{d_4}\sum_{k=1}^{d_4-1}\sum_{l=1}^{k}\sum_{j=0}^{i-1}\beta^j\|\mX_{i-j-l}-\mX_{i-1-j-l}\|_{L_2}.
\end{eqnarray}

\paragraph[Step 5: boundedness.]{Step 5: boundedness.}
 Let $\eta=\sqrt{\mu b_a}+\mu\frac{(d_4-1)b_a^2}{2\sigma_a^2}$. By \eqref{mu0 definition}, $\sqrt{\mu b_a}\le\frac18$ and $\mu\frac{(d_4-1)b_a^2}{2\sigma_a^2}<\frac14$, and, since $\mu\le\mu_0$, we have
\begin{eqnarray}
    \label{eta bound}
    \eta\le\frac12,\qquad \eta\le\sqrt{\mu}\,c_\eta .
\end{eqnarray}
We show by induction on $i$ that
\begin{eqnarray}
    \label{induction claim}
    \|\mD_i\|_{L_2}\le\sup_t\|\wt\mDelta_t\|_{L_2}+2\mu\Big(c_\omega\sup_t\omega_t+c_G\Big),
\end{eqnarray}
which holds for $i\le0$ since $\mD_i={\bf 0}$. Let $i\ge1$ and assume \eqref{induction claim} for all smaller indices. Since $\mDelta_t=\wt\mDelta_t+\mD_t$, the Minkowski inequality gives, for $t\le i-1$,
\begin{eqnarray}
    \label{Delta via D}
    \|\mDelta_t\|_{L_2}\le2\sup_t\|\wt\mDelta_t\|_{L_2}+2\mu\Big(c_\omega\sup_t\omega_t+c_G\Big),
\end{eqnarray}
and we denote the right-hand side by $Y$. We now bound the two parts of $\mD_i$ using $\|\mDelta_t\|_{L_2}\le Y$ for $t\le i-1$.

For $\mD_i^{(1)}$, all indices $i-1-j-k$ in \eqref{bound martingale part} are at most $i-1$. Hence, using $\|\mDelta_{i-1-j-k}\|_{L_2}\le Y$, then $\sum_{j\ge0}\beta^{2j}=\frac{1}{1-\beta^2}$, and finally $1-\beta^2=\mu\sigma_a^2(2-\mu\sigma_a^2)\ge\mu\sigma_a^2$, we have
\begin{eqnarray}
    \label{D1 bounded}
    \|\mD_i^{(1)}\|_{L_2}\le\frac{\mu}{d_4}\sum_{k=0}^{d_4-1}\Big(b_a\sigma_a^2Y^2\sum_{j\ge0}\beta^{2j}\Big)^{1/2}
    =\mu\Big(\frac{b_a\sigma_a^2}{1-\beta^2}\Big)^{1/2}Y\le\mu\Big(\frac{b_a\sigma_a^2}{\mu\sigma_a^2}\Big)^{1/2}Y=\sqrt{\mu b_a}\,Y,
\end{eqnarray}
where the equation uses that the $d_4$ summands over $k$ are equal.

For $\mD_i^{(2)}$, the increments in \eqref{bound remainder part} are $\mX_{t+1}-\mX_t$ with $t=i-1-j-l\le i-2$. For such $t\ge0$, \eqref{one step increment} with $\|\mDelta_t\|_{L_2}\le Y$ and $\omega_{t-n}\le\sup_t\omega_t$ gives
\begin{eqnarray}
    \|\mX_{t+1}-\mX_t\|_{L_2}\le\mu\Big(b_aY+(d_4-1)b_a\sup_t\omega_t+g_2\Big),
\end{eqnarray}
and this also holds for $t<0$, since the increment is then zero. Substituting this into \eqref{bound remainder part}, and using $\sum_{j\ge0}\beta^j=\frac{1}{1-\beta}=\frac{1}{\mu\sigma_a^2}$ and $\sum_{k=1}^{d_4-1}\sum_{l=1}^{k}1=\sum_{k=1}^{d_4-1}k=\frac{d_4(d_4-1)}{2}$, we have
\begin{eqnarray}
    \label{D2 bounded}
    \|\mD_i^{(2)}\|_{L_2}&\!\!\!\!\le\!\!\!\!&\frac{\mu b_a}{d_4}\cdot\frac{d_4(d_4-1)}{2}\cdot\frac{1}{\mu\sigma_a^2}\cdot\mu\Big(b_aY+(d_4-1)b_a\sup_t\omega_t+g_2\Big)\nonumber\\
    &\!\!\!\!=\!\!\!\!&\mu\frac{(d_4-1)b_a^2}{2\sigma_a^2}Y+\mu\Big(c_\omega\sup_t\omega_t+c_G\Big),
\end{eqnarray}
where the last equation uses the definitions of $c_\omega$ and $c_G$ in \eqref{lemma constants}.

Since $\mD_i=\mD_i^{(1)}+\mD_i^{(2)}$, adding \eqref{D1 bounded} and \eqref{D2 bounded} by the Minkowski inequality, and noting that the sum of the coefficients of $Y$ is $\eta$, we obtain, by \eqref{eta bound},
\begin{eqnarray}
    \label{induction step}
    \|\mD_i\|_{L_2}\le\eta Y+\mu\Big(c_\omega\sup_t\omega_t+c_G\Big)\le\frac12Y+\mu\Big(c_\omega\sup_t\omega_t+c_G\Big),
\end{eqnarray}
which equals the right-hand side of \eqref{induction claim}. This completes the induction; in particular, $\|\mDelta_t\|_{L_2}<\infty$ for all $t$, so Steps~3 and~4 apply for every $i$.

\paragraph[Step 6: limit.]{Step 6: limit.}
 Let $L=\limsup_{i\to\infty}\|\mD_i\|_{L_2}=\limsup_{i\to\infty}\|\mDelta_i-\wt\mDelta_i\|_{L_2}$, which is finite by Step~5. Since $\mDelta_t=\wt\mDelta_t+\mD_t$, the Minkowski inequality gives
\begin{eqnarray}
    \label{limsup Delta}
    \limsup_{t\to\infty}\|\mDelta_t\|_{L_2}\le\limsup_{t\to\infty}\|\wt\mDelta_t\|_{L_2}+L .
\end{eqnarray}
We use the following elementary fact: for a nonnegative sequence $\{x_t\}$ with $\sup_tx_t<\infty$ and $\gamma\in(0,1)$,
\begin{eqnarray}
    \label{elementary limit}
    \limsup_{i\to\infty}\sum_{j=0}^{i-1}\gamma^{j}x_{i-1-j}\le\frac{1}{1-\gamma}\limsup_{t\to\infty}x_t .
\end{eqnarray}
Indeed, for any $\epsilon>0$, choose $T$ such that $x_t\le\limsup_tx_t+\epsilon$ for $t\ge T$; then the terms with $i-1-j<T$ sum to at most $\frac{\gamma^{i-T}}{1-\gamma}\sup_tx_t$, which tends to zero, and the remaining terms sum to at most $\frac{\limsup_tx_t+\epsilon}{1-\gamma}$. All sequences below are bounded by Step~5.

For $\mD_i^{(1)}$, we take $\limsup_i$ in \eqref{bound martingale part}. For each $k$, applying \eqref{elementary limit} with $\gamma=\beta^2$ and $x_t=\|\mDelta_{t-k}\|_{L_2}^2$ to the inner sum gives $\limsup_i\sum_{j}\beta^{2j}\|\mDelta_{i-1-j-k}\|_{L_2}^2\le\frac{1}{1-\beta^2}(\limsup_t\|\mDelta_t\|_{L_2})^2$. Since this bound is the same for all $d_4$ values of $k$, the factor $\frac{1}{d_4}\sum_{k=0}^{d_4-1}$ in \eqref{bound martingale part} disappears, and, using $1-\beta^2\ge\mu\sigma_a^2$ and \eqref{limsup Delta}, we obtain
\begin{eqnarray}
    \label{D1 limit}
    \limsup_{i\to\infty}\|\mD_i^{(1)}\|_{L_2}&\!\!\!\!\le\!\!\!\!&\frac{\mu}{d_4}\sum_{k=0}^{d_4-1}\Big(\frac{b_a\sigma_a^2}{1-\beta^2}\Big)^{1/2}\limsup_{t\to\infty}\|\mDelta_t\|_{L_2}
    =\mu\Big(\frac{b_a\sigma_a^2}{1-\beta^2}\Big)^{1/2}\limsup_{t\to\infty}\|\mDelta_t\|_{L_2}\nonumber\\
    &\!\!\!\!\le\!\!\!\!&\mu\Big(\frac{b_a\sigma_a^2}{\mu\sigma_a^2}\Big)^{1/2}\limsup_{t\to\infty}\|\mDelta_t\|_{L_2}
    \le\sqrt{\mu b_a}\,\Big(\limsup_{t\to\infty}\|\wt\mDelta_t\|_{L_2}+L\Big).
\end{eqnarray}

For $\mD_i^{(2)}$, the bound \eqref{bound remainder part} is in terms of the increments $\|\mX_{t+1}-\mX_t\|_{L_2}$, so we first bound their limit superior. Taking $\limsup_t$ in \eqref{one step increment} and using \eqref{limsup Delta}, we have
\begin{eqnarray}
    \label{increment limit}
    \limsup_{t\to\infty}\|\mX_{t+1}-\mX_t\|_{L_2}\le\mu\Big(b_a\Big(\limsup_{t\to\infty}\|\wt\mDelta_t\|_{L_2}+L\Big)+(d_4-1)b_a\limsup_{t\to\infty}\omega_t+g_2\Big).
\end{eqnarray}
Now we take $\limsup_i$ in \eqref{bound remainder part}. For each $k,l$, applying \eqref{elementary limit} with $\gamma=\beta$ and $x_t=\|\mX_{t+1-l}-\mX_{t-l}\|_{L_2}$ to the sum over $j$ gives the factor $\frac{1}{1-\beta}=\frac{1}{\mu\sigma_a^2}$ times the left-hand side of \eqref{increment limit}, which is the same for all $k,l$. Hence, using $\sum_{k=1}^{d_4-1}\sum_{l=1}^{k}1=\frac{d_4(d_4-1)}{2}$, \eqref{increment limit} and the definitions of $c_\omega$ and $c_G$ in \eqref{lemma constants}, we obtain
\begin{eqnarray}
    \label{D2 limit}
    \limsup_{i\to\infty}\|\mD_i^{(2)}\|_{L_2}&\!\!\!\!\le\!\!\!\!&\frac{\mu b_a}{d_4}\cdot\frac{d_4(d_4-1)}{2}\cdot\frac{1}{\mu\sigma_a^2}\limsup_{t\to\infty}\|\mX_{t+1}-\mX_t\|_{L_2}\nonumber\\
    &\!\!\!\!\le\!\!\!\!&\mu\frac{(d_4-1)b_a^2}{2\sigma_a^2}\Big(\limsup_{t\to\infty}\|\wt\mDelta_t\|_{L_2}+L\Big)+\mu\Big(c_\omega\limsup_{t\to\infty}\omega_t+c_G\Big).
\end{eqnarray}
Since $\mD_i=\mD_i^{(1)}+\mD_i^{(2)}$, adding \eqref{D1 limit} and \eqref{D2 limit} by the Minkowski inequality, and noting that the sum of the coefficients of $\limsup_t\|\wt\mDelta_t\|_{L_2}+L$ is $\eta$, we obtain
\begin{eqnarray}
    L\le\eta\Big(\limsup_{t\to\infty}\|\wt\mDelta_t\|_{L_2}+L\Big)+\mu\Big(c_\omega\limsup_{t\to\infty}\omega_t+c_G\Big).
\end{eqnarray}
Since $L<\infty$, we can move $\eta L$ to the left-hand side, which gives $(1-\eta)L\le\eta\limsup_t\|\wt\mDelta_t\|_{L_2}+\mu(c_\omega\limsup_t\omega_t+c_G)$. Dividing by $1-\eta\ge\frac12$ and using $\eta\le\sqrt{\mu}\,c_\eta$ from \eqref{eta bound}, we finally obtain
\begin{eqnarray}
    \limsup_{i\to\infty}\|\mDelta_i-\wt\mDelta_i\|_{L_2}=L&\!\!\!\!\le\!\!\!\!&2\eta\limsup_{i\to\infty}\|\wt\mDelta_i\|_{L_2}+2\mu\Big(c_\omega\limsup_{i\to\infty}\omega_i+c_G\Big)\nonumber\\
    &\!\!\!\!\le\!\!\!\!&2\sqrt{\mu}\,c_\eta\limsup_{i\to\infty}\|\wt\mDelta_i\|_{L_2}+2\mu\Big(c_\omega\limsup_{i\to\infty}\omega_i+c_G\Big),
\end{eqnarray}
which is \eqref{long term accuracy bound}.
\end{proof}

\subsection[Steady state of the SGD iterate]{Steady state of the SGD iterate}
\label{Proof of steady state true}

\begin{proposition}
\label{prop:steady true}
Under assumptions (i)--(iv), there exists $C>0$ independent of $\mu$ such that, for all $\mu$ satisfying \eqref{mu0 definition},
\begin{eqnarray}
    \label{steady state true main}
    \limsup_{i\to\infty}\bigg|\E[\|\calX_i-\calX^\star\|_F^2]-\frac{\mu d_1d_2d_3\sigma_e^2}{2-\mu\sigma_a^2}\,\rho_{d_4}\bigg|\le C\mu^{3/2}.
\end{eqnarray}
\end{proposition}

\begin{proof}
With $\mD_i=\mDelta_i-\wt\mDelta_i$ and $\|\cdot\|_{L_2}$ as in \Cref{lemma:long term accuracy}, since $\mDelta_i=\wt\mDelta_i+\mD_i$, expanding $\|\mDelta_i\|_F^2$ and using the Cauchy--Schwarz inequality, we have
\begin{eqnarray}
    \label{comparison true and long term}
    \big|\E[\|\mDelta_i\|_F^2]-\E[\|\wt\mDelta_i\|_F^2]\big|\le2\|\wt\mDelta_i\|_{L_2}\|\mD_i\|_{L_2}+\|\mD_i\|_{L_2}^2 .
\end{eqnarray}
It suffices to bound $\|\wt\mDelta_i\|_{L_2}$ and $\|\mD_i\|_{L_2}$.

\paragraph[Step 1: size of $\wt\mDelta_i$.]{Step 1: size of $\wt\mDelta_i$.}
 Since $\rho_{d_4}\le1$ and $2-\mu\sigma_a^2>1$, by \eqref{steady state long term main}, we have
\begin{eqnarray}
    \label{size long term static}
    \lim_{i\to\infty}\|\wt\mDelta_i\|_{L_2}^2=\frac{\mu d_1d_2d_3\sigma_e^2}{2-\mu\sigma_a^2}\rho_{d_4}\le\mu d_1d_2d_3\sigma_e^2 .
\end{eqnarray}
Moreover, $\|\wt\mDelta_t\|_{L_2}<\infty$ for each $t$ by \eqref{unrolled sliding}, so $\sup_t\|\wt\mDelta_t\|_{L_2}<\infty$.

\paragraph[Step 2: size of $\mD_i$.]{Step 2: size of $\mD_i$.}
 We apply \Cref{lemma:long term accuracy}. Under (iii), $\mX_t^\star=\mX^\star$ for all $t$, so $\omega_t=0$ for all $t$. Moreover, $\|\wt\mDelta_t\|_{L_2}=\|\mX_0-\mX^\star\|_F$ for $t<0$, and $\sup_{t\ge0}\|\wt\mDelta_t\|_{L_2}<\infty$ by Step~1, so the conditions of \Cref{lemma:long term accuracy} hold. Since $\limsup_i\omega_i=0$, \eqref{long term accuracy bound} gives
\begin{eqnarray}
    \label{static lemma applied}
    \limsup_{i\to\infty}\|\mD_i\|_{L_2}\le2\sqrt{\mu}\,c_\eta\limsup_{i\to\infty}\|\wt\mDelta_i\|_{L_2}+2\mu c_G .
\end{eqnarray}
By \eqref{size long term static}, $\limsup_i\|\wt\mDelta_i\|_{L_2}=\big(\lim_i\|\wt\mDelta_i\|_{L_2}^2\big)^{1/2}\le\sqrt{\mu}\,\sigma_e\sqrt{d_1d_2d_3}$. Substituting this into \eqref{static lemma applied}, we obtain
\begin{eqnarray}
    \label{c1 static}
    \limsup_{i\to\infty}\|\mD_i\|_{L_2}\le2\sqrt{\mu}\,c_\eta\cdot\sqrt{\mu}\,\sigma_e\sqrt{d_1d_2d_3}+2\mu c_G=2\mu\big(c_\eta\sigma_e\sqrt{d_1d_2d_3}+c_G\big)=c_1\mu,
\end{eqnarray}
where $c_1=2\big(c_\eta\sigma_e\sqrt{d_1d_2d_3}+c_G\big)$.

\paragraph[Step 3: conclusion.]{Step 3: conclusion.}
 Taking $\limsup_{i\to\infty}$ in \eqref{comparison true and long term} and using \eqref{size long term static} and \eqref{c1 static}, we obtain
\begin{eqnarray}
    \label{comparison final}
    \limsup_{i\to\infty}\big|\E[\|\mDelta_i\|_F^2]-\E[\|\wt\mDelta_i\|_F^2]\big|\le2c_1\sqrt{d_1d_2d_3}\,\sigma_e\mu^{3/2}+c_1^2\mu^2\le C\mu^{3/2},
\end{eqnarray}
with $C=2c_1\sqrt{d_1d_2d_3}\,\sigma_e+c_1^2\sqrt{\mu_0}$. Since $\E[\|\wt\mDelta_i\|_F^2]$ converges to $\frac{\mu d_1d_2d_3\sigma_e^2}{2-\mu\sigma_a^2}\rho_{d_4}$ and $\|\calX_i-\calX^\star\|_F=\|\mDelta_i\|_F$, this proves \eqref{steady state true main}.
\end{proof}

\subsection{Proof of \Cref{The main theorem of error in the model II}}
\label{Proof of corollary 1}

\begin{proof}
Let $\nu=\mu\sigma_a^2\in(0,1)$ and $\beta=1-\nu$. Since
\begin{eqnarray}
    \frac{\mu d_1d_2d_3\sigma_e^2}{2-\mu\sigma_a^2}\rho_{d_4}-\frac{\mu d_1d_2d_3\sigma_e^2}{2}=\mu d_1d_2d_3\sigma_e^2\Big(\frac{\rho_{d_4}}{2-\nu}-\frac12\Big),
\end{eqnarray}
we first bound $\frac{\rho_{d_4}}{2-\nu}-\frac12$. Since $0\le1-\beta^m\le m\nu$ for $m\ge0$, we have
\begin{eqnarray}
    \label{rho bound}
    0\le1-\rho_{d_4}=\frac{1}{d_4^2}\sum_{k=0}^{d_4-1}\sum_{k'=0}^{d_4-1}\big(1-\beta^{|k-k'|}\big)\le\frac{\nu}{d_4^2}\sum_{k,k'}|k-k'|=\nu\,\frac{d_4^2-1}{3d_4},
\end{eqnarray}
where the last equation uses $\sum_{k,k'}|k-k'|=2\sum_{m=1}^{d_4-1}m(d_4-m)=\frac{d_4(d_4^2-1)}{3}$. Hence, using $2(2-\nu)>2$ and $1+\frac{2(d_4^2-1)}{3d_4}=\frac{(2d_4-1)(d_4+2)}{3d_4}$, we obtain
\begin{eqnarray}
    \label{fraction bound}
    \bigg|\frac{\rho_{d_4}}{2-\nu}-\frac12\bigg|=\frac{|\nu-2(1-\rho_{d_4})|}{2(2-\nu)}\le\frac{\nu+2(1-\rho_{d_4})}{2}\le\frac{(2d_4-1)(d_4+2)}{6d_4}\,\nu .
\end{eqnarray}
Multiplying \eqref{fraction bound} by $\mu d_1d_2d_3\sigma_e^2$ and using $\nu=\mu\sigma_a^2$, we have
\begin{eqnarray}
    \label{long term first order}
    \bigg|\frac{\mu d_1d_2d_3\sigma_e^2}{2-\mu\sigma_a^2}\rho_{d_4}-\frac{\mu d_1d_2d_3\sigma_e^2}{2}\bigg|\le c_4\mu^2,\qquad c_4=\frac{(2d_4-1)(d_4+2)}{6d_4}\,d_1d_2d_3\sigma_a^2\sigma_e^2 .
\end{eqnarray}
Finally, by the triangle inequality, \eqref{steady state true main} and \eqref{long term first order}, the left-hand side of \eqref{steady state performance in the ToM regression model main theorem} is at most $C\mu^{3/2}+c_4\mu^2\le(C+c_4\sqrt{\mu_0})\mu^{3/2}$, where we use $\mu^2=\mu^{3/2}\mu^{1/2}\le\mu^{3/2}\mu_0^{1/2}$. This proves \Cref{The main theorem of error in the model II} with $C'=C+c_4\sqrt{\mu_0}$.
\end{proof}

\section{Proof of \Cref{The main theorem of tracking ability in the model II}}
\label{Proof of tracking ability of SGD of model II}

\paragraph[Setting and notation.]{Setting and notation.}
 Throughout this appendix, the regression tensor $\calX_i^\star$ follows the AR(1) model \eqref{RW model definition}, with the initial tensors $\calX_t^\star$, $t=-d_4+1,\dots,0$, as in \Cref{sec:tracking ability performance}. As in \Cref{Proof of error analysis SGD of model II}, the averaged iterate $\wt\calX_i$ is obtained from the SGD update \eqref{updating method X another model II} by replacing $\va_{i-k}\va_{i-k}^\top$ with its mean $\sigma_a^2\mId_{d_3\times d_3}$:
\begin{eqnarray}
    \label{tracking averaged iterate}
    \wt\calX_0=\calX_0,\qquad
    \wt\calX_{i+1}=\wt\calX_i-\frac{\mu\sigma_a^2}{d_4}\sum_{k=0}^{d_4-1}\big(\wt\calX_i-\calX^\star_{i-k}\big)+\frac{\mu}{d_4}\calE_i\times_3^2\mA_i .
\end{eqnarray}
We use the reshaped quantities
\begin{eqnarray}
    \label{tracking G definition}
    \mX_i&\!\!\!\!=\!\!\!\!&\text{reshape}(\calX_i,[d_1d_2,d_3]),\quad \wt\mX_i=\text{reshape}(\wt\calX_i,[d_1d_2,d_3]),\quad \mX_i^\star=\text{reshape}(\calX_i^\star,[d_1d_2,d_3]),\nonumber\\
    \ve_i&\!\!\!\!=\!\!\!\!&\text{reshape}(\calE_i(:,:,1),[d_1d_2,1]),\quad \mG_i=\frac{1}{d_4}\sum_{k=0}^{d_4-1}\ve_{i-k}\va_{i-k}^\top ,
\end{eqnarray}
where $\mG_i$ is the reshaped $\frac{1}{d_4}\calE_i\times_3^2\mA_i$. Then \eqref{tracking averaged iterate} reads
\begin{eqnarray}
    \label{tracking averaged iterate matrix}
    \wt\mX_{i+1}=\wt\mX_i-\frac{\mu\sigma_a^2}{d_4}\sum_{k=0}^{d_4-1}\big(\wt\mX_i-\mX^\star_{i-k}\big)+\mu\mG_i .
\end{eqnarray}
We also use $\rho_{d_4}=\frac{1}{d_4^2}\sum_{k=0}^{d_4-1}\sum_{k'=0}^{d_4-1}(1-\mu\sigma_a^2)^{|k-k'|}$ from \Cref{prop:steady long term}, and $\|\mM\|_{L_2}=(\E[\|\mM\|_F^2])^{1/2}$ for a random matrix $\mM$. Finally, throughout this appendix, we write
\begin{eqnarray}
    \label{tracking constants}
    \nu=\mu\sigma_a^2,\qquad \beta=1-\nu,\qquad \theta=1-\lambda,\qquad s=\frac{\lambda\sigma_q^2}{2-\lambda},
\end{eqnarray}
where $\beta\in(0,1)$ for $\mu<\frac{1}{\sigma_a^2}$, and $\theta\in(0,1)$ since $\lambda\in(0,1)$.

\subsection[Tracking error of the averaged recursion]{Tracking error of the averaged recursion}
\label{Proof of tracking long term}

\begin{proposition}
\label{prop:tracking long term}
Let $\wt\calX_i$ be the averaged iterate in \eqref{tracking averaged iterate}, $\rho_{d_4}$ be as above, and $\beta$ and $\theta$ be as in \eqref{tracking constants}. Define the coefficients $c_{-1}=-1$, $c_0=\beta+\frac{\mu\sigma_a^2}{d_4}$, $c_r=\frac{\mu\sigma_a^2}{d_4}$ for $1\le r\le d_4-1$, the function
\begin{eqnarray}
    \label{Psi definition main}
    \Psi(u)=\frac{\theta^u+\beta^{u+1}\theta}{1-\beta\theta}+\sum_{w=1}^{u}\beta^w\theta^{u-w},\qquad u=0,1,2,\dots,
\end{eqnarray}
and
\begin{eqnarray}
    \label{Gamma definition main}
    \Gamma_{d_4}=\frac{1}{1-\beta^2}\sum_{r=-1}^{d_4-1}\sum_{r'=-1}^{d_4-1}c_rc_{r'}\Psi(|r-r'|).
\end{eqnarray}
Under assumptions (i)--(ii) and the AR(1) model \eqref{RW model definition}, provided that the step size obeys $\mu<\frac{1}{\sigma_a^2}$,
\begin{eqnarray}
    \label{tracking long term main}
    \lim_{i\to\infty}\E[\|\wt\calX_i-\calX_i^\star\|_F^2]=\frac{\mu d_1d_2d_3\sigma_e^2}{2-\mu\sigma_a^2}\,\rho_{d_4}+\frac{d_1d_2d_3\lambda\sigma_q^2}{2-\lambda}\,\Gamma_{d_4}.
\end{eqnarray}
\end{proposition}

\begin{proof}
Let $\mQ_i=\text{reshape}(\calQ_i,[d_1d_2,d_3])$ and $\wt\mDelta_i=\wt\mX_i-\mX_i^\star$, and let $\nu$, $\beta$, $\theta$ and $s$ be as in \eqref{tracking constants}.

\paragraph[Step 1: error expansion.]{Step 1: error expansion.}
 Subtracting $\mX^\star_{i+1}$ from both sides of \eqref{tracking averaged iterate matrix}, we obtain
\begin{eqnarray}
    \label{long term model tracking}
    \wt\mDelta_{i+1}=\beta\wt\mDelta_i+\mu\mG_i+\mU_i,\qquad
    \mU_i=-\mX^\star_{i+1}+\beta\mX_i^\star+\frac{\nu}{d_4}\sum_{k=0}^{d_4-1}\mX^\star_{i-k}=\sum_{r=-1}^{d_4-1}c_r\mX^\star_{i-r},
\end{eqnarray}
where the coefficients satisfy
\begin{eqnarray}
    \label{c properties}
    \sum_{r=-1}^{d_4-1}c_r=-1+\beta+\nu=0,\qquad \sum_{r=-1}^{d_4-1}|c_r|=1+\beta+\nu=2 .
\end{eqnarray}
Iterating \eqref{long term model tracking}, we obtain, for $i\ge1$,
\begin{eqnarray}
    \label{decomposition tracking}
    \wt\mDelta_i=\beta^i\wt\mDelta_0+\mu\sum_{j=0}^{i-1}\beta^j\mG_{i-1-j}+\sum_{j=0}^{i-1}\beta^j\mU_{i-1-j}.
\end{eqnarray}
Since $\mU_t$ involves only nonnegative indices if and only if $t\ge d_4-1$, i.e., $j\le i-d_4$ for $t=i-1-j$, we write, for $i\ge d_4$,
\begin{eqnarray}
    \label{error expansion}
    \wt\mDelta_i=\wt\mDelta_i^{e}+\mS_i+\mR_i,
\end{eqnarray}
where
\begin{eqnarray}
    \label{error expansion parts}
    \wt\mDelta_i^{e}=\mu\sum_{j=0}^{i-1}\beta^j\mG_{i-1-j},\qquad
    \mS_i=\sum_{j=0}^{i-d_4}\beta^j\mU_{i-1-j},\qquad
    \mR_i=\beta^i\wt\mDelta_0+\sum_{j=i-d_4+1}^{i-1}\beta^j\mU_{i-1-j}.
\end{eqnarray}
Taking the expectation of the squared norm, we have
\begin{eqnarray}
    \label{split tracking}
    \E[\|\wt\mDelta_i\|_F^2]
    =\underbrace{\E[\|\wt\mDelta_i^{e}\|_F^2]}_{\text{(I)}}+\underbrace{\E[\|\mS_i\|_F^2]}_{\text{(II)}}
    +\underbrace{\E[\|\mR_i\|_F^2]+2\E[\trace(\mS_i^\top\mR_i)]}_{\text{(III)}}
    +\underbrace{2\E\big[\trace\big((\wt\mDelta_i^{e})^\top(\mS_i+\mR_i)\big)\big]}_{\text{(IV)}}.
\end{eqnarray}
We show that (IV) is zero and compute the limit of (I) (Step~2), show that (III) tends to zero (Step~3), and compute the limit of (II) (Step~4); Step~5 combines them.

\paragraph[Step 2: the terms (IV) and (I).]{Step 2: the terms (IV) and (I).}
 The matrix $\wt\mDelta_i^{e}$ is a function of $\{\va_t,\ve_t\}$ with $\E[\wt\mDelta_i^{e}]={\bf 0}$, while $\mS_i+\mR_i$ is a function of $\{\mX_t^\star\}$, which is independent of $\{\va_t,\ve_t\}$. Hence
\begin{eqnarray}
    \label{cross noise drift}
    \text{(IV)}=2\,\trace\big(\E[\wt\mDelta_i^{e}]^\top\E[\mS_i+\mR_i]\big)=0 .
\end{eqnarray}
Moreover, $\wt\mDelta_i^{e}$ does not involve $\{\mX_t^\star\}$, and it equals the error of the averaged iterate in \Cref{prop:steady long term} when the initial error is zero. Hence, by \Cref{prop:steady long term}, we have
\begin{eqnarray}
    \label{noise part tracking}
    \lim_{i\to\infty}\text{(I)}=\frac{\mu d_1d_2d_3\sigma_e^2}{2-\mu\sigma_a^2}\,\rho_{d_4}.
\end{eqnarray}

\paragraph[Step 3: the term (III).]{Step 3: the term (III).}
 By the Cauchy--Schwarz inequality, we have
\begin{eqnarray}
    \label{III bound}
    |\text{(III)}|\le\|\mR_i\|_{L_2}^2+2\|\mS_i\|_{L_2}\|\mR_i\|_{L_2},
\end{eqnarray}
so it suffices to show that $\|\mR_i\|_{L_2}\to0$ and that $\|\mS_i\|_{L_2}$ is bounded. By the definitions of $\mR_i$ and $\mS_i$ in \eqref{error expansion parts} and the Minkowski inequality, we have
\begin{eqnarray}
    \label{R S expanded}
    \|\mR_i\|_{L_2}\le\beta^i\|\wt\mDelta_0\|_{L_2}+\sum_{j=i-d_4+1}^{i-1}\beta^j\|\mU_{i-1-j}\|_{L_2},\qquad
    \|\mS_i\|_{L_2}\le\sum_{j=0}^{i-d_4}\beta^j\|\mU_{i-1-j}\|_{L_2}.
\end{eqnarray}
The first term of $\|\mR_i\|_{L_2}$ tends to zero, since $0<\beta<1$ and $\wt\mDelta_0=\mX_0-\mX_0^\star$ has a finite second moment. For the remaining terms, we need a bound on $\|\mU_t\|_{L_2}$, $t\ge0$. Since $\mU_t=\sum_rc_r\mX^\star_{t-r}$ by \eqref{long term model tracking}, we first bound the second moments of $\mX_t^\star$. With $\theta=1-\lambda$, \eqref{AR solution} reads $\mX_b^\star=\theta^b\mX_0^\star+\lambda\sum_{n=0}^{b-1}\theta^n\mQ_{b-n}$ for $b\ge0$. Since $\mQ_1,\dots,\mQ_b$ are independent with zero mean and independent of $\mX_0^\star$, all cross terms vanish when the squared norm is expanded, and we have
\begin{eqnarray}
    \label{AR variance}
    \E[\|\mX_b^\star\|_F^2]=\theta^{2b}\E[\|\mX_0^\star\|_F^2]+\lambda^2d_1d_2d_3\sigma_q^2\sum_{n=0}^{b-1}\theta^{2n}
    =\theta^{2b}\E[\|\mX_0^\star\|_F^2]+d_1d_2d_3\,s\,(1-\theta^{2b}),
\end{eqnarray}
where the last equation uses $\sum_{n=0}^{b-1}\theta^{2n}=\frac{1-\theta^{2b}}{1-\theta^2}$, $1-\theta^2=\lambda(2-\lambda)$ and $\frac{\lambda^2\sigma_q^2}{\lambda(2-\lambda)}=s$. In particular, since $0<\theta<1$, $\E[\|\mX_b^\star\|_F^2]\to d_1d_2d_3s=\frac{d_1d_2d_3\lambda\sigma_q^2}{2-\lambda}$ as $b\to\infty$. Moreover, the right-hand side of \eqref{AR variance} is a convex combination of $\E[\|\mX_0^\star\|_F^2]$ and $d_1d_2d_3s$, and the finitely many initial tensors $\mX_t^\star$, $t=-d_4+1,\dots,-1$, have finite second moments, so
\begin{eqnarray}
    \label{kappa star}
    \kappa_\star=\sup_{t\ge-d_4+1}\E[\|\mX_t^\star\|_F^2]<\infty .
\end{eqnarray}
For $t\ge0$, all indices $t-r$, $r=-1,\dots,d_4-1$, in $\mU_t$ are at least $-d_4+1$. Hence, by the Minkowski inequality, \eqref{kappa star} and \eqref{c properties}, we have
\begin{eqnarray}
    \label{U bound}
    \|\mU_t\|_{L_2}\le\sum_{r=-1}^{d_4-1}|c_r|\,\|\mX^\star_{t-r}\|_{L_2}\le\sqrt{\kappa_\star}\sum_{r=-1}^{d_4-1}|c_r|=2\sqrt{\kappa_\star},\qquad t\ge0 .
\end{eqnarray}
Substituting \eqref{U bound} into \eqref{R S expanded}, and using $\sum_{j\ge i-d_4+1}\beta^j=\frac{\beta^{i-d_4+1}}{1-\beta}$ and $\sum_{j\ge0}\beta^j=\frac{1}{1-\beta}$, we obtain
\begin{eqnarray}
    \label{R S bounds}
    \|\mR_i\|_{L_2}\le\beta^i\|\wt\mDelta_0\|_{L_2}+\frac{2\sqrt{\kappa_\star}\,\beta^{i-d_4+1}}{1-\beta}\ \to\ 0,\qquad
    \|\mS_i\|_{L_2}\le\frac{2\sqrt{\kappa_\star}}{1-\beta}.
\end{eqnarray}
Substituting \eqref{R S bounds} into \eqref{III bound}, we conclude that $\text{(III)}\to0$ as $i\to\infty$.

\paragraph[Step 4: the drift part (II).]{Step 4: the drift part (II).}
 By the definition of $\mS_i$ in \eqref{error expansion parts}, expanding the squared norm, we have
\begin{eqnarray}
    \label{II expanded}
    \text{(II)}=\E[\|\mS_i\|_F^2]=\sum_{j=0}^{i-d_4}\sum_{l=0}^{i-d_4}\beta^{j+l}\,\E\big[\trace\big(\mU_{i-1-j}^\top\mU_{i-1-l}\big)\big],
\end{eqnarray}
where $i-1-j,i-1-l\ge d_4-1$. Hence we need $\E[\trace(\mU_t^\top\mU_{t'})]$ for $t,t'\ge d_4-1$. Since $\mU_t=\sum_rc_r\mX^\star_{t-r}$ by \eqref{long term model tracking}, we have
\begin{eqnarray}
    \label{U cov expanded}
    \E[\trace(\mU_t^\top\mU_{t'})]=\sum_{r,r'}c_rc_{r'}\,\E\big[\trace\big((\mX^\star_{t-r})^\top\mX^\star_{t'-r'}\big)\big],
\end{eqnarray}
where $\sum_{r,r'}$ runs over $r,r'\in\{-1,\dots,d_4-1\}$, and all indices $t-r,t'-r'$ are nonnegative. We therefore compute $\E[\trace((\mX_a^\star)^\top\mX_b^\star)]$ for $a,b\ge0$. For $a\ge b\ge0$, applying \eqref{RW model definition} repeatedly gives $\mX_a^\star=\theta^{a-b}\mX_b^\star+\lambda\sum_{n=0}^{a-b-1}\theta^n\mQ_{a-n}$, where each $\mQ_{a-n}$ has zero mean and is independent of $\mX_b^\star$. Hence $\E[\trace((\mX_a^\star)^\top\mX_b^\star)]=\theta^{a-b}\E[\|\mX_b^\star\|_F^2]$, and, by \eqref{AR variance} and symmetry in $a,b$,
\begin{eqnarray}
    \label{AR autocovariance}
    \E[\trace((\mX_a^\star)^\top\mX_b^\star)]=d_1d_2d_3\,s\,\theta^{|a-b|}+\big(\E[\|\mX_0^\star\|_F^2]-d_1d_2d_3s\big)\theta^{a}\theta^{b},\qquad a,b\ge0 .
\end{eqnarray}
Substituting \eqref{AR autocovariance} with $a=t-r$ and $b=t'-r'$ into \eqref{U cov expanded}, we obtain
\begin{eqnarray}
    \label{U autocovariance}
    \E[\trace(\mU_t^\top\mU_{t'})]=d_1d_2d_3\,s\,K(t'-t)+\big(\E[\|\mX_0^\star\|_F^2]-d_1d_2d_3s\big)h_th_{t'},
\end{eqnarray}
where
\begin{eqnarray}
    \label{K h definition}
    K(m)=\sum_{r,r'}c_rc_{r'}\theta^{|m+r-r'|},\qquad
    h_t=\sum_{r=-1}^{d_4-1}c_r\theta^{t-r}.
\end{eqnarray}
Since $\theta^{|m+r-r'|}\le1$ and $\theta^{t-r}\le\theta^{t-d_4+1}$, \eqref{c properties} gives $|K(m)|\le4$ and $|h_t|\le2\theta^{t-d_4+1}$. Substituting \eqref{U autocovariance} with $t=i-1-j$ and $t'=i-1-l$ into \eqref{II expanded}, we obtain
\begin{eqnarray}
    \label{S second moment}
    \text{(II)}=d_1d_2d_3\,s\sum_{j=0}^{i-d_4}\sum_{l=0}^{i-d_4}\beta^{j+l}K(j-l)+\big(\E[\|\mX_0^\star\|_F^2]-d_1d_2d_3s\big)\Big(\sum_{j=0}^{i-d_4}\beta^jh_{i-1-j}\Big)^2.
\end{eqnarray}
We now take $i\to\infty$ in the two terms of \eqref{S second moment}. For the second term, $|\sum_{j}\beta^jh_{i-1-j}|\le2\theta^{-(d_4-1)}\sum_j\beta^j\theta^{i-1-j}\le2\theta^{-(d_4-1)}\,i\max\{\beta,\theta\}^{i-1}\to0$. For the first term, the summands satisfy $|\beta^{j+l}K(j-l)|\le4\beta^{j+l}$ and $\sum_{j,l\ge0}\beta^{j+l}<\infty$, so the sum over $j,l$ converges absolutely; inserting the definition of $K$ and exchanging the order of summation, we obtain
\begin{eqnarray}
    \label{drift limit Phi}
    \lim_{i\to\infty}\text{(II)}=d_1d_2d_3\,s\sum_{r,r'}c_rc_{r'}\Phi(r-r'),\qquad
    \Phi(u)=\sum_{j\ge0}\sum_{l\ge0}\beta^{j+l}\theta^{|j-l+u|}.
\end{eqnarray}
Finally, we show that $\Phi(u)=\Psi(|u|)/(1-\beta^2)$ for every integer $u$. Exchanging the roles of $j$ and $l$ in \eqref{drift limit Phi} gives $\Phi(-u)=\Phi(u)$, so it suffices to consider $u\ge0$. We group the terms in \eqref{drift limit Phi} by $m=j-l\in\mathbb{Z}$. For a fixed $m$, the pairs $(j,l)$ with $j,l\ge0$ and $j-l=m$ are $(n+|m|,n)$ if $m\ge0$ and $(n,n+|m|)$ if $m<0$, with $n=0,1,2,\dots$; in both cases $j+l=|m|+2n$, and hence
\begin{eqnarray}
    \label{Phi pair sum}
    \sum_{j,l\ge0,\,j-l=m}\beta^{j+l}=\beta^{|m|}\sum_{n\ge0}\beta^{2n}=\frac{\beta^{|m|}}{1-\beta^2}.
\end{eqnarray}
Since $\theta^{|j-l+u|}=\theta^{|m+u|}$ depends on $(j,l)$ only through $m$, \eqref{Phi pair sum} gives
\begin{eqnarray}
    \label{Phi single sum}
    \Phi(u)=\frac{1}{1-\beta^2}\sum_{m\in\mathbb{Z}}\beta^{|m|}\theta^{|m+u|}.
\end{eqnarray}
We split the sum in \eqref{Phi single sum} according to the signs of $m$ and $m+u$. For $m\ge0$, we have $|m|=m$ and $|m+u|=m+u$, so
\begin{eqnarray}
    \label{Phi part 1}
    \sum_{m\ge0}\beta^{m}\theta^{m+u}=\theta^u\sum_{m\ge0}(\beta\theta)^m=\frac{\theta^u}{1-\beta\theta}.
\end{eqnarray}
For $-u\le m\le-1$, we have $|m|=-m$ and $|m+u|=m+u$; substituting $w=-m$, we obtain
\begin{eqnarray}
    \label{Phi part 2}
    \sum_{m=-u}^{-1}\beta^{-m}\theta^{m+u}=\sum_{w=1}^{u}\beta^w\theta^{u-w}.
\end{eqnarray}
For $m\le-u-1$, we have $|m|=-m$ and $|m+u|=-m-u$; substituting $k=-m-u\ge1$, we obtain
\begin{eqnarray}
    \label{Phi part 3}
    \sum_{m\le-u-1}\beta^{-m}\theta^{-m-u}=\beta^u\sum_{k\ge1}(\beta\theta)^k=\frac{\beta^{u+1}\theta}{1-\beta\theta}.
\end{eqnarray}
Substituting \eqref{Phi part 1}--\eqref{Phi part 3} into \eqref{Phi single sum} and comparing with the definition of $\Psi$ in \eqref{Psi definition main}, we obtain $\Phi(u)=\Psi(u)/(1-\beta^2)$ for $u\ge0$. Together with $\Phi(-u)=\Phi(u)$, this gives
\begin{eqnarray}
    \label{Phi closed form}
    \Phi(u)=\frac{\Psi(|u|)}{1-\beta^2},\qquad u\in\mathbb{Z}.
\end{eqnarray}

\paragraph[Step 5: combination.]{Step 5: combination.}
 Taking $i\to\infty$ in \eqref{split tracking} and using \eqref{cross noise drift} and \eqref{noise part tracking} for (IV) and (I), \eqref{III bound} and \eqref{R S bounds} for (III), and \eqref{drift limit Phi}, \eqref{Phi closed form} and \eqref{Gamma definition main} for (II), we obtain
\begin{eqnarray}
    \lim_{i\to\infty}\E[\|\wt\mDelta_i\|_F^2]
    =\frac{\mu d_1d_2d_3\sigma_e^2}{2-\mu\sigma_a^2}\,\rho_{d_4}+\frac{d_1d_2d_3\lambda\sigma_q^2}{2-\lambda}\,\Gamma_{d_4},
\end{eqnarray}
which is \eqref{tracking long term main}.
\end{proof}

\subsection[A property of $\Gamma_{d_4}$]{A property of $\Gamma_{d_4}$}
\label{Proof of Gamma lemma}

\begin{lemma}
\label{lemma:Gamma}
Let $\beta$, $\theta$, $c_r$, $\Psi$ and $\Gamma_{d_4}$ be as in \Cref{prop:tracking long term}, with $\nu=\mu\sigma_a^2\in(0,1)$ and $\lambda\in(0,1)$, and define
\begin{eqnarray}
    \label{g G definition}
    g(v)=\frac{\theta^v-\beta^{v+1}}{1-\beta\theta}+\sum_{w=1}^{v}\beta^w\theta^{v-w},\quad v\ge0,\qquad G(u)=\sum_{v=0}^{u-1}g(v),\quad u\ge0,
\end{eqnarray}
with $G(0)=0$. For vectors $\vx=\begin{bmatrix}x_{-1}&\cdots&x_{d_4-1}\end{bmatrix}^\top$ and $\vy=\begin{bmatrix}y_{-1}&\cdots&y_{d_4-1}\end{bmatrix}^\top$, define
\begin{eqnarray}
    \label{B definition}
    B(\vx,\vy)=\sum_{r=-1}^{d_4-1}\sum_{r'=-1}^{d_4-1}x_ry_{r'}\Psi(|r-r'|).
\end{eqnarray}
Then:
\begin{enumerate}
\item[(a)] $\Psi(v)-\Psi(v+1)=\lambda g(v)$ for all $v\ge0$, and, if $\sum_rx_r=0$,
\begin{eqnarray}
    \label{bilinear G form}
    B(\vx,\vy)=-\lambda\sum_{r,r'}x_ry_{r'}G(|r-r'|);
\end{eqnarray}
in particular, $\Gamma_{d_4}=\frac{B(\vc,\vc)}{1-\beta^2}=-\frac{\lambda}{1-\beta^2}\sum_{r,r'}c_rc_{r'}G(|r-r'|)$ with $\vc=\begin{bmatrix}c_{-1}&\cdots&c_{d_4-1}\end{bmatrix}^\top$;
\item[(b)] $|g(v)|\le2v+1$ and $|G(u)|\le u^2$; hence, if $\sum_rx_r=0$, $|B(\vx,\vy)|\le\lambda d_4^2\big(\sum_r|x_r|\big)\big(\sum_r|y_r|\big)$, and $\displaystyle|\Gamma_{d_4}|\le\frac{4d_4^2\lambda}{\nu(2-\nu)}$.
\end{enumerate}
\end{lemma}

\begin{proof}
\emph{Proof of (a).} Separating the term $w=v+1$ in \eqref{Psi definition main} with $u=v+1$, we have
\begin{eqnarray}
    \label{Psi v+1}
    \Psi(v+1)=\frac{\theta^{v+1}+\beta^{v+2}\theta}{1-\beta\theta}+\theta\sum_{w=1}^{v}\beta^w\theta^{v-w}+\beta^{v+1}.
\end{eqnarray}
Subtracting \eqref{Psi v+1} from \eqref{Psi definition main} with $u=v$, and using $1-\theta=\lambda$, $1-\beta=\nu$ and $\nu\theta-(1-\beta\theta)=\theta-1=-\lambda$, we obtain
\begin{eqnarray}
    \label{Psi increment}
    \Psi(v)-\Psi(v+1)=\frac{\lambda\theta^v+\nu\theta\beta^{v+1}-(1-\beta\theta)\beta^{v+1}}{1-\beta\theta}+\lambda\sum_{w=1}^{v}\beta^w\theta^{v-w}=\lambda\,g(v).
\end{eqnarray}
Hence, by telescoping, $\Psi(u)-\Psi(0)=-\lambda G(u)$ for $u\ge0$. Since $\sum_rx_r=0$, $\sum_{r,r'}x_ry_{r'}\Psi(0)=0$, and therefore
\begin{eqnarray}
    B(\vx,\vy)=\sum_{r,r'}x_ry_{r'}\big(\Psi(|r-r'|)-\Psi(0)\big)=-\lambda\sum_{r,r'}x_ry_{r'}G(|r-r'|).
\end{eqnarray}
By \eqref{c properties}, this applies to $\vx=\vy=\vc$, and $\Gamma_{d_4}=\frac{B(\vc,\vc)}{1-\beta^2}$ by \eqref{Gamma definition main}.

\emph{Proof of (b).} We bound $g$, $G$, $B$ and $\Gamma_{d_4}$ in turn. By \eqref{g G definition} and the triangle inequality, we have
\begin{eqnarray}
    \label{g split}
    |g(v)|\le\frac{|\theta^v-\beta^{v+1}|}{1-\beta\theta}+\sum_{w=1}^{v}\beta^w\theta^{v-w}.
\end{eqnarray}
Since $\beta,\theta\in(0,1)$, each of the $v$ summands of the second term lies in $(0,1)$, so the second term is at most $v$. For the first term, we write $\theta^v-\beta^{v+1}=(\theta^v-\beta^v)+\beta^v(1-\beta)$. Since $\theta^v-\beta^v=(\theta-\beta)\sum_{k=0}^{v-1}\theta^k\beta^{v-1-k}$ and each of the $v$ summands lies in $(0,1]$, we have $|\theta^v-\beta^v|\le v|\theta-\beta|$. Moreover, $|\theta-\beta|=|\nu-\lambda|\le\max\{\lambda,\nu\}$ and $\beta^v(1-\beta)\le1-\beta=\nu$. Hence, the numerator satisfies
\begin{eqnarray}
    \label{g numerator}
    |\theta^v-\beta^{v+1}|\le v|\theta-\beta|+\beta^v(1-\beta)\le(v+1)\max\{\lambda,\nu\}.
\end{eqnarray}
Using $\theta=1-\lambda$ and $\beta=1-\nu$, the denominator satisfies
\begin{eqnarray}
    \label{denominator max}
    1-\beta\theta=\lambda+\nu(1-\lambda)=\nu+\lambda(1-\nu)\ge\max\{\lambda,\nu\}.
\end{eqnarray}
Substituting \eqref{g numerator} and \eqref{denominator max} into \eqref{g split}, we obtain
\begin{eqnarray}
    \label{g bound}
    |g(v)|\le(v+1)+v=2v+1,\qquad v\ge0.
\end{eqnarray}
By the definition of $G$ in \eqref{g G definition} and \eqref{g bound}, we have
\begin{eqnarray}
    \label{G bound}
    |G(u)|\le\sum_{v=0}^{u-1}(2v+1)=u^2,\qquad u\ge0.
\end{eqnarray}
Now let $\sum_rx_r=0$. Since $r,r'\in\{-1,\dots,d_4-1\}$, we have $|r-r'|\le d_4$, and hence $|G(|r-r'|)|\le d_4^2$ by \eqref{G bound}. Applying the triangle inequality to \eqref{bilinear G form}, we obtain
\begin{eqnarray}
    \label{B bound}
    |B(\vx,\vy)|\le\lambda\sum_{r,r'}|x_r||y_{r'}|\,|G(|r-r'|)|\le\lambda d_4^2\Big(\sum_r|x_r|\Big)\Big(\sum_{r}|y_{r}|\Big).
\end{eqnarray}
Finally, by (a), $\Gamma_{d_4}=B(\vc,\vc)/(1-\beta^2)$, where $\sum_rc_r=0$ by \eqref{c properties}. Applying \eqref{B bound} with $\vx=\vy=\vc$, and using $\sum_r|c_r|=2$ from \eqref{c properties} and $1-\beta^2=(1-\beta)(1+\beta)=\nu(2-\nu)$, we obtain
\begin{eqnarray}
    \label{Gamma bound}
    |\Gamma_{d_4}|\le\frac{\lambda d_4^2\cdot2^2}{\nu(2-\nu)}=\frac{4d_4^2\lambda}{\nu(2-\nu)}.
\end{eqnarray}
\end{proof}

\subsection[Tracking error of the SGD iterate]{Tracking error of the SGD iterate}
\label{Proof of tracking true}

\begin{proposition}
\label{prop:tracking true}
Under assumptions (i), (ii) and (iv) and the AR(1) model \eqref{RW model definition}, there exists $C>0$ independent of $\mu$, $\lambda$ and $\sigma_q$ such that, for all $\mu$ satisfying \eqref{mu0 definition},
\begin{eqnarray}
    \label{tracking true main}
    \limsup_{i\to\infty}\Big|\E[\|\calX_i-\calX_i^\star\|_F^2]-\lim_{i\to\infty}\E[\|\wt\calX_i-\calX_i^\star\|_F^2]\Big|\le C\sqrt{\mu}\Big(\mu+\frac{\lambda^2\sigma_q^2}{\mu}\Big).
\end{eqnarray}
\end{proposition}

\begin{proof}
Let $\nu$ and $s$ be as in \eqref{tracking constants}, $\tau=\mu+\lambda^2\sigma_q^2/\mu$ and $\mD_i=\mDelta_i-\wt\mDelta_i$, where $\mDelta_i=\mX_i-\mX_i^\star$ and $\wt\mDelta_i=\wt\mX_i-\mX_i^\star$. Since $\mDelta_i=\wt\mDelta_i+\mD_i$, expanding $\|\mDelta_i\|_F^2$ and using the Cauchy--Schwarz inequality, we have
\begin{eqnarray}
    \label{comparison tracking pre}
    \big|\E[\|\mDelta_i\|_F^2]-\E[\|\wt\mDelta_i\|_F^2]\big|\le2\|\wt\mDelta_i\|_{L_2}\|\mD_i\|_{L_2}+\|\mD_i\|_{L_2}^2 ,
\end{eqnarray}
so it suffices to bound $\|\wt\mDelta_i\|_{L_2}$ and $\|\mD_i\|_{L_2}$.

\paragraph[Step 1: size of $\wt\mDelta_i$.]{Step 1: size of $\wt\mDelta_i$.}
 By \Cref{lemma:Gamma}(b), $\Gamma_{d_4}\le\frac{4d_4^2\lambda}{\nu}$. Hence, by \eqref{tracking long term main}, $\rho_{d_4}\le1$, $2-\nu>1$ and $2-\lambda>1$, we have
\begin{eqnarray}
    \label{size long term tracking}
    \lim_{i\to\infty}\|\wt\mDelta_i\|_{L_2}^2\le d_1d_2d_3\,\sigma_e^2\mu+d_1d_2d_3\,\lambda\sigma_q^2\cdot\frac{4d_4^2\lambda}{\mu\sigma_a^2}\le c_3\tau,\qquad c_3=d_1d_2d_3\Big(\sigma_e^2+\frac{4d_4^2}{\sigma_a^2}\Big).
\end{eqnarray}
Moreover, by \eqref{tracking averaged iterate matrix} and \eqref{kappa star}, $\|\wt\mDelta_t\|_{L_2}<\infty$ for each $t$, so $\sup_t\|\wt\mDelta_t\|_{L_2}<\infty$.

\paragraph[Step 2: size of $\mD_i$.]{Step 2: size of $\mD_i$.}
 Let $\omega_t=\|\mX^\star_{t+1}-\mX^\star_t\|_{L_2}$ as in \Cref{lemma:long term accuracy}. By \eqref{kappa star}, $\sup_t\omega_t\le2\sqrt{\kappa_\star}$ and $\|\wt\mDelta_t\|_{L_2}\le\|\mX_0\|_F+\sqrt{\kappa_\star}$ for $t<0$. For $t\ge0$, $\mX_{t+1}^\star-\mX_t^\star=\lambda(\mQ_{t+1}-\mX_t^\star)$, where $\mQ_{t+1}$ has zero mean and is independent of $\mX_t^\star$, so $\omega_t^2=\lambda^2(d_1d_2d_3\sigma_q^2+\E[\|\mX_t^\star\|_F^2])$. Since $\E[\|\mX_t^\star\|_F^2]\to d_1d_2d_3s\le d_1d_2d_3\sigma_q^2$ by \eqref{AR variance}, and $\lambda^2\sigma_q^2\le\mu\tau$, we have
\begin{eqnarray}
    \label{omega AR}
    \limsup_{t\to\infty}\omega_t=\lambda\sqrt{d_1d_2d_3(\sigma_q^2+s)}\le\sqrt{2d_1d_2d_3}\,\lambda\sigma_q\le\sqrt{2d_1d_2d_3}\,\sqrt{\mu\tau}.
\end{eqnarray}
Hence the conditions of \Cref{lemma:long term accuracy} hold. Substituting $\limsup_i\|\wt\mDelta_i\|_{L_2}\le\sqrt{c_3\tau}$ from \eqref{size long term tracking} and $\limsup_i\omega_i\le\sqrt{2d_1d_2d_3}\sqrt{\mu\tau}$ from \eqref{omega AR} into \eqref{long term accuracy bound}, we have
\begin{eqnarray}
    \label{tracking lemma applied}
    \limsup_{i\to\infty}\|\mD_i\|_{L_2}\le2\sqrt{\mu}\,c_\eta\sqrt{c_3\tau}+2\mu c_\omega\sqrt{2d_1d_2d_3}\,\sqrt{\mu\tau}+2\mu c_G .
\end{eqnarray}
We now write each of the three terms as a multiple of $\sqrt{\mu\tau}$. The first term equals $2c_\eta\sqrt{c_3}\sqrt{\mu\tau}$. For the second term, $\mu\le\mu_0$ gives $2\mu c_\omega\sqrt{2d_1d_2d_3}\sqrt{\mu\tau}\le2\sqrt{2d_1d_2d_3}\,c_\omega\mu_0\sqrt{\mu\tau}$. For the third term, $\mu\le\tau$ gives $\mu=\sqrt{\mu}\sqrt{\mu}\le\sqrt{\mu\tau}$, so $2\mu c_G\le2c_G\sqrt{\mu\tau}$. Adding the three bounds, we obtain
\begin{eqnarray}
    \label{distance tracking}
    \limsup_{i\to\infty}\|\mD_i\|_{L_2}\le\big(2c_\eta\sqrt{c_3}+2\sqrt{2d_1d_2d_3}\,c_\omega\mu_0+2c_G\big)\sqrt{\mu\tau}=c_2\sqrt{\mu\tau},
\end{eqnarray}
where $c_2=2c_\eta\sqrt{c_3}+2\sqrt{2d_1d_2d_3}\,c_\omega\mu_0+2c_G$.

\paragraph[Step 3: conclusion.]{Step 3: conclusion.}
 Taking $\limsup_{i\to\infty}$ in \eqref{comparison tracking pre} and using \eqref{size long term tracking} and \eqref{distance tracking}, we obtain
\begin{eqnarray}
    \label{comparison tracking}
    \limsup_{i\to\infty}\big|\E[\|\mDelta_i\|_F^2]-\E[\|\wt\mDelta_i\|_F^2]\big|\le2c_2\sqrt{c_3}\,\sqrt{\mu}\,\tau+c_2^2\mu\tau\le\big(2c_2\sqrt{c_3}+c_2^2\sqrt{\mu_0}\big)\sqrt{\mu}\,\tau,
\end{eqnarray}
which is \eqref{tracking true main}, since $\E[\|\wt\mDelta_i\|_F^2]$ converges.
\end{proof}

\subsection{Proof of \Cref{The main theorem of tracking ability in the model II}}
\label{Proof of corollary 2}

\begin{proof}
Let $\nu$, $\beta$, $\theta$ and $s$ be as in \eqref{tracking constants}, where $\nu\in(0,1)$ since $\mu\le\mu_0<\frac{1}{\sigma_a^2}$, and let $\Gamma_1$ denote $\Gamma_{d_4}$ in \eqref{Gamma definition main} with $d_4=1$. Since, by \eqref{tracking long term main},
\begin{eqnarray}
    \label{cor2 split}
    \lefteqn{\lim_{i\to\infty}\E[\|\wt\calX_i-\calX_i^\star\|_F^2]-\frac{\mu d_1d_2d_3\sigma_e^2}{2}-\frac{d_1d_2d_3\lambda^2\sigma_q^2}{2(\lambda+\nu)}}\nonumber\\
    &\!\!\!\!=\!\!\!\!&\underbrace{\frac{\mu d_1d_2d_3\sigma_e^2}{2-\nu}\rho_{d_4}-\frac{\mu d_1d_2d_3\sigma_e^2}{2}}_{\text{(B1)}}
    +\underbrace{d_1d_2d_3\,s\,(\Gamma_{d_4}-\Gamma_1)}_{\text{(B2)}}
    +\underbrace{d_1d_2d_3\,s\,\Gamma_1-\frac{d_1d_2d_3\lambda^2\sigma_q^2}{2(\lambda+\nu)}}_{\text{(B3)}},
\end{eqnarray}
we bound the three terms in turn. For (B1), \eqref{long term first order} in the proof of \Cref{The main theorem of error in the model II} gives $|\text{(B1)}|\le c_4\mu^2$.

For (B2), let $\vc^{(1)}=\begin{bmatrix}-1&1&0&\cdots&0\end{bmatrix}^\top$, which is the vector $\vc$ of \Cref{lemma:Gamma} with $d_4=1$ (where $c_{-1}=-1$ and $c_0=\beta+\nu=1$), padded with zeros. Then $\Gamma_1=\frac{B(\vc^{(1)},\vc^{(1)})}{1-\beta^2}$ by \eqref{Gamma definition main} and \eqref{B definition}, and, since $B$ is bilinear and symmetric,
\begin{eqnarray}
    \label{window effect pre}
    (1-\beta^2)(\Gamma_{d_4}-\Gamma_1)=B(\vc,\vc)-B(\vc^{(1)},\vc^{(1)})=B(\vc-\vc^{(1)},\vc+\vc^{(1)}).
\end{eqnarray}
By the definition of $c_r$ in \Cref{prop:tracking long term} and $\beta=1-\nu$, we have
\begin{eqnarray}
    \label{c split}
    \vc-\vc^{(1)}=\nu\begin{bmatrix}0&\frac{1}{d_4}-1&\frac{1}{d_4}&\cdots&\frac{1}{d_4}\end{bmatrix}^\top,
\end{eqnarray}
which has zero sum and $\sum_r|c_r-c^{(1)}_r|=\frac{2(d_4-1)\nu}{d_4}\le2\nu$. Moreover, $\vc+\vc^{(1)}$ has zero sum and, by \eqref{c properties}, $\sum_r|c_r+c^{(1)}_r|\le\sum_r|c_r|+\sum_r|c^{(1)}_r|=4$. Hence, applying \Cref{lemma:Gamma}(b) to \eqref{window effect pre} and using $1-\beta^2=\nu(2-\nu)$ and $2-\nu>1$, we have
\begin{eqnarray}
    \label{window effect}
    |\Gamma_{d_4}-\Gamma_1|\le\frac{\lambda d_4^2\cdot2\nu\cdot4}{\nu(2-\nu)}\le8\lambda d_4^2.
\end{eqnarray}
Multiplying \eqref{window effect} by $d_1d_2d_3s=\frac{d_1d_2d_3\lambda\sigma_q^2}{2-\lambda}$ and using $2-\lambda>1$, we obtain
\begin{eqnarray}
    \label{B2 bound}
    |\text{(B2)}|\le8d_4^2\,d_1d_2d_3\lambda^2\sigma_q^2.
\end{eqnarray}

For (B3), we first compute $\Gamma_1$. By \Cref{lemma:Gamma}(a) with $\vx=\vy=\vc^{(1)}$, the terms with $r=r'$ vanish since $G(0)=0$, and the two terms with $r\ne r'$ have $c^{(1)}_rc^{(1)}_{r'}=-1$ and $|r-r'|=1$, so that $\Gamma_1=\frac{2\lambda G(1)}{1-\beta^2}=\frac{2\lambda g(0)}{1-\beta^2}$. By \eqref{g G definition}, $g(0)=\frac{1-\beta}{1-\beta\theta}=\frac{\nu}{1-\beta\theta}$. Hence, using $1-\beta^2=\nu(2-\nu)$ and $d_1d_2d_3s=\frac{d_1d_2d_3\lambda\sigma_q^2}{2-\lambda}$, we have
\begin{eqnarray}
    \label{Gamma one}
    d_1d_2d_3\,s\,\Gamma_1=\frac{d_1d_2d_3\lambda\sigma_q^2}{2-\lambda}\cdot\frac{2\lambda}{(2-\nu)(1-\beta\theta)}=\frac{2d_1d_2d_3\lambda^2\sigma_q^2}{(2-\lambda)(2-\nu)(1-\beta\theta)}.
\end{eqnarray}
Writing (B3) over a common denominator, and using $(2-\lambda)(2-\nu)=4-(2\lambda+2\nu-\lambda\nu)$ and $4(\lambda+\nu)-4(1-\beta\theta)=4\lambda\nu$ from \eqref{denominator max}, we have
\begin{eqnarray}
    \label{d4 one versus M0}
    \text{(B3)}&\!\!\!\!=\!\!\!\!&d_1d_2d_3\lambda^2\sigma_q^2\,\frac{4(\lambda+\nu)-(2-\lambda)(2-\nu)(1-\beta\theta)}{2(2-\lambda)(2-\nu)(1-\beta\theta)(\lambda+\nu)}\nonumber\\
    &\!\!\!\!=\!\!\!\!&d_1d_2d_3\lambda^2\sigma_q^2\Big(\frac{2\lambda\nu}{(2-\lambda)(2-\nu)(1-\beta\theta)(\lambda+\nu)}+\frac{2\lambda+2\nu-\lambda\nu}{2(2-\lambda)(2-\nu)(\lambda+\nu)}\Big).
\end{eqnarray}
Since $(2-\lambda)(2-\nu)>1$, $(1-\beta\theta)(\lambda+\nu)\ge\lambda\nu$ by \eqref{denominator max}, and $0\le2\lambda+2\nu-\lambda\nu\le2(\lambda+\nu)$, the first term in the parentheses of \eqref{d4 one versus M0} is at most $2$ and the second is at most $1$, so that
\begin{eqnarray}
    \label{d4 one bound}
    |\text{(B3)}|\le3\,d_1d_2d_3\lambda^2\sigma_q^2.
\end{eqnarray}

Substituting $|\text{(B1)}|\le c_4\mu^2$, \eqref{B2 bound} and \eqref{d4 one bound} into \eqref{cor2 split}, and using $\nu=\mu\sigma_a^2$, we obtain
\begin{eqnarray}
    \label{tracking first order main}
    \bigg|\lim_{i\to\infty}\E[\|\wt\calX_i-\calX_i^\star\|_F^2]-\frac{\mu d_1d_2d_3\sigma_e^2}{2}-\frac{d_1d_2d_3\lambda^2\sigma_q^2}{2(\lambda+\mu\sigma_a^2)}\bigg|\le c_4\mu^2+(8d_4^2+3)\,d_1d_2d_3\lambda^2\sigma_q^2.
\end{eqnarray}
Finally, by the triangle inequality, \eqref{tracking true main} and \eqref{tracking first order main}, the left-hand side of \eqref{tracking ability performance in the ToM regression model main theorem} is at most
\begin{eqnarray}
    \label{tracking final}
    \lefteqn{C\sqrt{\mu}\Big(\mu+\frac{\lambda^2\sigma_q^2}{\mu}\Big)+c_4\mu^2+(8d_4^2+3)\,d_1d_2d_3\lambda^2\sigma_q^2}\nonumber\\
    &\!\!\!\!\le\!\!\!\!&\Big(C+\max\{c_4,(8d_4^2+3)d_1d_2d_3\}\sqrt{\mu_0}\Big)\sqrt{\mu}\Big(\mu+\frac{\lambda^2\sigma_q^2}{\mu}\Big),
\end{eqnarray}
where we use $\mu^2=\sqrt{\mu}\cdot\mu\cdot\sqrt{\mu}\le\sqrt{\mu_0}\sqrt{\mu}\,\mu$ and $\lambda^2\sigma_q^2=\sqrt{\mu}\cdot\frac{\lambda^2\sigma_q^2}{\mu}\cdot\sqrt{\mu}\le\sqrt{\mu_0}\sqrt{\mu}\,\frac{\lambda^2\sigma_q^2}{\mu}$. This proves \Cref{The main theorem of tracking ability in the model II} with $C''=C+\max\{c_4,(8d_4^2+3)d_1d_2d_3\}\sqrt{\mu_0}$.
\end{proof}

\subsection[The optimal step size]{The optimal step size}
\label{Proof of optimal step size}

\begin{proof}[Derivation of $\mu^\star$ in \Cref{sec:tracking ability performance}]
Let $f(\mu)=\frac{\mu d_1d_2d_3\sigma_e^2}{2}+\frac{d_1d_2d_3\lambda^2\sigma_q^2}{2(\lambda+\mu\sigma_a^2)}$. Differentiating, we have
\begin{eqnarray}
    \label{f derivative}
    f'(\mu)=\frac{d_1d_2d_3}{2}\bigg(\sigma_e^2-\frac{\lambda^2\sigma_q^2\sigma_a^2}{(\lambda+\mu\sigma_a^2)^2}\bigg),
\end{eqnarray}
which is increasing in $\mu>0$, so $f$ is convex. If $\sigma_q\sigma_a>\sigma_e$, then $f'(\mu)=0$ if and only if $\lambda+\mu\sigma_a^2=\frac{\lambda\sigma_q\sigma_a}{\sigma_e}$, i.e., $\mu=\mu^\star=\frac{\lambda(\sigma_q\sigma_a-\sigma_e)}{\sigma_e\sigma_a^2}$. If $\sigma_q\sigma_a\le\sigma_e$, then $f'(\mu)>0$ for all $\mu>0$.
\end{proof}

\section[Proofs of \Cref{RIP condition fro the ToM regression} and \Cref{prop:necessity}]{Proofs of \Cref{RIP condition fro the ToM regression} and \Cref{prop:necessity}}
\label{RIP proof}

\subsection{Proof of \Cref{RIP condition fro the ToM regression}}
\label{Proof of RIP}

\begin{proof}
We use the following result.
\begin{lemma}[{\cite[Theorem 4.6.1, Eq.~(4.22)]{vershynin2018high}}]
\label{lemma:vershynin}
Let $\mB\in\R^{m\times d}$ have independent, mean-zero, isotropic rows with $\max_j\|\mB(j,:)\|_{\psi_2}\le K$. There is an universal constant $C_0$ such that, for any $t\ge0$, with probability at least $1-2e^{-t^2}$,
\begin{eqnarray}
    \label{vershynin bound}
    \Big\|\frac{1}{m}\mB^\top\mB-\mId_{d\times d}\Big\|\le K^2\max\{\epsilon,\epsilon^2\},\qquad \epsilon=C_0\Big(\sqrt{\tfrac{d}{m}}+\tfrac{t}{\sqrt m}\Big).
\end{eqnarray}
\end{lemma}

\paragraph[Step 1: reduction.]{Step 1: reduction.}
 Let $\mX=\text{reshape}(\calX,[d_1d_2,d_3])$, whose rows $\vx_{s_1,s_2}^\top$ are the mode-3 fibers, and $\mS=\frac{1}{d_4\sigma_a^2}\mA\mA^\top$. Since $(\calX\times_3^1\mA)(s_1,s_2,s_4)=(\mX\mA)((s_1,s_2),s_4)$, we have $\|\calX\times_3^1\mA\|_F^2=\trace(\mX\mA\mA^\top\mX^\top)$ and $\|\calX\|_F^2=\trace(\mX\mX^\top)$, and therefore
\begin{eqnarray}
    \label{RIP difference}
    \Big|\frac{1}{d_4\sigma_a^2}\|\calX\times_3^1\mA\|_F^2-\|\calX\|_F^2\Big|=\Big|\sum_{s_1,s_2}\vx_{s_1,s_2}^\top(\mS-\mId_{d_3\times d_3})\vx_{s_1,s_2}\Big|\le\|\mS-\mId_{d_3\times d_3}\|\,\|\calX\|_F^2 .
\end{eqnarray}
Hence \eqref{RIP condition fro the ToM regression single} holds for all $\calX$ on the event $\{\|\mS-\mId_{d_3\times d_3}\|\le\delta_{d_3}\}$.

\paragraph[Step 2: concentration.]{Step 2: concentration.}
 The rows of $\mB=\mA^\top/\sigma_a\in\R^{d_4\times d_3}$ are $\va_i^\top/\sigma_a,\dots,\va_{i-d_4+1}^\top/\sigma_a$, which by (i) are independent, mean zero and isotropic, with sub-Gaussian norm at most $K$, and $\frac{1}{d_4}\mB^\top\mB=\mS$. Taking $m=d_4$, $d=d_3$ and $t=\frac{\delta_{d_3}\sqrt{d_4}}{2C_0K^2}$ in \Cref{lemma:vershynin}, if $d_4\ge\frac{4C_0^2K^4d_3}{\delta_{d_3}^2}$, then $\epsilon\le\frac{\delta_{d_3}}{K^2}\le\delta_{d_3}<1$ by $K\ge1$, and therefore
\begin{eqnarray}
    \label{RIP step two}
    \|\mS-\mId_{d_3\times d_3}\|\le K^2\epsilon\le\delta_{d_3}
\end{eqnarray}
with probability at least $1-2e^{-\delta_{d_3}^2d_4/(4C_0^2K^4)}$. By Step~1, this proves \Cref{RIP condition fro the ToM regression} with $C=4C_0^2$ and $c=\frac{1}{4C_0^2}$.
\end{proof}

\subsection[Proof of \Cref{prop:necessity}]{Proof of \Cref{prop:necessity}}
\label{Proof of necessity}
\begin{proof}
Since $\rank(\mA^\top)\le d_4<d_3$, there is $\vz\neq{\bf 0}$ with $\mA^\top\vz={\bf 0}$. For $\calX=\vu\circ\vv\circ\vz$, we have
\begin{eqnarray}
    \label{null space tensor}
    (\calX\times_3^1\mA)(s_1,s_2,s_4)=\vu(s_1)\vv(s_2)(\mA^\top\vz)(s_4)=0 .
\end{eqnarray}
For $\vu,\vv\neq{\bf 0}$, $\|\calX\|_F>0$, so the lower bound in \eqref{RIP condition fro the ToM regression single} fails.
\end{proof}

\section[Proofs for \Cref{sec: low-dimensional ToM regression}]{Proofs for \Cref{sec: low-dimensional ToM regression}}
\label{sparse and low-rank proof}

\subsection[Proof of \Cref{prop:unstructured baseline}]{Proof of \Cref{prop:unstructured baseline}}
\label{Proof of baseline}

\begin{proof}
Write $\delta=\delta_{d_3}$. Let $\mX=\text{reshape}(\calX,[d_1d_2,d_3])$, and let $\mY,\mE\in\R^{d_1d_2\times d_4}$ collect the vectorized slices of $\calY$ and $\calE$, so that $\calX\times_3^1\mA$ corresponds to $\mX\mA$ and $\mY=\mX^\star\mA+\mE$. On the event \eqref{RIP condition fro the ToM regression single}, applying it to tensors with a single nonzero mode-3 fiber $\vx$ gives $(1-\delta)\|\vx\|_2^2\le\frac{1}{d_4\sigma_a^2}\vx^\top\mA\mA^\top\vx\le(1+\delta)\|\vx\|_2^2$, so the eigenvalues of $\mA\mA^\top$ lie in $[(1-\delta)d_4\sigma_a^2,(1+\delta)d_4\sigma_a^2]$; in particular, $\mA\mA^\top$ is invertible. The objective $\frac{1}{2d_4}\|\mX\mA-\mY\|_F^2$ is then strictly convex, with unique minimizer $\wh\mX_{\text{LS}}=\mY\mA^\top(\mA\mA^\top)^{-1}$, so $\wh\mX_{\text{LS}}-\mX^\star=\mE\mM$ with $\mM=\mA^\top(\mA\mA^\top)^{-1}$. The $d_1d_2$ rows of $\mE$ are i.i.d.\ $\calN({\bf 0},\sigma_e^2\mId_{d_4\times d_4})$ and independent of $\mA$, which implies that
\begin{eqnarray}
    \E\big[\|\mE\mM\|_F^2\,\big|\,\mA\big]=d_1d_2\sigma_e^2\trace(\mM^\top\mM)=d_1d_2\sigma_e^2\trace\big((\mA\mA^\top)^{-1}\big).
\end{eqnarray}
The $d_3$ eigenvalues of $(\mA\mA^\top)^{-1}$ lie in $[\frac{1}{(1+\delta)d_4\sigma_a^2},\frac{1}{(1-\delta)d_4\sigma_a^2}]$, which gives \eqref{unstructured baseline bound}.
\end{proof}

\subsection{Proof of \Cref{theorem:sparse  structure conclusion main paper,theorem: low-rank structure conclusion main paper,theorem:sparse and low-rank structure conclusion main paper}}
\label{Proof of structured}

\begin{proof}
We first establish \Cref{theorem:sparse and low-rank structure conclusion main paper}, which addresses the general sparse plus low-rank setting, and then treat the purely sparse and purely low-rank cases (\Cref{theorem:sparse  structure conclusion main paper,theorem: low-rank structure conclusion main paper}) with the same argument. Throughout, we fix a time instant $i$, assume \eqref{snapshot assumption}, and drop the time index, so that $\calY=(\calS^\star+\calL^\star)\times_3^1\mA+\calE$. We write $D=d_1d_2d_3$ and $\delta=\delta_{d_3}$, so that $s\le D/3$ by the assumptions of \Cref{theorem:sparse  structure conclusion main paper,theorem:sparse and low-rank structure conclusion main paper}, and denote by $F$ the event that $\mA$ satisfies \eqref{RIP condition fro the ToM regression single}.

We define $\wh\calS$ and $\wh\calL$ as the minimizers of the following optimization problem:
\begin{eqnarray}
    \label{the definition of minimum value of sparse and LR loss}
    (\wh\calS, \wh\calL) = \min_{\calS\in \wt\setS_{s}, \calL\in \wt\setL_{\vr}}  \frac{1}{2d_4}\|(\calS + \calL) \times_{3}^{1} \mA -   \calY  \|_F^2.
\end{eqnarray}
Let $\wh\calH=\wh\calS + \wh\calL - \calS^\star - \calL^\star$. Since the objective at $(\wh\calS,\wh\calL)$ is at most the objective at $(\calS^\star,\calL^\star)$ and $(\wh\calS+\wh\calL)\times_3^1\mA-\calY=\wh\calH\times_3^1\mA-\calE$, we have
\begin{eqnarray}
    \label{minimum value and ground-truth value relationship}
    0 &\!\!\!\!\leq\!\!\!\!&   \frac{1}{d_4}\|(\calS^\star + \calL^\star) \times_{3}^{1} \mA -   \calY  \|_F^2 -  \frac{1}{d_4}\|(\wh\calS + \wh\calL) \times_{3}^{1} \mA -   \calY  \|_F^2\nonumber\\
    &\!\!\!\!=\!\!\!\!& \frac{2}{d_4}\<  \calE,  \wh\calH \times_{3}^{1} \mA   \>  -  \frac{1}{d_4}\|\wh\calH \times_{3}^{1} \mA   \|_F^2,
\end{eqnarray}
and further derive
\begin{eqnarray}
    \label{minimum value and ground-truth value relationship1}
    \frac{1}{d_4}\|\wh\calH \times_{3}^{1} \mA   \|_F^2\leq \frac{2}{d_4}\<  \calE,  \wh\calH \times_{3}^{1} \mA   \>.
\end{eqnarray}

According to \Cref{RIP condition fro the ToM regression}, on the event $F$, we have
\begin{eqnarray}
    \label{lower bound of the left term}
    \frac{1}{d_4}\|\wh\calH \times_{3}^{1} \mA   \|_F^2\geq (1-\delta)\sigma_a^2\|\wh\calH\|_F^2.
\end{eqnarray}

On the other hand, we need to apply the covering argument to bound $\frac{2}{d_4}\<  \calE,  \wh\calH \times_{3}^{1} \mA   \>$. We define two sets $\wt\setS_{2s} = \{ \calS:    \|\calS\|_F\leq 2S, \text{supp}(\calS) \leq 2s     \}$ and $\wt\setL_{2\vr} = \{ \calL:  \|\calL\|_F\leq 2L, \text{rank}(\calL) \leq 2\vr     \}$, where the rank is taken in the chosen decomposition, so that $\wh\calS-\calS^\star\in\wt\setS_{2s}$ and $\wh\calL-\calL^\star\in\wt\setL_{2\vr}$. Since these sets have bounded norms, they are not closed under scaling, and we use nets at a fixed resolution $\epsilon_0>0$. For each set $\Omega$ of $2s$ entries, an $\epsilon_0$-net of $\{\calS\in\wt\setS_{2s}:\text{supp}(\calS)\subseteq\Omega\}$ has at most $(1+\frac{4S}{\epsilon_0})^{2s}$ elements, so $\wt\setS_{2s}$ has an $\epsilon_0$-net $\ol\setS_{2s}$ with $|\ol\setS_{2s}|\le\binom{D}{2s}(1+\frac{4S}{\epsilon_0})^{2s}$; by Lemmas~\ref{covering number of Tucker decomposition}--\ref{covering number of Tubal decomposition}, $\wt\setL_{2\vr}$ has an $\epsilon_0$-net $\ol\setL_{2\vr}$ with $\log|\ol\setL_{2\vr}|\le c_0N\log(1+\frac{c_1L}{\epsilon_0})$ for universal constants $c_0,c_1$. Choose $\calS^{(h)}\in\ol\setS_{2s}$ and $\calL^{(h)}\in\ol\setL_{2\vr}$ with $\|\wh\calS-\calS^\star-\calS^{(h)}\|_F\le\epsilon_0$ and $\|\wh\calL-\calL^\star-\calL^{(h)}\|_F\le\epsilon_0$. Then we can rewrite the right hand side of \eqref{minimum value and ground-truth value relationship1} as
\begin{eqnarray}
    \label{upper bound another form}
    \frac{2}{d_4}\<  \calE,  \wh\calH \times_{3}^{1} \mA   \>
    =\frac{2}{d_4}\<  \calE,  (\calS^{(h)} + \calL^{(h)}) \times_{3}^{1} \mA   \>+\frac{2}{d_4}\<  \calE,  (\wh\calH-\calS^{(h)} - \calL^{(h)}) \times_{3}^{1} \mA   \>,
\end{eqnarray}
where $\|\wh\calH-\calS^{(h)}-\calL^{(h)}\|_F\le2\epsilon_0$.

For the second term of \eqref{upper bound another form}, by the Cauchy--Schwarz inequality and the upper bound in \eqref{RIP condition fro the ToM regression single}, which holds for all tensors, we have
\begin{eqnarray}
    \label{approximation term}
    \frac{2}{d_4}\<  \calE,  (\wh\calH-\calS^{(h)} - \calL^{(h)}) \times_{3}^{1} \mA   \>\le\frac{2}{d_4}\|\calE\|_F\sqrt{(1+\delta)d_4}\,\sigma_a\cdot2\epsilon_0 .
\end{eqnarray}
Since the $Dd_4$ entries of $\calE$ are i.i.d.\ $\calN(0,\sigma_e^2)$, $\|\calE\|_F^2/\sigma_e^2$ is a chi-square random variable with $Dd_4$ degrees of freedom. Hence, by \cite[Lemma 1]{laurent2000adaptive}, for every $x>0$,
\begin{eqnarray}
    \label{chi square tail}
    \P{\frac{\|\calE\|_F^2}{\sigma_e^2}\ge Dd_4+2\sqrt{Dd_4x}+2x}\le e^{-x}.
\end{eqnarray}
Taking $x=Dd_4$ in \eqref{chi square tail}, so that $Dd_4+2\sqrt{Dd_4x}+2x=5Dd_4$, we obtain
\begin{eqnarray}
    \label{noise norm bound}
    \P{\|\calE\|_F^2\le5\sigma_e^2Dd_4}\ge1-e^{-Dd_4},
\end{eqnarray}
Substituting $\|\calE\|_F\le\sigma_e\sqrt{5Dd_4}$ from \eqref{noise norm bound} into \eqref{approximation term}, we obtain, on the event in \eqref{noise norm bound},
\begin{eqnarray}
    \label{second term bound}
    \frac{2}{d_4}\<  \calE,  (\wh\calH-\calS^{(h)} - \calL^{(h)}) \times_{3}^{1} \mA   \>\le4\epsilon_0\sqrt{5(1+\delta)D}\,\sigma_a\sigma_e .
\end{eqnarray}

For the first term of \eqref{upper bound another form}, note that each element in $\calE$ follows the normal distribution $\calN(0,\sigma_e^2)$. When conditional on $\mA$, for any fixed $\calH=\calS^{(h)} + \calL^{(h)}$ in the net $\setN=\{\calS^{(h)}+\calL^{(h)}:\calS^{(h)}\in\ol\setS_{2s},\calL^{(h)}\in\ol\setL_{2\vr}\}$, $\frac{1}{d_4}\<  \calE,  \calH \times_{3}^{1} \mA   \>$ has normal distribution with zero mean and variance $\frac{\sigma_e^2\|\calH \times_{3}^{1} \mA\|_F^2 }{d_4^2}$, which, on the event $F$, is at most $\frac{(1+\delta)\sigma_a^2\sigma_e^2\|\calH\|_F^2}{d_4}$. Hence, for every $t\ge0$,
\begin{eqnarray}
    \label{concentration inequality of Gaussian distribution}
    \P{\frac{1}{d_4}\<  \calE,  \calH \times_{3}^{1} \mA   \>> t\|\calH\|_F\, \big|\, \mA  }\leq e^{-\frac{d_4 t^2}{2(1+\delta)\sigma_a^2\sigma_e^2}},
\end{eqnarray}
which also holds for $\calH={\bf 0}$, since the left-hand side is then zero. Note that the threshold $t\|\calH\|_F$ depends only on the fixed tensor $\calH$, and not on $\wh\calH$. Let
\begin{eqnarray}
    \label{union threshold}
    t_0=\sigma_a\sigma_e\sqrt{\frac{2(1+\delta)(\log|\setN|+s\log\frac{D}{s})}{d_4}}.
\end{eqnarray}
We now apply \eqref{concentration inequality of Gaussian distribution} to all elements of $\setN$ and get, by the union bound,
\begin{eqnarray}
    \label{the tail function of fixed gaussian random variable 2}
    \P{\frac{1}{d_4}\<  \calE,  \calH \times_{3}^{1} \mA   \>\le t_0\|\calH\|_F \text{ for all }\calH\in\setN\, \big|\, \mA}
    &\!\!\!\!\geq\!\!\!\!&1-|\setN|e^{-(\log|\setN|+s\log\frac{D}{s})}\nonumber\\
    &\!\!\!\!=\!\!\!\!&1-|\setN|\cdot\frac{1}{|\setN|}\,e^{-s\log\frac{D}{s}}=1-e^{-s\log\frac{D}{s}}.
\end{eqnarray}
On the event in \eqref{the tail function of fixed gaussian random variable 2}, taking $\calH=\calS^{(h)}+\calL^{(h)}\in\setN$ and using $\|\calS^{(h)}+\calL^{(h)}\|_F\le\|\wh\calH\|_F+\|\wh\calH-\calS^{(h)}-\calL^{(h)}\|_F\le\|\wh\calH\|_F+2\epsilon_0$, we obtain
\begin{eqnarray}
    \label{first term bound}
    \frac{2}{d_4}\<  \calE,  (\calS^{(h)} + \calL^{(h)}) \times_{3}^{1} \mA   \>\le2t_0\|\calS^{(h)}+\calL^{(h)}\|_F\le2t_0\|\wh\calH\|_F+4\epsilon_0t_0 .
\end{eqnarray}

On the intersection of $F$ and the events in \eqref{noise norm bound} and \eqref{the tail function of fixed gaussian random variable 2}, substituting \eqref{first term bound} and \eqref{second term bound} into \eqref{upper bound another form}, and combining with \eqref{minimum value and ground-truth value relationship1} and \eqref{lower bound of the left term}, we obtain
\begin{eqnarray}
    \label{SL quadratic inequality}
    (1-\delta)\sigma_a^2\|\wh\calH\|_F^2\le2t_0\|\wh\calH\|_F+4\epsilon_0\big(t_0+\sqrt{5(1+\delta)D}\,\sigma_a\sigma_e\big),
\end{eqnarray}
i.e., $x^2\le bx+c$ for $x=\|\wh\calH\|_F$, with $b=\frac{2t_0}{(1-\delta)\sigma_a^2}$ and $c=\frac{4\epsilon_0(t_0+\sqrt{5(1+\delta)D}\sigma_a\sigma_e)}{(1-\delta)\sigma_a^2}$. Since $x^2\le bx+c$ implies $x\le\frac{b+\sqrt{b^2+4c}}{2}\le b+\sqrt{c}$, and $\sqrt{c}\le\sqrt{2\epsilon_0b}+\sqrt{\frac{4\epsilon_0\sqrt{5(1+\delta)D}\sigma_e}{(1-\delta)\sigma_a}}$ with $\sqrt{2\epsilon_0b}\le\frac{b}{2}+\epsilon_0$, we have
\begin{eqnarray}
    \label{SL error bound pre}
    \|\wh\calH\|_F\le\frac32b+\epsilon_0+\sqrt{\frac{4\epsilon_0\sqrt{5(1+\delta)D}\,\sigma_e}{(1-\delta)\sigma_a}} .
\end{eqnarray}
We take $\epsilon_0=\frac{\sigma_e}{\sigma_ad_4\sqrt{5D}}$ and bound the three terms on the right-hand side of \eqref{SL error bound pre} in turn. For the last term, substituting $\epsilon_0$, we have
\begin{eqnarray}
    \label{SL last term}
    \sqrt{\frac{4\epsilon_0\sqrt{5(1+\delta)D}\,\sigma_e}{(1-\delta)\sigma_a}}=\sqrt{\frac{4\sqrt{1+\delta}\,\sigma_e^2}{(1-\delta)d_4\sigma_a^2}}=\frac{2(1+\delta)^{1/4}}{\sqrt{(1-\delta)d_4}}\,\frac{\sigma_e}{\sigma_a}\le\frac{2\sqrt{1+\delta}}{(1-\delta)\sqrt{d_4}}\,\frac{\sigma_e}{\sigma_a},
\end{eqnarray}
where the inequality uses $(1+\delta)^{1/4}\le\sqrt{1+\delta}$ and $\sqrt{1-\delta}\ge1-\delta$. For the second term, since $d_4\sqrt{5D}\ge\sqrt{d_4}$, we have
\begin{eqnarray}
    \label{SL second term}
    \epsilon_0\le\frac{1}{\sqrt{d_4}}\,\frac{\sigma_e}{\sigma_a}\le\frac{\sqrt{1+\delta}}{(1-\delta)\sqrt{d_4}}\,\frac{\sigma_e}{\sigma_a}.
\end{eqnarray}
For the first term, by $b=\frac{2t_0}{(1-\delta)\sigma_a^2}$ and the definition of $t_0$ in \eqref{union threshold}, we have
\begin{eqnarray}
    \label{SL first term}
    \frac32b=\frac{3t_0}{(1-\delta)\sigma_a^2}=\frac{3\sqrt2\sqrt{(1+\delta)(\log|\setN|+s\log\frac{D}{s})}}{(1-\delta)\sqrt{d_4}}\,\frac{\sigma_e}{\sigma_a}.
\end{eqnarray}
It remains to bound $\log|\setN|+s\log\frac{D}{s}$. Since every element of $\setN$ is the sum of an element of $\ol\setS_{2s}$ and an element of $\ol\setL_{2\vr}$, we have $|\setN|\le|\ol\setS_{2s}||\ol\setL_{2\vr}|$, and hence, by the sizes of $\ol\setS_{2s}$ and $\ol\setL_{2\vr}$ given before \eqref{upper bound another form},
\begin{eqnarray}
    \label{SL net size}
    \log|\setN|\le\log|\ol\setS_{2s}|+\log|\ol\setL_{2\vr}|\le\log\binom{D}{2s}+2s\log\Big(1+\frac{4S}{\epsilon_0}\Big)+c_0N\log\Big(1+\frac{c_1L}{\epsilon_0}\Big).
\end{eqnarray}
We bound the three terms of \eqref{SL net size} in turn. For the first term, using $\binom{D}{2s}\le(\frac{eD}{2s})^{2s}$, where $e$ is Euler's number, and $\log\frac{D}{s}\ge\log3\ge1$ since $s\le D/3$, we have
\begin{eqnarray}
    \label{SL binomial}
    \log\binom{D}{2s}\le2s\log\frac{eD}{2s}\le2s\Big(1+\log\frac{D}{s}\Big)\le4s\log\frac{D}{s}.
\end{eqnarray}
For the second term, substituting $\epsilon_0$ and using $S\le S+L$ and $4\sqrt5\ge1$, we have
\begin{eqnarray}
    \label{SL log term S}
    \log\Big(1+\frac{4S}{\epsilon_0}\Big)&\!\!\!\!\le\!\!\!\!&\log\Big(1+4\sqrt5\,\frac{(S+L)\sigma_ad_4\sqrt{D}}{\sigma_e}\Big)\le\log\Big(4\sqrt5+4\sqrt5\,\frac{(S+L)\sigma_ad_4\sqrt{D}}{\sigma_e}\Big)\nonumber\\
    &\!\!\!\!=\!\!\!\!&\log(4\sqrt5)+\log\Big(1+\frac{(S+L)\sigma_ad_4\sqrt{D}}{\sigma_e}\Big).
\end{eqnarray}
Similarly, for the third term, using $L\le S+L$ and $\max\{\sqrt5c_1,1\}\ge1$, we have
\begin{eqnarray}
    \label{SL log term L}
    \log\Big(1+\frac{c_1L}{\epsilon_0}\Big)\le\log\max\{\sqrt5c_1,1\}+\log\Big(1+\frac{(S+L)\sigma_ad_4\sqrt{D}}{\sigma_e}\Big).
\end{eqnarray}
Substituting \eqref{SL binomial}, \eqref{SL log term S} and \eqref{SL log term L} into \eqref{SL net size}, adding $s\log\frac{D}{s}$, and using $s\le s\log\frac{D}{s}$ and the definition of $\Lambda_{\text{SL}}$ in \eqref{SL complexity}, we obtain
\begin{eqnarray}
    \label{SL net size bound}
    \log|\setN|+s\log\frac{D}{s}\le c_2\Lambda_{\text{SL}},\qquad c_2=\max\big\{5+2\log(4\sqrt5),\,c_0\log\max\{\sqrt5c_1,1\},\,2,\,c_0\big\}.
\end{eqnarray}
Substituting \eqref{SL net size bound} into \eqref{SL first term}, we obtain
\begin{eqnarray}
    \label{SL first term bound}
    \frac32b\le3\sqrt{2c_2}\,\frac{\sqrt{(1+\delta)\Lambda_{\text{SL}}}}{(1-\delta)\sqrt{d_4}}\,\frac{\sigma_e}{\sigma_a}.
\end{eqnarray}

Substituting \eqref{SL first term bound}, \eqref{SL second term} and \eqref{SL last term} into \eqref{SL error bound pre}, and using $\Lambda_{\text{SL}}\ge s\log\frac{D}{s}\ge1$, we obtain
\begin{eqnarray}
    \label{upper bound of error final}
    \|\wh\calS + \wh\calL - \calS^\star - \calL^\star\|_F\le\big(3\sqrt{2c_2}+3\big)\frac{\sqrt{(1+\delta)\Lambda_{\text{SL}}}}{(1-\delta)\sqrt{d_4}}\,\frac{\sigma_e}{\sigma_a},
\end{eqnarray}
with conditional probability at least $1-e^{-s\log\frac{D}{s}}-e^{-Dd_4}$ given $\mA\in F$, which proves \eqref{upper bound of error final main paper}. Since $\P{F}\ge1-2e^{-c\delta^2d_4/K^4}$ by \Cref{RIP condition fro the ToM regression}, the unconditional probability that \eqref{upper bound of error final} holds is at least $\P{F}(1-e^{-s\log\frac{D}{s}}-e^{-Dd_4})\ge1-2e^{-c\delta^2d_4/K^4}-e^{-s\log\frac{D}{s}}-e^{-Dd_4}$.

Finally, we specialize the argument to the two fundamental structural priors considered in this work. In both cases, the normalized error lies in a fixed set of unit-norm tensors, so a net of this set suffices, and the logarithmic factor in $\Lambda_{\text{SL}}$ does not appear.
\paragraph{Sparse tensors:}
By setting $\calL^\star = \wh\calL = {\bm 0}$ in \eqref{minimum value and ground-truth value relationship1} and using \eqref{lower bound of the left term}, we obtain
\begin{eqnarray}
    \label{minimum value and ground-truth value relationship1 sparse}
    (1-\delta)\sigma_a^2\|\wh\calS  - \calS^\star \|_F^2\le\frac{1}{d_4}\|(\wh\calS - \calS^\star) \times_{3}^{1} \mA   \|_F^2\leq \frac{2}{d_4}\<  \calE,  (\wh\calS - \calS^\star ) \times_{3}^{1} \mA   \>.
\end{eqnarray}
For the right-hand side of \eqref{minimum value and ground-truth value relationship1 sparse}, since $\wh\calS-\calS^\star$ has at most $2s$ nonzero entries, we can further bound it as
\begin{eqnarray}
    \label{upper bound of right hand side for sparse}
    \frac{2}{d_4}\<  \calE,  (\wh\calS - \calS^\star ) \times_{3}^{1} \mA   \> \leq 2\|\wh\calS - \calS^\star \|_F\,Z,\qquad Z=\max_{\calS\in\wh\setS_{2s}}\frac{1}{d_4}\<  \calE,  \calS \times_{3}^{1} \mA   \>,
\end{eqnarray}
where $\wh\setS_{2s} = \{ \calS:   \text{supp}(\calS) \leq 2s, \|\calS\|_F=1     \}$; this follows from the linearity of the inner product with $\calS=\frac{\wh\calS-\calS^\star}{\|\wh\calS-\calS^\star\|_F}$, and is trivial if $\wh\calS=\calS^\star$. For each set $\Omega$ of $2s$ entries, there exists a $\frac14$-net $\ol\setS_\Omega$ of $\{\calS\in\wh\setS_{2s}:\text{supp}(\calS)\subseteq\Omega\}$ with $|\ol\setS_\Omega|\le9^{2s}$, and we let $\ol{\wh \setS}_{2s}=\bigcup_\Omega\ol\setS_\Omega$, so that $|\ol{\wh \setS}_{2s}|\le\binom{D}{2s}9^{2s}$. For $\calS\in\wh\setS_{2s}$ supported on $\Omega$, choose $\calS^{(h)}\in\ol\setS_\Omega$ with $\|\calS-\calS^{(h)}\|_F\le\frac14$. Since $\calS-\calS^{(h)}$ is also supported on $\Omega$, $\frac{\calS-\calS^{(h)}}{\|\calS-\calS^{(h)}\|_F}\in\wh\setS_{2s}$ (if $\calS\ne\calS^{(h)}$), so its inner product term is at most $Z$. Moreover, $Z\ge0$: if $\calS\in\wh\setS_{2s}$, then $-\calS\in\wh\setS_{2s}$, and $\frac{1}{d_4}\<\calE,\calS\times_3^1\mA\>$ and $\frac{1}{d_4}\<\calE,(-\calS)\times_3^1\mA\>$ have opposite signs, so one of them is nonnegative. Hence, using $\|\calS-\calS^{(h)}\|_F\le\frac14$, we can further derive
\begin{eqnarray}
    \label{upper bound of right hand side for sparse1}
    \frac{1}{d_4}\<  \calE,  \calS \times_{3}^{1} \mA   \>&\!\!\!\!=\!\!\!\!&  \frac{1}{d_4}\<  \calE,  \calS^{(h)} \times_{3}^{1} \mA   \> + \|\calS - \calS^{(h)}\|_F\,\frac{1}{d_4}\<  \calE,  \frac{\calS - \calS^{(h)}}{\|\calS - \calS^{(h)}\|_F} \times_{3}^{1} \mA   \>\nonumber\\
    &\!\!\!\!\leq\!\!\!\!&\max_{\calS^{(h)}\in\ol{\wh \setS}_{2s}}\frac{1}{d_4}\<  \calE,  \calS^{(h)} \times_{3}^{1} \mA   \> + \frac14Z.
\end{eqnarray}
Taking the maximum over $\calS\in\wh\setS_{2s}$ gives
\begin{eqnarray}
    \label{upper bound of right hand side for sparse2}
    Z \leq \frac43\max_{\calS^{(h)}\in\ol{\wh \setS}_{2s}}\frac{1}{d_4}\<  \calE,  \calS^{(h)} \times_{3}^{1} \mA   \>.
\end{eqnarray}
Since each $\calS^{(h)}$ has unit norm, applying \eqref{concentration inequality of Gaussian distribution} and the union bound as in \eqref{the tail function of fixed gaussian random variable 2} with $\setN$ replaced by $\ol{\wh \setS}_{2s}$, the maximum in \eqref{upper bound of right hand side for sparse2} is at most $\sigma_a\sigma_e\sqrt{2(1+\delta)(\log|\ol{\wh \setS}_{2s}|+s\log\frac{D}{s})/d_4}$ with probability at least $1-e^{-s\log\frac{D}{s}}$. Using $\binom{D}{2s}\le(\frac{eD}{2s})^{2s}$, where $e$ is Euler's number, we have $\log|\ol{\wh \setS}_{2s}|\le2s\log\frac{9eD}{2s}=2s\big(\log\frac{9e}{2}+\log\frac{D}{s}\big)\le7s\log\frac{D}{s}$, where the last step uses $\log\frac{9e}{2}\le\frac52\log3\le\frac52\log\frac{D}{s}$ since $s\le D/3$. Hence $\log|\ol{\wh \setS}_{2s}|+s\log\frac{D}{s}\le8s\log\frac{D}{s}$, and, combining \eqref{minimum value and ground-truth value relationship1 sparse}--\eqref{upper bound of right hand side for sparse2}, we obtain
\begin{eqnarray}
    \label{upper bound of error final main paper sparse appendix}
    \|\wh\calS  - \calS^\star \|_F\leq\frac{2Z}{(1-\delta)\sigma_a^2}\le \frac{32}{3}\,\frac{\sqrt{(1+\delta)s\log(\frac{D}{s})}}{(1-\delta)\sqrt{d_4}}\,\frac{\sigma_e}{\sigma_a},
\end{eqnarray}
which proves \eqref{upper bound of error final main paper sparse}.

\paragraph{Low-rank tensors:}
 By applying an analogous argument as in the case of sparse tensors, with $\wh\setS_{2s}$ replaced by $\wh\setL_{2\vr}=\{\calL:\text{rank}(\calL)\le2\vr,\ \|\calL\|_F=1\}$, which contains $\frac{\wh\calL-\calL^\star}{\|\wh\calL-\calL^\star\|_F}$ by the subadditivity of the ranks, we only need to replace the net. We use the nets of Lemmas~\ref{covering number of Tucker decomposition}--\ref{covering number of Tubal decomposition} with ranks $2\vr$ at resolution $\epsilon=\frac{1}{4M}$, where $M=4$ for the Tucker and $M=3$ for the TT and tubal decompositions. By the factor-wise construction of these nets, for every $\calL\in\wh\setL_{2\vr}$, the net contains $\calL^{(h)}$ such that
\begin{eqnarray}
    \label{upper bound another form111}
    \calL - \calL^{(h)}= \begin{cases}
    [\![\calB - \calB^{(h)}; \mU_1, \mU_2,\mU_3 ]\!] + [\![\calB^{(h)}; \mU_1 - \mU_1^{(h)}, \mU_2,\mU_3 ]\!]\\
    \quad+ [\![\calB^{(h)}; \mU_1^{(h)}, \mU_2 - \mU_2^{(h)},\mU_3 ]\!]+ [\![\calB^{(h)};  \mU_1^{(h)}, \mU_2^{(h)}, \mU_3 - \mU_3^{(h)} ]\!], & \text{Tucker}, \\
     [\mX_1 - \mX_1^{(h)},\mX_2,\mX_3] + [\mX_1^{(h)},\mX_2 - \mX_2^{(h)},\mX_3]+ [\mX_1^{(h)},\mX_2^{(h)},\mX_3 - \mX_3^{(h)}], & \text{TT},\\
     [\calU - \calU^{(h)}, \calD, \calV] + [ \calU^{(h)}, \calD - \calD^{(h)}, \calV] + [\calU^{(h)}, \calD^{(h)}, \calV- \calV^{(h)}] , & \text{Tubal},
    \end{cases}
\end{eqnarray}
where we denote the $M$ components by $\calG_1,\dots,\calG_M$. Each $\calG_m$ has rank at most $2\vr$ and, by the construction of the nets, $\|\calG_m\|_F\le\epsilon$. Therefore, as in \eqref{upper bound of right hand side for sparse1}, $Z\le\max_{\calL^{(h)}}\frac{1}{d_4}\<\calE,\calL^{(h)}\times_3^1\mA\>+M\epsilon Z$, i.e., $Z\le\frac43\max_{\calL^{(h)}}\frac{1}{d_4}\<\calE,\calL^{(h)}\times_3^1\mA\>$. The covering numbers with ranks $2\vr$ give $\log|\setN|\le c_0N$ for an universal constant $c_0$, since doubling the ranks increases the number of parameters by at most a factor of $8$. Applying the union bound as in \eqref{the tail function of fixed gaussian random variable 2} with $\setN$ replaced by this net and $s\log\frac{D}{s}$ replaced by $N$, with probability at least $1-e^{-N}$, we obtain
\begin{eqnarray}
    \|\wh\calL-\calL^\star\|_F\le\frac{8}{3(1-\delta)\sigma_a^2}\,\sigma_a\sigma_e\sqrt{\frac{2(1+\delta)(c_0+1)N}{d_4}}=O\bigg(\frac{\sqrt{(1+\delta)N}}{(1-\delta)\sqrt{d_4}}\,\frac{\sigma_e}{\sigma_a}\bigg),
\end{eqnarray}
which proves \eqref{upper bound of error final main paper LR}. The unconditional probabilities in \Cref{theorem:sparse  structure conclusion main paper,theorem: low-rank structure conclusion main paper} follow from $\P{F}\ge1-2e^{-c\delta^2d_4/K^4}$ as above.
\end{proof}

\subsection[Proof of the bound in \Cref{rem:within window}]{Proof of the bound in \Cref{rem:within window}}
\label{Proof of within window}

\begin{proof}
We first prove \eqref{within window SL}. Let $\calX_i^\star=\calS^\star+\calL^\star$ with $\calS^\star\in\wt\setS_s$ and $\calL^\star\in\wt\setL_{\vr}$, and let $\wh\calH=\wh\calS+\wh\calL-\calX_i^\star$. If \eqref{snapshot assumption} does not hold, then $\calY=\calX_i^\star\times_3^1\mA+\calE+\calW_i$, so that $(\wh\calS+\wh\calL)\times_3^1\mA-\calY=\wh\calH\times_3^1\mA-\calE-\calW_i$. Following the argument leading to \eqref{minimum value and ground-truth value relationship1} with $\calE$ replaced by $\calE+\calW_i$, and using \eqref{lower bound of the left term}, we have
\begin{eqnarray}
    \label{within window basic}
    (1-\delta)\sigma_a^2\|\wh\calH\|_F^2\le\frac{2}{d_4}\<\calE,\wh\calH\times_3^1\mA\>+\frac{2}{d_4}\<\calW_i,\wh\calH\times_3^1\mA\>.
\end{eqnarray}
The first term of \eqref{within window basic} involves only $\calE$ and $\mA$, and is therefore bounded exactly as in \eqref{upper bound another form}--\eqref{first term bound}. For the second term, by the Cauchy--Schwarz inequality and the upper bound in \eqref{RIP condition fro the ToM regression single}, we have
\begin{eqnarray}
    \label{within window lag}
    \frac{2}{d_4}\<\calW_i,\wh\calH\times_3^1\mA\>\le\frac{2}{d_4}\|\calW_i\|_F\sqrt{(1+\delta)d_4}\,\sigma_a\|\wh\calH\|_F=2\sqrt{1+\delta}\,\sigma_a\frac{\|\calW_i\|_F}{\sqrt{d_4}}\|\wh\calH\|_F,
\end{eqnarray}
which does not require $\calW_i$ to be independent of $\mA$. Hence, on the same event as in \eqref{SL quadratic inequality}, we obtain
\begin{eqnarray}
    \label{within window quadratic}
    (1-\delta)\sigma_a^2\|\wh\calH\|_F^2\le\Big(2t_0+2\sqrt{1+\delta}\,\sigma_a\frac{\|\calW_i\|_F}{\sqrt{d_4}}\Big)\|\wh\calH\|_F+4\epsilon_0\big(t_0+\sqrt{5(1+\delta)D}\,\sigma_a\sigma_e\big),
\end{eqnarray}
which differs from \eqref{SL quadratic inequality} only in the additional coefficient $2\sqrt{1+\delta}\,\sigma_a\frac{\|\calW_i\|_F}{\sqrt{d_4}}$ of $\|\wh\calH\|_F$. Dividing by $(1-\delta)\sigma_a^2$, this coefficient becomes $\frac{2\sqrt{1+\delta}}{(1-\delta)\sigma_a}\frac{\|\calW_i\|_F}{\sqrt{d_4}}$, and, since $x^2\le bx+c$ implies $x\le b+\sqrt{c}$, it enters the bound on $\|\wh\calH\|_F$ additively. Hence, the argument leading to \eqref{SL error bound pre} gives
\begin{eqnarray}
    \label{within window pre}
    \|\wh\calH\|_F\le\frac32b+\epsilon_0+\sqrt{\frac{4\epsilon_0\sqrt{5(1+\delta)D}\,\sigma_e}{(1-\delta)\sigma_a}}+\frac{2\sqrt{1+\delta}}{(1-\delta)\sigma_a}\frac{\|\calW_i\|_F}{\sqrt{d_4}}.
\end{eqnarray}
The first three terms of \eqref{within window pre} are bounded as in \eqref{SL last term}--\eqref{upper bound of error final}, which proves \eqref{within window SL}. For the sparse and low-rank models, adding \eqref{within window lag} to the right-hand side of \eqref{minimum value and ground-truth value relationship1 sparse} gives $(1-\delta)\sigma_a^2\|\wh\calH\|_F^2\le2Z\|\wh\calH\|_F+2\sqrt{1+\delta}\,\sigma_a\frac{\|\calW_i\|_F}{\sqrt{d_4}}\|\wh\calH\|_F$, where $Z$ is defined in \eqref{upper bound of right hand side for sparse} (with $\wh\setL_{2\vr}$ in place of $\wh\setS_{2s}$ in the low-rank case). Dividing by $(1-\delta)\sigma_a^2\|\wh\calH\|_F$ and bounding $Z$ as before, the bounds of \Cref{theorem:sparse  structure conclusion main paper,theorem: low-rank structure conclusion main paper} hold with the same additional term $\frac{2\sqrt{1+\delta}}{(1-\delta)\sigma_a}\frac{\|\calW_i\|_F}{\sqrt{d_4}}$.

Finally, we bound $\E[\|\calW_i\|_F^2]$ for the AR(1) model in stationarity. Since $\va_{i-k}$ is independent of $\{\calX_t^\star\}$ and $\E[\va\va^\top]=\sigma_a^2\mId_{d_3\times d_3}$, we have
\begin{eqnarray}
    \label{within window slice}
    \E[\|(\calX^\star_{i-k}-\calX_i^\star)\times_3^1\va_{i-k}\|_F^2]=\sigma_a^2\E[\|\calX^\star_{i-k}-\calX_i^\star\|_F^2].
\end{eqnarray}
By the stationary covariance $\E[\<\calX_a^\star,\calX_b^\star\>]=\frac{D\lambda\sigma_q^2}{2-\lambda}(1-\lambda)^{|a-b|}$ (see \eqref{AR autocovariance}) and $1-(1-\lambda)^k\le\lambda k$, we have
\begin{eqnarray}
    \label{within window increment}
    \E[\|\calX^\star_{i-k}-\calX_i^\star\|_F^2]=\frac{2D\lambda\sigma_q^2}{2-\lambda}\big(1-(1-\lambda)^k\big)\le\frac{2D\lambda^2\sigma_q^2k}{2-\lambda}.
\end{eqnarray}
Summing \eqref{within window slice} over $k=0,\dots,d_4-1$ and using \eqref{within window increment} and $\sum_{k=0}^{d_4-1}k=\frac{d_4(d_4-1)}{2}$, we obtain
\begin{eqnarray}
    \label{within window W bound}
    \E[\|\calW_i\|_F^2]\le\frac{\sigma_a^2D\lambda^2\sigma_q^2d_4(d_4-1)}{2-\lambda}.
\end{eqnarray}
\end{proof}

\section[Structured Adaptive Algorithms]{Structured Adaptive Algorithms}
\label{app:algorithms}

This appendix details the structured adaptive updates \eqref{projected SGD sparse}--\eqref{projected SGD SpL L} of \Cref{sec:structured updates}, which are used in the experiments of \Cref{sec:experimental_results}, together with a data-driven selection of the sparsity level and the ranks.

We begin with the standard SGD update \eqref{updating method X another model II} as a baseline, which serves as the building block for all subsequent algorithmic developments. To improve recovery accuracy at each time instant, we allow multiple gradient steps per observation rather than a single pass. Specifically, upon receiving the $i$-th observation $(\mA_i, \calY_i)$, the estimator performs up to $N_{\max}$ inner iterations before moving to the next sample, with early termination when the relative change falls below a threshold $\epsilon$. We write this criterion as $\|\ol\calX_{i}(t+1) - \ol\calX_i(t)\|_F^2\le\epsilon\|\ol\calX_i(t)\|_F^2$, which is also well defined at the zero initialization. The formal procedure is summarized in Algorithm~\ref{alg:SGD}, and serves as the reference point against which all structured variants are compared.

\begin{algorithm}
\caption{SGD}
\label{alg:SGD}
\footnotesize
\begin{algorithmic}[0]
\State \textbf{Initialization:}
$ \calX_1= \mathbf{0}_{d_1\times d_2\times d_3} $.
\State Preset parameters: $ \mu$, $N_{\max}$, $\epsilon$.
\For {$i = 1,2,\dots$ }
   \State $\ol\calX_i(1) = \calX_{i}$
\For {$t = 1,\dots, N_{\max}$ }
    \State $\ol\calX_{i}(t+1) = \ol\calX_i(t) - \frac{\mu}{d_4} (\ol\calX_i(t) \times_{3}^{1} \mA_i -   \calY_i) \times_{3}^{2} \mA_i$
    \If{$\|\ol\calX_{i}(t+1) - \ol\calX_i(t)\|_F^2 \leq \epsilon\|\ol\calX_i(t)\|_F^2$}
    \State $t_{\text{last}} = t+1$ \ \textbf{break}
    \EndIf
\EndFor
   \State $\calX_{i+1} = \ol\calX_{i}(t_{\text{last}})$
\EndFor
\end{algorithmic}
\end{algorithm}
This baseline framework treats the coefficient tensor as a generic, unstructured array. In the following subsections, we extend it by incorporating low-dimensional structural priors---sparsity, low-rankness, and their combination---into the update rule and develop a family of iterative hard thresholding (IHT) algorithms.

\subsection[Selection of the Sparsity Level and the Ranks]{Selection of the Sparsity Level and the Ranks}
\label{subsec:selection}

The sparsity level $s$ and the ranks $\vr$ in \eqref{projected SGD sparse}--\eqref{projected SGD SpL L} are unknown in practice and may change over time. Since components removed by the projection can hardly be recovered by the projected iterate, we select them from the unstructured SGD iterate
\begin{eqnarray}
    \label{pilot SGD}
    \calZ_{i+1}=\calZ_i-\frac{\mu}{d_4}(\calZ_i\times_3^1\mA_i-\calY_i)\times_3^2\mA_i,\qquad \calZ_1=\mathbf{0}_{d_1\times d_2\times d_3},
\end{eqnarray}
which runs in parallel with the structured update.

At window $i$, we retain the entries and the singular values of $\calZ_{i+1}$ that are at least a fraction $\eta_s,\eta_r\in(0,1)$ of the largest one, i.e.,
\begin{eqnarray}
    \label{selection sparse}
    s_i&\!\!\!\!=\!\!\!\!&\big|\{(s_1,s_2,s_3):|\calZ_{i+1}(s_1,s_2,s_3)|\ge\eta_s\|\calZ_{i+1}\|_\infty\}\big|,\\
    \label{selection Tucker}
    r_{i,k}&\!\!\!\!=\!\!\!\!&\big|\{j:\sigma_j\big((\calZ_{i+1})_{(k)}\big)\ge\eta_r\,\sigma_1\big((\calZ_{i+1})_{(k)}\big)\}\big|,\qquad k=1,2,3,
\end{eqnarray}
where $|\cdot|$ denotes the cardinality, $\|\calZ\|_\infty$ the largest magnitude of the entries of $\calZ$, and $\calZ_{(k)}$ the mode-$k$ unfolding of $\calZ$. \eqref{selection Tucker} gives the Tucker ranks, and $(r_{i,1},r_{i,3})$ gives the TT ranks, since the TT ranks of a third-order tensor are the ranks of its mode-1 and mode-3 unfoldings. The tubal rank is selected by applying the same rule to the frontal slices of $\text{fft}(\calZ_{i+1},[\ ],3)$, with the largest singular value over all slices as the reference.

The idea behind \eqref{selection sparse} and \eqref{selection Tucker} is simple. $\calZ_{i+1}$ is the true tensor $\calX_{i+1}^\star$ plus an estimation error. The nonzero entries and singular values of $\calX_{i+1}^\star$ remain large in $\calZ_{i+1}$, whereas the zero ones become small values of the size of the error (for the singular values, by Weyl's inequality). A threshold between these two levels therefore separates the signal from the error and returns the true sparsity level and ranks. By \Cref{The main theorem of error in the model II,The main theorem of tracking ability in the model II}, a smaller $\mu$ reduces the error of $\calZ_{i+1}$ and thus makes the two levels easier to separate, whereas a larger $\mu$ allows $\calZ_{i+1}$, and thus the selection, to follow a change of $\calX_i^\star$ faster.

\subsection{Sparse Adaptive Tensor Algorithm}
\label{subsec:sparse methods}

Building upon the SGD baseline, we now develop algorithms that explicitly promote sparsity in the recovered tensor. In the broader context of high-dimensional signal recovery, sparsity-enforcing strategies---such as classical $\ell_1$ regularization \cite{chen2009sparse}, reweighted $\ell_1$ techniques \cite{candes2008enhancing}, and smooth $\ell_0$ approximations \cite{gu2009l_}---have been widely adopted due to their theoretical guarantees and empirical effectiveness. Nevertheless, in rapidly evolving or time-varying tensor models, these regularization-based methods can be insufficiently responsive, often lagging in capturing sudden sparse variations.

To address this limitation, we focus on a strategy that is well-suited for online or adaptive scenarios, namely \emph{iterative hard thresholding} (IHT). This approach enforces sparsity via direct projection onto a low-cardinality support set, effectively retaining only the most significant entries at each update while discarding insignificant components. As a result, IHT provides a computationally efficient and dynamically responsive mechanism for tracking sparse tensor structures in high-dimensional, fast-changing environments.

\paragraph{Iterative Hard Thresholding}
Given a sparsity level $s$, the IHT update \eqref{projected SGD sparse} projects the gradient step onto the set of $s$-sparse tensors, where $\calP_s(\cdot)$ retains the $s$ largest-magnitude entries and sets the remainder to zero. At window $i$, the sparsity level $s_i$ is selected by \eqref{selection sparse} and is kept fixed during the inner iterations, so that the selection does not depend on the convergence of the inner iterations. The full procedure is summarized in Algorithm~\ref{alg:sparseIHT}.

\begin{algorithm}[!ht]
\caption{Sparse Iterative Hard Thresholding (Sparse IHT)}
\label{alg:sparseIHT}
\footnotesize
\begin{algorithmic}[0]
\State \textbf{Initialization:} $\calS_1 = \calZ_1 = \mathbf{0}_{d_1\times d_2\times d_3}$.
\State Preset parameters: $\mu$, $\eta_s$, $N_{\max}$, $\epsilon$.
\For {$i = 1,2,\dots$ }
   \State $\calZ_{i+1}=\calZ_i-\frac{\mu}{d_4}(\calZ_i\times_3^1\mA_i-\calY_i)\times_3^2\mA_i$
   \State Select $s_i$ by \eqref{selection sparse}
   \State $\ol\calS_i(1) = \calS_{i}$
\For {$t = 1,\dots, N_{\max}$ }
    \State $\wt{\calS} = \ol\calS_i(t) - \frac{\mu}{d_4}\big(\ol\calS_i(t) \times_{3}^{1}\mA_i - \calY_i\big)\times_{3}^{2}\mA_i$
    \State $\ol\calS_{i}(t+1) = \calP_{s_i}(\wt{\calS})$
    \If{$\|\ol\calS_{i}(t+1) - \ol\calS_{i}(t)\|_F^2 \leq \epsilon\|\ol\calS_{i}(t)\|_F^2$}
    \State $t_{\text{last}} = t+1$ \ \textbf{break}
    \EndIf
\EndFor
   \State $\calS_{i+1} = \ol\calS_{i}(t_{\text{last}})$
\EndFor
\end{algorithmic}
\end{algorithm}

\subsection{Low-rank Adaptive Tensor Algorithm}
\label{subsec:lowrank methods}

Unlike sparsity, which can be imposed in various flexible and localized ways, low-rankness encodes a global structural constraint across all modes of the tensor, making it inherently more challenging to adapt or update over time. A common strategy in tensor recovery is to directly optimize the tensor factors associated with a fixed tensor decomposition \cite{cai2019nonconvex,xia2019polynomial,hao2021sparse,tong2022scaling,han2022optimal,qin2024guaranteed,qin2025robust,qin2025scalable,qin2025computational,li2025adaptive}. However, such approaches are not well suited for adaptive tensor recovery, as the tensor ranks remain fixed once the dimensions of the factors are determined, preventing dynamic adjustment during the recovery process. Moreover, fixed-rank optimization can lead to suboptimal recovery performance when the true underlying tensor exhibits rank variations or temporal changes. To overcome these limitations, it is necessary to develop low-rank adaptive tensor algorithms that can dynamically update the ranks of tensor factors in response to incoming data or changing structural information. In the following, we introduce a framework for such adaptive rank updating, which balances computational efficiency with the flexibility required for online or time-varying tensor recovery tasks.

\paragraph{Iterative Hard Thresholding} Analogous to sparse IHT, low-rank IHT enforces the low-rank structure by truncating the tensor ranks after each gradient step, as in \eqref{projected SGD lowrank}, where $\calP_{\vr}(\cdot)$ is computed by HOSVD~\cite{sidiropoulos2000uniqueness} for the Tucker decomposition, TT-SVD~\cite{oseledets2011tensor} for the TT decomposition, or t-SVD~\cite{kilmer2011factorization} for the tubal decomposition. This operation requires the ranks. Soft thresholding of the singular values, i.e., nuclear-norm regularization, adapts the ranks continuously, but its result depends on the regularization parameter, and the shrinkage biases the retained singular values. We therefore keep the hard thresholding in \eqref{projected SGD lowrank} and select the ranks $\vr_i$ at window $i$ by \eqref{selection Tucker}, which are kept fixed during the inner iterations. The detailed procedure is summarized in \Cref{alg:lowrankIHT}.

\begin{algorithm}[!ht]
\caption{Low-rank  Iterative Hard Thresholding (Low-rank IHT)}
\label{alg:lowrankIHT}
\footnotesize
\begin{algorithmic}[0]
\State \textbf{Initialization:} $\calL_1 = \calZ_1 = \mathbf{0}_{d_1\times d_2\times d_3}$.
\State Preset parameters: $\mu$, $\eta_r$, $N_{\max}$,  $\epsilon$.
\For{$i = 1,2,\dots$}
   \State $\calZ_{i+1}=\calZ_i-\frac{\mu}{d_4}(\calZ_i\times_3^1\mA_i-\calY_i)\times_3^2\mA_i$
   \State Select $\vr_i$ by \eqref{selection Tucker}
   \State $\ol\calL_i(1) = \calL_{i}$
\For {$t = 1,\dots, N_{\max}$ }
    \State $\wt{\calL} = \ol\calL_i(t) - \frac{\mu}{d_4}\big(\ol\calL_i(t) \times_{3}^{1}\mA_i - \calY_i\big)\times_{3}^{2}\mA_i$
    \State $\ol\calL_{i}(t+1) = \calP_{\vr_i}(\wt{\calL})$ \Comment{via HOSVD, TT-SVD, or t-SVD}
    \If{$\|\ol\calL_{i}(t+1) - \ol\calL_{i}(t)\|_F^2 \leq \epsilon\|\ol\calL_{i}(t)\|_F^2$}
    \State $t_{\text{last}} = t+1$ \ \textbf{break}
    \EndIf
\EndFor
   \State $\calL_{i+1} = \ol\calL_{i}(t_{\text{last}})$
\EndFor
\end{algorithmic}
\end{algorithm}

\subsection{Sparse and Low-rank Adaptive Tensor Algorithm}
\label{subsec:sparse and lowrank methods}

Finally, when a tensor can be decomposed into the sum of a sparse component and a low-rank component, one can further integrate the advantages of both structural priors. We employ the alternating updates \eqref{projected SGD SpL S}--\eqref{projected SGD SpL L}, in which the sparse part $\calS$ and the low-rank part $\calL$ are updated with the gradient of the same residual and with step sizes $\mu_1$ and $\mu_2$. Since $\calZ_{i+1}$ estimates $\calS_{i+1}^\star+\calL_{i+1}^\star$, which is neither sparse nor low-rank, the two components are selected one after the other: the sparsity level $s_i$ is selected by \eqref{selection sparse}, which retains the large entries of $\calZ_{i+1}$ contributed by the sparse component, and the ranks $\vr_i$ by \eqref{selection Tucker} with $\calZ_{i+1}$ replaced by $\calZ_{i+1}-\calP_{s_i}(\calZ_{i+1})$, i.e., after removing these entries. The alternating structure allows the model to capture both global correlations (through $\calL$) and localized anomalies (through $\calS$) in a data-adaptive manner. The full algorithm is shown in \Cref{alg:sparse_lowrank}.

\begin{algorithm}[!ht]
\caption{Sparse and Low-rank Iterative Hard Thresholding (Sparse and Low-rank IHT)}
\label{alg:sparse_lowrank}
\footnotesize
\begin{algorithmic}[0]
\State \textbf{Initialization:} $\calS_1 = \calL_1 = \calZ_1 = \mathbf{0}_{d_1\times d_2\times d_3}$.
\State Preset parameters: $\mu$, $\mu_1$, $\mu_2$, $\eta_s$, $\eta_r$,  $N_{\max}$, $\epsilon_1$, $\epsilon_2$.
\For {$i = 1,2,\dots$ }
   \State $\calZ_{i+1}=\calZ_i-\frac{\mu}{d_4}(\calZ_i\times_3^1\mA_i-\calY_i)\times_3^2\mA_i$
   \State Select $s_i$ by \eqref{selection sparse}, and $\vr_i$ by \eqref{selection Tucker} from $\calZ_{i+1}-\calP_{s_i}(\calZ_{i+1})$
   \State $\ol\calS_i(1) = \calS_{i}$, $\ol\calL_i(1) = \calL_{i}$
\For {$t = 1,\dots, N_{\max}$ }
    \State $\wt{\calS} = \ol\calS_i(t) - \frac{\mu_1}{d_4}\big((\ol\calS_i(t) +\ol\calL_i(t) )\times_{3}^{1}\mA_i - \calY_i\big)\times_{3}^{2}\mA_i$
    \State $\wt{\calL} = \ol\calL_i(t) - \frac{\mu_2}{d_4}\big((\ol\calS_i(t) +\ol\calL_i(t) ) \times_{3}^{1}\mA_i - \calY_i\big)\times_{3}^{2}\mA_i$
    \State $\ol\calS_{i}(t+1) = \calP_{s_i}(\wt{\calS})$
    \State $\ol\calL_{i}(t+1) = \calP_{\vr_i}(\wt{\calL})$ \Comment{via HOSVD, TT-SVD, or t-SVD}
    \If{$\|\ol\calS_{i}(t+1) - \ol\calS_{i}(t)\|_F^2 \leq \epsilon_1\|\ol\calS_{i}(t)\|_F^2$ and $\|\ol\calL_{i}(t+1) - \ol\calL_{i}(t)\|_F^2 \leq \epsilon_2\|\ol\calL_{i}(t)\|_F^2$}
    \State $t_{\text{last}} = t+1$ \ \textbf{break}
    \EndIf
\EndFor
   \State $\calS_{i+1} = \ol\calS_{i}(t_{\text{last}})$, $\calL_{i+1} = \ol\calL_{i}(t_{\text{last}})$
\EndFor
\end{algorithmic}
\end{algorithm}

\subsection[Choice of Parameters]{Choice of Parameters}
\label{subsec:parameters}

The parameters of the algorithms play the following roles.
\begin{itemize}
\item \emph{Step size $\mu$.} It trades off the accuracy of the selection against the speed at which the selection follows a change of $\calX_i^\star$. In \Cref{alg:sparse_lowrank}, $\mu_1$ and $\mu_2$ are the step sizes of the two components, and $\mu$ is used only in \eqref{pilot SGD}.
\item \emph{Thresholds $\eta_s$ and $\eta_r$.} They should lie between the two levels in {Appendix}~\ref{subsec:selection}, i.e., they should be larger than the relative error of $\calZ_{i+1}$ and smaller than the ratio of the smallest to the largest nonzero entry or singular value of $\calX_i^\star$. The sparsity level and the ranks are not bounded by any preset value.
\item \emph{$N_{\max}$ and $\epsilon$, $\epsilon_1$, $\epsilon_2$.} They only control the inner iterations within one window, in which the selected structure is fixed.
\end{itemize}

\subsection{Computational Complexity}
\label{subsec:complexity}

We summarize the computational complexity of each algorithm in Table~\ref{tab:complexity}, where complexity is measured in terms of the number of multiplications required per time instant $i$, including all $N_{\max}$ inner iterations. The dominant cost in each algorithm is the gradient computation $(\calX \times_3^1 \mA_i - \calY_i)\times_3^2\mA_i$, which
requires $O(2N_{\max}d_1d_2d_3d_4)$ multiplications. The structured algorithms incur additional computational costs for enforcing structural priors, either through projection steps or structure-aware multiplicative updates, depending on the specific algorithmic design. In addition, the structure selection requires the update \eqref{pilot SGD}, which costs $O(2d_1d_2d_3d_4)$ multiplications, and the SVDs in \eqref{selection Tucker}, which cost the same as one projection. Both are performed once per window, and thus add the cost of about one inner iteration.

\begin{table*}[!ht]
\centering
\footnotesize
\caption{Multiplication complexity (including $N_{\max}$ inner iterations).}
\label{tab:complexity}
\renewcommand{\arraystretch}{1.2}
\begin{tabular}{lll}
\hline
\textbf{Algorithm} & \textbf{Projection} &
\textbf{Complexity per $i$} \\
\hline
SGD & --- & $O\big(N_{\max} d_1 d_2 d_3 (2d_4 +2)  \big)$ \\

Sparse IHT & Sparse & $O\big(N_{\max} d_1 d_2 d_3 (2d_4 +2 + \log(d_1 d_2 d_3))    \big)$ \\

Low-rank IHT & Tucker & $O\big(N_{\max} d_1 d_2 d_3 (2d_4 +2 + d_1 +d_2 +d_3 + r_1^\text{tk}r_2^\text{tk}r_3^\text{tk} )  \big)$ \\
& Tensor train & $O\big(N_{\max}d_1 d_2 d_3 (2d_4 +2 + 3(\max\{r_1^{\text{tt}}, r_2^{\text{tt}}\})^2)  \big)$ \\
& Tubal & $O\big(N_{\max} (d_1 d_2 d_3 2d_4 +2 + 2\log d_3 + \min\{ d_1, d_2 \} + r_{\text{tubal}}   )\big)$ \\

Sparse and Low-rank IHT & Sparse + Tucker & $O\big(N_{\max} d_1 d_2 d_3 (2d_4 +4 + \log(d_1 d_2 d_3) + d_1 +d_2 +d_3 + r_1^\text{tk}r_2^\text{tk}r_3^\text{tk} )  )\big)$ \\
& Sparse + Tensor train & $O\big(N_{\max}d_1 d_2 d_3 (2d_4 +4 + \log(d_1 d_2 d_3) + 3(\max\{r_1^{\text{tt}}, r_2^{\text{tt}}\})^2)  )\big)$ \\
& Sparse + Tubal & $O\big(N_{\max} d_1 d_2 d_3 (2d_4 + 4 + \log(d_1 d_2 d_3) + 2\log d_3 + \min\{ d_1, d_2 \} + r_{\text{tubal}} )  )\big)$ \\
\hline
\end{tabular}
\end{table*}

\section{Auxiliary Material}
\label{sec: Auxiliary Material}

\begin{lemma}(Covering number of Tucker decomposition, \cite[Lemma 15]{luo2024tensor})
\label{covering number of Tucker decomposition}
Consider the set $\setX_{\text{tk}} = \{\calX\in\R^{d_1\times d_2\times d_3}: \calX= [\![\calB; \mU_1, \mU_2,\mU_3 ]\!],  \mU_1\in\R^{d_1\times r_1^\text{tk}}, \mU_2\in\R^{d_2\times r_2^\text{tk}}, \mU_3\in\R^{d_3\times r_3^\text{tk}}, \calB\in\R^{r_1^\text{tk}\times r_2^\text{tk}\times  r_3^\text{tk}}, \|\calX\|_F\leq L  \}$. There exists an $\epsilon$-net $\wt\setX_{\text{tk}}$ for $\setX_{\text{tk}}$ under the Frobenius norm, i.e., $\|\calX - \calX^{(h)}\|_F\leq \epsilon$ for any $\calX^{(h)}\in\wt\setX_{\text{tk}}$. Moreover, the covering number is bounded as
\begin{eqnarray}
    \label{covering number in the Tucker decomposition}
    O\bigg(\frac{12L }{\epsilon}\bigg)^{r_1^\text{tk}r_2^\text{tk}r_3^\text{tk} + (d_1-r_1^\text{tk})r_1^\text{tk} + (d_2-r_2^\text{tk})r_2^\text{tk} + (d_3-r_3^\text{tk})r_3^\text{tk}}.
\end{eqnarray}
\end{lemma}

\begin{lemma}(Covering number of TT decomposition, \cite[Lemma 3]{qin2025enhancing})
\label{covering number of TT decomposition}
Consider the set $\setX_{\text{tt}} = \{\calX\in\R^{d_1\times d_2\times d_3}: \calX = [\mX_1, \mX_2, \mX_3],  \mX_1\in\R^{d_1\times r_1^{\text{tt}}}, \mX_2\in\R^{r_1^{\text{tt}} \times d_2\times r_2^{\text{tt}}}, \mX_3\in\R^{r_2^{\text{tt}} \times d_3}, \|\calX\|_F\leq L  \}$. There exists an $\epsilon$-net $\wt\setX_{\text{tt}}$ for $\setX_{\text{tt}}$ under the Frobenius norm, i.e.,   $\|\calX - \calX^{(h)}\|_F\leq \epsilon$ for any $\calX^{(h)}\in\wt\setX_{\text{tt}}$. Moreover, the covering number is bounded as
\begin{eqnarray}
    \label{covering number in the TT decomposition}
    \bigg(\frac{12L + \epsilon}{\epsilon}\bigg)^{d_1r_1^{\text{tt}}+d_2r_1^{\text{tt}}r_2^{\text{tt}}+d_3r_2^{\text{tt}}}.
\end{eqnarray}
\end{lemma}

\begin{lemma}(Covering number of Tubal decomposition)\
\label{covering number of Tubal decomposition} Consider the set $\setX_{\text{tubal}} = \{\calX\in\R^{d_1\times d_2\times d_3}: \calX = [\calU, \calD, \calV],  \calU \in\R^{d_1\times r_{\text{tubal}}\times d_3}, \calV \in\R^{d_2\times r_{\text{tubal}}\times d_3}, \calD \in\R^{r_{\text{tubal}}\times r_{\text{tubal}}\times d_3}, \|\calX\|_F\leq L  \}$. There exists an $\epsilon$-net $\wt\setX_{\text{tubal}}$ for $\setX_{\text{tubal}}$ under the Frobenius norm, i.e.,   $\|\calX - \calX^{(h)}\|_F\leq \epsilon$ for any $\calX^{(h)}\in\wt\setX_{\text{tubal}}$. Moreover, the covering number is bounded as
\begin{eqnarray}
    \label{covering number in the Tubal decomposition}
    \bigg(\frac{12L+\epsilon}{\epsilon}\bigg)^{d_1d_3r_{\text{tubal}} + d_3r_{\text{tubal}} + d_2d_3r_{\text{tubal}} }.
\end{eqnarray}
\end{lemma}

\begin{proof}
First, we can derive $\|\calX \|_F^2 = \frac{1}{d_3}\| \text{fft}(\calX, [\ ],3 )\|_F^2 = \frac{1}{d_3} \sum_{s_3=1}^{d_3} \|\ol\calU(:,:,s_3)\ol\calD(:,:,s_3)\ol\calV^\top(:,:,s_3)\|_F^2 = \frac{1}{d_3} \sum_{s_3=1}^{d_3} \|\ol\calD(:,:,s_3)\|_F^2 = \frac{1}{d_3}\|\ol\calD \|_F^2 = \| \calD \|_F^2\leq L^2$.

In addition, according to  \cite{zhang2018tensor}, we can construct $\epsilon_1$-net $\{\ol\calU^{(1)}(:,:,s_3), \dots, \ol\calU^{(N_1)}(:,:,s_3)  \}, s_3\in[d_3]$ with the covering number $N_1\leq (\frac{4+\epsilon_1}{\epsilon_1})^{d_1d_3r_{\text{tubal}}}$. Similarly, we can construct  $\epsilon_2$-net $\{\calD^{(1)}(:,:,s_3), \dots, \calD^{(N_2)}(:,:,s_3)  \}, s_3\in[d_3]$ with the covering number $N_2\leq (\frac{2 L +\epsilon_2}{\epsilon_2})^{d_3r_{\text{tubal}}}$ and  $\epsilon_3$-net $\{\ol\calV^{(1)}(:,:,s_3), \dots, \ol\calV^{(N_3)}(:,:,s_3)  \}, s_3\in[d_3]$ with the covering number $N_3\leq (\frac{4+\epsilon_3}{\epsilon_3})^{d_2d_3r_{\text{tubal}}}$.

Now we can  compute
\begin{eqnarray}
    \label{expansion of difference between X and Xh}
    \|\calX - \calX^{(h)}\|_F^2 &\!\!\!\! = \!\!\!\!& \frac{1}{d_3}\| \text{fft}(\calX, [\ ],3 ) - \text{fft}(\calX^{(h)}, [\ ],3 )  \|_F^2\nonumber\\
    &\!\!\!\! = \!\!\!\!& \frac{1}{d_3}\sum_{s_3=1}^{d_3}\| \ol\calU(:,:,s_3)\ol\calD(:,:,s_3)\ol\calV^\top(:,:,s_3) - \ol\calU^{(h)}(:,:,s_3)\ol\calD^{(h)}(:,:,s_3){\ol\calV^{(h)}}^\top(:,:,s_3)  \|_F^2\nonumber\\
    &\!\!\!\! \leq \!\!\!\!& \frac{3}{d_3}\sum_{s_3=1}^{d_3}\|(\ol\calU(:,:,s_3) - \ol\calU^{(h)}(:,:,s_3)   ) \ol\calD(:,:,s_3)  \|_F^2 + \frac{3}{d_3}\sum_{s_3=1}^{d_3}\| \ol\calD(:,:,s_3) - \ol\calD^{(h)}(:,:,s_3) \|_F^2\nonumber\\
    &\!\!\!\! \!\!\!\!& + \frac{3}{d_3}\sum_{s_3=1}^{d_3}\| \ol\calD^{(h)}(:,:,s_3)( \ol\calV (:,:,s_3) - \ol\calV^{(h)}(:,:,s_3)  )^\top   \|_F^2\nonumber\\
    &\!\!\!\! \leq \!\!\!\!& \frac{3}{d_3}\sum_{s_3=1}^{d_3}\|\ol\calU(:,:,s_3) - \ol\calU^{(h)}(:,:,s_3) \|_F^2\|\ol\calD(:,:,s_3)  \|_F^2 + 3\sum_{s_3=1}^{d_3}\| \calD(:,:,s_3) - \calD^{(h)}(:,:,s_3) \|_F^2\nonumber\\
    &\!\!\!\! \!\!\!\!& + \frac{3}{d_3}\sum_{s_3=1}^{d_3}\| \ol\calD^{(h)}(:,:,s_3)\|_F^2 \| \ol\calV (:,:,s_3) - \ol\calV^{(h)}(:,:,s_3)     \|_F^2\nonumber\\
    &\!\!\!\! \leq \!\!\!\!& 3L^2\epsilon_1^2 + 3\epsilon_2^2 + 3L^2\epsilon_3^2\nonumber\\
    &\!\!\!\! \leq \!\!\!\!&\epsilon^2,
\end{eqnarray}
where we take $\epsilon_1 = \epsilon_3 = \frac{\epsilon}{3L}$ and $\epsilon_2 = \frac{\epsilon}{3}$ in the last line.
Consequently, there exists an $\epsilon$-net $\wt\setX_{\text{tubal}}$ for $\setX_{\text{tubal}}$ under the Frobenius norm, i.e.,   $\|\calX - \calX^{(h)}\|_F\leq \epsilon$ for any $\calX^{(h)}\in\wt\setX_{\text{tubal}}$, with the covering number
\begin{eqnarray}
    \label{covering number in the tubal tensor}
    N_1N_2N_3&\!\!\!\! \leq \!\!\!\!& \bigg(\frac{12L+\epsilon}{\epsilon}\bigg)^{d_1d_3r_{\text{tubal}}}\bigg(\frac{12L+\epsilon}{\epsilon}\bigg)^{d_2d_3r_{\text{tubal}}}\bigg(\frac{6 L +\epsilon}{\epsilon}\bigg)^{d_3r_{\text{tubal}}}\nonumber\\
    &\!\!\!\! \leq \!\!\!\!&\bigg(\frac{12L+\epsilon}{\epsilon}\bigg)^{d_1d_3r_{\text{tubal}} + d_3r_{\text{tubal}} + d_2d_3r_{\text{tubal}} }.
\end{eqnarray}

\end{proof}


\end{document}